%% file: main.tex
\documentclass{article} 
\usepackage{iclr2027_conference,times}

\usepackage{amsmath, amssymb, amsthm}
\usepackage{mathtools}
\usepackage{hyperref}
\usepackage[dvipsnames]{xcolor}
\usepackage{cleveref}
\usepackage{cancel}
\usepackage[multiple]{footmisc} 
\usepackage{booktabs} 
\usepackage[table]{xcolor} 
\definecolor{highlight}{gray}{0.93}
\usepackage{caption}
\usepackage{threeparttable} 
\usepackage{makecell}
\usepackage{thmtools}
\usepackage{cases}
\usepackage{etoc} 
\usepackage{thm-restate}

\allowdisplaybreaks

\usepackage{xcolor}
\usepackage{soul}        
\usepackage[most]{tcolorbox}  
\definecolor{accent}{HTML}{5E3B9E}
\definecolor{accentbg}{HTML}{E8DFF5}
\sethlcolor{accentbg}

\newtheorem{theorem}{Theorem}[section]
\newtheorem{lemma}[theorem]{Lemma}

\newtheorem{remark}[theorem]{Remark}
\newtheorem*{remark*}{Remark}
\newtheorem{corollary}[theorem]{Corollary}
\newtheorem{assumption}{Assumption}
\Crefname{prop}{Proposition}{Propositions}
\usepackage{bm}
\input{math_commands.tex}

\input{abbrv.tex}

\usepackage{hyperref}
\usepackage{url}

\title{Vocabulary-size-independent Convergence of 
Discrete Diffusion Models: 
adjoint equations induce the right space
\vspace{1em}}

\author{
  \centerline{%
    \begin{tabular}{c c c}
      \textbf{Kelvin Kan} & \textbf{Xingjian Li} & \textbf{Benjamin J. Zhang} \\
      Department of Mathematics & Oden Institute & Department of Mathematics \\
      UCLA & University of Texas at Austin & Rutgers University \\
      \texttt{kelvin.kan@math.ucla.edu} & \texttt{xingjian.li@austin.utexas.edu} & \texttt{ben.zhang@rutgers.edu} \\
      \\
      \textbf{Tuhin Sahai} & \textbf{Stanley Osher} & \textbf{Markos A. Katsoulakis} \\
      Computational and Applied Sciences Group & Department of Mathematics & Department of Mathematics and Statistics \\
      SRI International & UCLA & University of Massachusetts Amherst \\
       \texttt{tuhin.sahai@sri.com}& \texttt{sjo@math.ucla.edu} & \texttt{markos@umass.edu}
    \end{tabular}
  }%
}

\iclrfinalcopy 
\begin{document}

\maketitle

\begin{abstract}
Discrete diffusion has become a leading framework for generative modeling in various applications including language, vision, and biology. Existing convergence theory, however, exhibits fundamental limitations. KL-based analyses diverge under singular priors such as the masked distribution, while bounds in total variation (TV) depend on the vocabulary size $S$ and become vacuous for modern language tasks, where vocabularies contain hundreds of thousands of tokens. We develop a unified adjoint-equation-based framework that establishes vocabulary-size-independent convergence guarantees in any integral probability metric (IPM). To the best of our knowledge, our bounds are the first to be entirely free of $S$ and applicable to both masked and uniform priors. Importantly, our results can extend existing step complexity guarantees to any IPM. Also, our theory relies only on a single standard rate-matrix regularity assumption and applies to general priors. 

Five novel techniques drive our improvements: 1. working in the space of observables via adjoint equations rather than directly with probability measures; 2. a regularity analysis that yields bounds on any IPM; 3. a coupling argument that removes $S$-dependence under uniform transitions; and 4. score–marginal cancellation and 5. exit-routing techniques that remove $S$-dependence under masked transitions. Our framework thus sharply departs from prior analyses and avoids the shortcomings of pathspace-KL and existing TV-based approaches. Beyond convergence bounds, our framework provides a versatile toolkit for further theoretical study of discrete diffusion models, including principled choices of loss functions and vocabulary-size-independent step complexity.
\end{abstract}

\section{Introduction}
Diffusion-based generative models have achieved remarkable success across continuous-data domains such as images, audio, and video~\citep{ho2020denoising,song2021scorebased,rombach2022high}. For discrete data, where Gaussian-based corruption does not directly apply, a parallel line of work has developed \emph{discrete diffusion models}~\citep{sohl2015deep,austin2021structured,hoogeboom2021argmax,campbell2022continuous}. These models have demonstrated strong empirical performance across a wide range of applications, including language modeling~\citep{lou2024discrete,sahoo2024simple}, image generation~\citep{campbell2022continuous}, tabular data synthesis~\citep{kotelnikov2023tabddpm}, molecular generation~\citep{vignac2023digress}, symbolic music generation~\citep{campbell2022continuous}, genomic sequence design~\citep{schiff2025simple}, and protein design~\citep{campbell2024generative}.

Most discrete diffusion models are based on the continuous-time Markov chain (CTMC) framework, either directly or through simplified variants inspired by it. In this framework, the forward noising process is governed by a transition rate matrix. The reverse process is parameterized by a learned approximation of the \emph{discrete score}, defined as the ratios of forward marginal probabilities. This quantity serves as the discrete analog of the Stein score in continuous diffusion~\citep{campbell2022continuous,meng2022concrete,lou2024discrete}. Two choices of rate matrix dominate practice: the \emph{uniform} rate matrix and the \emph{absorbing} (or masked) rate matrix~\citep{austin2021structured,lou2024discrete,shi2024simplified}. The latter, in particular, has emerged as the dominant paradigm for large-scale discrete generative modeling, especially in language tasks~\citep{sahoo2024simple,nie2026large}.

Despite this empirical progress, existing convergence analyses for score-based discrete diffusion suffer from fundamental limitations. These include divergent bounds for singular priors such as masked diffusion, vacuous scaling with vocabulary size in large-vocabulary regimes, restriction to specific rate matrices or evaluation metrics, and reliance on strong regularity assumptions. In this work, we develop a unified analysis framework based on adjoint equations that simultaneously addresses all of these limitations.
\begin{table}[ht]
\centering
\caption{Comparison of results for discrete diffusion models.} 
\label{table:comparison}
\small 
\begin{tabular}{ccccc}
\toprule
Reference & Rates/Prior & \makecell{$S$-free$^1$} & Key Conditions & \makecell{Metrics/ \\ Divergences} \\ 
\midrule
\makecell{This work; Thm~\ref{thm:spec}, \\ Lem~\ref{lem:IPM} \& Cors~\ref{cor:ES}-\ref{cor:StepComplexity}} & \makecell{\textbf{General, including} \\ \textbf{Uniform \& Masked}} & \textbf{Yes}$^5$ & \makecell{\textbf{One standard} \\ \textbf{rate-matrix condition}} & \textbf{Any IPM} \\
\midrule
\makecell{\cite{zhang2025convergence}; Thm 1} & Uniform & No & Fully-supported prior & KL, TV  \\
\midrule
\makecell{\cite{liang2026absorb}; Thm 2} & Surrogate Masked & No$^2$ & Fully-supported prior & KL  \\
\midrule
\makecell{\cite{ren2025discrete}; Thm 4.7} & Uniform$^3$ & No & Fully-supported prior & KL  \\
\midrule
\makecell{\cite{liang2026discrete}; Thm 2} & Uniform & No & Fully-supported prior & KL  \\
\midrule
\makecell{\cite{dmitriev2026efficient}; \\ Thms 1\& 3}  & \makecell{Uniform \& \\ Surrogate Masked} & No & Fully-supported prior & KL  \\
\midrule
\makecell{\cite{liang2026sharp}; Thm 2} & Masked & No$^4$ & Differentiating TV metric & TV  \\
\bottomrule
\end{tabular}
\begin{tablenotes}
        \footnotesize
        \item[1] $1$: $S$: vocabulary size. In modern language tasks, $S > 10^5$~\citep{nie2026large} renders $S$-dependent bounds loose/vacuous.
        \item[2] $2$: \cite{liang2026absorb}'s Theorem 1 omitted an $S$-dependent constant.
        \item[3] $3$: \cite{ren2025discrete}'s Theorem C.1 assumes a uniform prior.
        \item[4] $4$: \cite{liang2026sharp}'s Proposition 1 contains an $S$-dependent coefficient.
        \item[5] $5$: Our convergence bounds for uniform and masked rates are $S$-independent.
    \end{tablenotes}
\end{table}

We summarize our contributions as follows.
\begin{itemize}
    \item We establish the first convergence bounds that are \textbf{entirely independent of the vocabulary size $S$} and applicable to non-fully-supported priors such as the singular masked distribution. These properties are jointly essential for language modeling, the most important application of discrete diffusion, where vocabulary sizes are prohibitively large and masked diffusion is the dominant paradigm.
    \item Our framework yields a single unified derivation that \textbf{encompasses all integral probability metrics (IPMs) and all priors}. In contrast, existing analyses are typically restricted to one or two specific metrics or priors. Moreover, our results are based on a single standard rate-matrix regularity assumption. See~\Cref{table:comparison} for detailed comparisons. In turn, our results can \textbf{extend existing step complexity guarantees to any IPM}, since their analyses can be directly applied to our bounds.
    \item Finally, our theoretical framework extends beyond convergence bounds. It establishes a versatile mathematical toolkit that paves the way for future, broader theoretical analyses of discrete diffusion models, including principled choices of loss functions for different dynamics and sharper ($S$-free) step complexity bounds.
\end{itemize}

These contributions are driven by the following main technical novelties.
\begin{itemize}
    \item We work with the \textbf{adjoint equation} (specifically, the Kolmogorov backward equation) in the \textbf{space of observables}, in contrast to prior analyses that operate directly in the space of probability measures. This dual perspective sharply departs from existing pathspace-KL and TV-based estimates. Together with a regularity analysis of the adjoint equation, it is the key enabler of the insensitivity to non-fully-supported priors and the applicability to all IPMs. To the best of our knowledge, this is \textbf{the first such application} to discrete diffusion.
    \item We introduce \textbf{novel score-marginal cancellation and exit-rate routing techniques} that removes the $S$-dependence of convergence bounds for masked transitions. To the best of our knowledge, this is \textbf{the first $S$-free bound for masked transitions.}
    \item We apply a \textbf{coupling argument} that bounds the initialization error under uniform transitions independently of $S$. To the best of our knowledge, \textbf{this is the first such bound}, and the first use of coupling techniques in discrete diffusion.
\end{itemize}

\section{Problem Formulation}\label{sec:setup}
We state the setup used throughout; a more detailed description is given in Appendix~\ref{sec:setup_full}.

\paragraph{Notation.} Scalars are lowercase ($x$) and vectors boldface lowercase ($\bfx$); all vectors are columns. We consider data on the discrete state space $\mathcal{X} = [S]^d$, where $[S] = \{1,\dotsc,S\}$ and $\bfx = [x^1 \dotsm x^d]$ has entries $x^i \in [S]$. This setup is common to text~\citep{lou2024discrete}, image pixels~\citep{campbell2022continuous}, tabular data~\citep{kotelnikov2023tabddpm}, molecular graphs~\citep{vignac2023digress}, musical notes~\citep{campbell2022continuous}, and biological sequences~\citep{schiff2025simple,campbell2024generative}, among others.

We write $\xito$ for $\bfx$ with its $i$-th entry replaced by $\hat{x}^i$, and $\bfx^{\backslash i}$ for $\bfx$ with that entry removed. The Hamming distance $\mathrm{Ham}(\bfx,\bfy) = |\{i \in [d] : x^i \neq y^i\}|$ counts differing entries, where $|\cdot|$ is set cardinality. We write $\bfe_i$ for the $i$-th standard basis vector, $\mathbf{1}_n$ for the all-ones vector, $I_n$ for the identity, and $\delta\{\cdot,\cdot\}$ for the Kronecker delta. A distribution on $\mathcal{X}$ is a probability mass vector $p \in \Delta(\mathcal{X})$, the simplex over the $S^d$ states.
 
\paragraph{Forward Noising Process.} The forward noising process $X=\{X_t\}_{0 \leq t \leq T}$ is a continuous-time Markov chain (CTMC) on $\mathcal{X}$ whose marginal $p_t \in \Delta(\mathcal{X})$ is governed by the \emph{Kolmogorov forward equation} (KFE), a linear ordinary differential equation (ODE) given by
\newcommand{\KFEntokeneq}{\frac{dp_t}{dt}=Q_t^\top p_t, \quad p_0=p_{\mathrm{data}}, \quad \text{for } t \in [0, T]}
\begin{equation}\label{eq:KFE_ntoken}
    \KFEntokeneq ,
\end{equation}
with rate matrix (infinitesimal generator) $Q_t \in \mathbb{R}^{S^d \times S^d}$, where $Q_t(\bfx, \bfy)$ is the infinitesimal rate of transitioning from states $\bfx$ to $\bfy$. Because $Q_t$ is constrained to have nonnegative off-diagonal entries and zero row sums, the dynamics~\eqref{eq:KFE_ntoken} preserve the total mass of $p_t$, ensuring it remains a valid probability distribution throughout the evolution. 
 
Because the dimensions of $Q_t$ scale exponentially with the sequence length $d$, direct computation or storage of the generator is intractable. In practice, $Q_t$ is structured to operate on each dimension independently via a token-level rate matrix $Q_t^\mathrm{tok} \in \mathbb{R}^{S \times S}$. In standard frameworks~\citep{campbell2022continuous,lou2024discrete,zhang2025convergence}, the nonzero off-diagonal entries of $Q_t$ are structured as 
\newcommand{\Qfactoreq}{Q_t(\bfx, \xito) = Q_t \left( [x^1, \dotsm, x^i, \dotsm, x^d], [x^1, \dotsm, \hat{x}^i, \dotsm, x^d] \right) = Q_t^\mathrm{tok}(x^i, \hat{x}^i),}
\begin{equation}\label{eq:Q_factor}
    \Qfactoreq
\end{equation}
where $x^i \neq \hat{x}^i \in [S]$.
This structure implies that $Q_t$ is a Kronecker sum of independent generators, connecting only sequences at Hamming distance $1$ (our results also apply to general rate matrices; see \Cref{sec:result}). The two common choices for $Q_t^\mathrm{tok}$ are the masked (absorbing) and uniform rates,
\newcommand{\Qabsorbunieq}{(Q_t^{\rm masked})^\top = \beta(t) ( \bfe_{\mask} \mathbf{1}_S^\top - I_S) \quad \text{and} \quad (Q_t^{\rm uniform})^\top = \beta(t) \left(\tfrac{1}{S} \mathbf{1}_S \mathbf{1}_S^\top - I_S \right),}
\begin{equation}\label{eq:Q_absorb_uni}
    \Qabsorbunieq
\end{equation}
respectively. For the masked rate, the $S$-th token is the mask $\mask$. Consequently, $p_t$ approaches a tractable limiting distribution $p_{\mathrm{base}}$ as $t \to \infty$: either a Dirac distribution centered on the all-mask sequence or the uniform distribution on $\mathcal{X}$. The ease of sampling from these distributions makes them ideal priors for the reverse generative process.
 
Under~\eqref{eq:Q_factor}, the forward transition kernel $p_{t|s}(\bfx_t|\bfx_s):=\Pr(X_t=\bfx_t| X_s=\bfx_s)$ factorizes as $p_{t|s}(\bfx_t|\bfx_s)=\prod_{i=1}^d p^i_{t|s}(x_t^i|x_s^i)$ for $0 \leq s < t \leq T$, where $p^i_{t|s}$ is the forward transition kernel of the $i$-th token. This factorization is standard both for scaling to high-dimensional state spaces in practice~\citep{austin2021structured,campbell2022continuous,campbell2024generative,lou2024discrete} and in theoretical analysis~\citep{zhao2025unified,chen2025convergence,zhang2025convergence,ren2025discrete}.
 
\paragraph{Reverse Denoising Process.} The forward process~\eqref{eq:KFE_ntoken} admits a reverse-time process, which is also a CTMC~\citep{kelly1979reversibility}. The reverse process satisfies the reverse-time Kolmogorov equation
\newcommand{\reverseeq}{-\frac{dp_{t}}{dt} = (\backQt)^\top p_{t}, \quad \text{for } t \in [0, T],}
\begin{equation}\label{eq:reverse}
    \reverseeq
\end{equation}
where the reverse rate matrix $\backQt$ is given by
\newcommand{\reversematrixntokeneq}{\backQt(\bfx, \bfy) = \frac{p_{t}(\bfy)}{p_{t}(\bfx)} Q_{t}(\bfy, \bfx) = \sum_{i=1}^d \frac{p_{t}(\bfy)}{p_{t}(\bfx)} \delta\{\bfx^{\backslash i}, \bfy^{\backslash i}\} Q^{\mathrm{tok}}_t(y^i, x^i) \quad \text{for } \bfx \neq \bfy,}
\begin{equation}\label{eq:reverse_matrix_ntoken}
    \reversematrixntokeneq
\end{equation}
and $\backQt(\bfx, \bfx) = -\sum_{\bfy \neq \bfx} \backQt(\bfx, \bfy)$. Here, the second equality uses~\eqref{eq:Q_factor}. The Kronecker delta ensures that at most one term in the summation survives, specifically, the term where $\bfx$ and $\bfy$ differ by exactly one token. Thus, $\backQt$ inherits the same sparse structure of the forward rate $Q_t$, restricting nonzero transition rates to sequences at Hamming distance $1$. For the reverse process under a general rate, see Appendix~\ref{sec:setup_full}.
 
\paragraph{Training Formulation.} 
The reverse process therefore depends on $p_t$ only through the ratios
\begin{equation}\label{eq:score_def}
    s_t(\bfx)_{\bfy} := \frac{p_t(\bfy)}{p_t(\bfx)}, \quad \bfx \neq \bfy \in \mathcal{X},
\end{equation}
known as the discrete/concrete score~\citep{meng2022concrete}, which is analogous to the Stein score $\nabla_x \log p_t$ in continuous diffusion models~\citep{song2021scorebased}. Under~\eqref{eq:Q_factor} only Hamming-distance-$1$ pairs carry nonzero rate, so the
reverse dynamics involve $s_t(\bfx)_{\bfy}$ only at $\bfy = \xito$. In score-based discrete diffusion, the goal is to learn a score estimator $\tilde{s}: [0, T] \times \mathcal{X} \to \mathbb{R}^{d(S-1)}$ approximating~\eqref{eq:score_def}. To this end, we consider two standard training losses: the weighted score-matching (WSM) loss~\citep{meng2022concrete}
\begin{equation}\label{eq:LWSM}
    \LWSM(s, \tilde{s})=\int_0^T \half \mathbb{E}_{\bfx \sim p_t}
    \sum_{\bfy \neq \bfx}  w_t{(\bfx,\bfy)} \left( s_t(\bfx)_\bfy - \tilde{s}_t(\bfx)_\bfy \right)^2 dt,
\end{equation}
and the score entropy (SE) loss~\citep{benton2024denoising,lou2024discrete}
\begin{equation}\label{eq:LSE}
    \LSE(s, \tilde{s})=\int_0^T \mathbb{E}_{\bfx \sim p_t}
    \sum_{\bfy \neq \bfx}  w_t{(\bfx,\bfy)} \left( s_t(\bfx)_\bfy \log\left( \frac{s_t(\bfx)_\bfy}{\tilde{s}_t(\bfx)_\bfy}\right) + \tilde{s}_t(\bfx)_\bfy - s_t(\bfx)_\bfy \right) dt,
\end{equation}
where $w_t{(\bfx,\bfy)} \geq 0$ are weighting coefficients, and $\LSE$ is the Bregman divergence generated by $a \mapsto a \log a - a$. Following~\cite{lou2024discrete}, we use the transition weight $w_t{(\bfx,\bfy)}=Q_t(\bfy, \bfx)$. This choice parallels the score-matching loss in continuous
space, where it is typically weighted by the diffusion coefficient to align with the objective with the
evidence lower bound (ELBO)~\citep{song2021scorebased}.  Both objectives are intractable, as the ground-truth score depends on the unknown marginal $p_t$. In practice, one uses denoising or implicit score matching, which are tractable and equivalent to the original loss up to a constant~\citep{lou2024discrete}.
 
\paragraph{Approximate Reverse Process.} The generative process is governed by the approximate reverse-time dynamics
\newcommand{\approxreverseeq}{-\frac{d\tilde{p}_{t}}{dt} = (\tilde{Q}_{t}^{\leftarrow})^\top \tilde{p}_{t}, \quad \tilde{p}_T=p_{\rm base} \approx p_T, \quad \text{for } t \in [0, T],}
\begin{equation}\label{eq:approx_reverse}
    \approxreverseeq
\end{equation}
where $\tilde{p}$ denotes the marginal of the approximate reverse process and
\newcommand{\approxreversematrixeq}{\tilde{Q}_t^{\leftarrow}(\bfx, \bfy) = \tilde{s}_t(\bfx)_\bfy Q_{t}(\bfy, \bfx) \quad \text{for } \bfx \neq \bfy, \quad \text{and} \quad \tilde{Q}_t^{\leftarrow}(\bfx, \bfx) = -\sum_{\bfy \neq \bfx} \tilde{Q}_t^{\leftarrow}(\bfx, \bfy).}
\begin{equation}\label{eq:approx_reverse_matrix}
    \approxreversematrixeq
\end{equation}
This differs from the exact reverse process~\eqref{eq:reverse} in two ways: the true rate $\backQt$ is replaced by $\tilde{Q}_t^{\leftarrow}$, constructed using the score estimator, and the intractable $p_T$ is replaced by the prior $p_{\mathrm{base}}$.
 
\paragraph{Integral Probability Metric (IPM).} IPMs are a general class of metrics for measuring the discrepancy between two distributions $p, \tilde{p} \in \Delta(\mathcal{X})$. They are defined as
\newcommand{\IPMeq}{\gamma_{\Psi}(p, \tilde{p}) = \sup_{\psi \in \Psi} \left| \mathbb{E}_{\bfx \sim p}[\psi(\bfx)] - \mathbb{E}_{\bfx \sim \tilde{p}}[\psi(\bfx)] \right|,}
\begin{equation}\label{eq:IPM}
    \IPMeq
\end{equation}
where the admissible function class $\Psi$ determines the metric. Taking $\Psi$ to be $\{\psi : \mathcal{X} \to \mathbb{R}, \| \psi \|_\infty \leq 1 \}$, $\{ \| \psi \|_{\mathrm{Lip}} \leq 1 \}$, or $\{ \| \psi \|_{\mathcal{H}} \leq 1 \}$ gives the total variation distance, the Wasserstein-$1$ distance, and the maximum mean discrepancy, respectively, where $\|\cdot\|_\mathcal{H}$ is the norm of a reproducing kernel Hilbert space $\mathcal{H}$. We refer readers to~\cite{sriperumbudur2009integral,peyre2019computational,Birrell2022} for a more detailed overview of IPMs and their properties.
 
\paragraph{Error and Adjoint Equations.} Our analysis relies on computing the mismatch between the exact generative process~\eqref{eq:reverse} and its approximation~\eqref{eq:approx_reverse} over time. Defining the error $\lambda_t:=p_t-\tilde{p}_t$ and subtracting~\eqref{eq:approx_reverse} from~\eqref{eq:reverse}, we obtain the \emph{error equation} for $t \in [0, T]$,
\newcommand{\erroreq}{\frac{d\lambda_{t}}{dt}
    = (-Q_{t}^{\leftarrow} + \tilde{Q}_{t}^{\leftarrow})^\top p_{t} - (\tilde{Q}_{t}^{\leftarrow})^\top \lambda_{t}, \quad \text{with}\quad \lambda_T = p_T-p_{\rm base}.}
\begin{equation}\label{eq:error}
    \erroreq
\end{equation}
To quantify this discrepancy under any IPM, we extend the duality argument of \citet[Section~4.1]{mimikos2024score}, originally developed for continuous diffusion, by developing new technical tools specialized for the discrete setting. The core idea is to test~\eqref{eq:error} against a function $\phi:[0, T] \times \mathcal{X} \to \mathbb{R}$ satisfying the \emph{adjoint equation} associated with the approximate reverse dynamics~\eqref{eq:approx_reverse}, namely the \emph{Kolmogorov backward equation}
\newcommand{\KBEeq}{\frac{d\phi_t}{dt}=\tilde{Q}_{t}^{\leftarrow} \phi_t, \quad \phi_0(\bfx)=\psi(\bfx), \quad \text{for } t \in [0, T].}
\begin{equation}\label{eq:KBE}
    \KBEeq
\end{equation}
By selecting $\psi$ from the admissible function class $\Psi$ of the target metric~\eqref{eq:IPM}, we can leverage~\eqref{eq:KBE} to explicitly compute, in any IPM, how the approximation error propagates back to the initial data distribution.

\section{Main Results}\label{sec:result}
Recall that any valid rate matrix of a CTMC satisfies 
two structural properties by definition: \emph{nonnegative off-diagonal entries}, $Q_t(\bfx,\bfy) \ge 0$ for all $\bfx \ne \bfy$, and \emph{zero row sums}, $\sum_\bfy Q_t(\bfx,\bfy) = 0$ for all $\bfx$, which together conserve probability mass. Beyond these structural properties, we impose only the following mild regularity condition to keep our framework as general as possible.
\begin{assumption}[Rate-matrix Integrability]
\label{ass:rate}
$t \mapsto Q_t(\bfx,\bfy)$ is integrable on $[0,T]$ for all $(\bfx,\bfy).$
\end{assumption}
Assumption~\ref{ass:rate} ensures that the transition kernel associated with $Q_t$ is well-defined on $[0,T]$. Notably, both the structural properties and Assumption~\ref{ass:rate} are satisfied by virtually all discrete-diffusion-based generative models used in practice, underscoring the broad applicability of our theory.

\subsection{Main Theorem} 
Under the mild Assumption~\ref{ass:rate}, our main theorem bounds the discrepancy between the data distribution and the generated distribution, simultaneously for all IPMs~\eqref{eq:IPM}. Due to space constraints, we defer all proofs to the Appendix.

\begin{restatable}[IPM Bounds for Score-based Discrete Diffusion]{thm}{thmspec}\label{thm:spec}
Under Assumption \ref{ass:rate}, for any function class $\Psi= \{\psi: \mathcal{X} \to \mathbb{R} \}$, the IPM $\gamma_\Psi$ (defined in~\eqref{eq:IPM}) between the data distribution
$p_{\mathrm{data}}$ and the generated distribution $\tilde p_0$ satisfies the following. For bounds involving $\LSE$, we additionally assume $\tilde s_t(\bfx)_\bfy \in [s_t(\bfx)_\bfy/2,\,3 s_t(\bfx)_\bfy/2]$ for all $(t, \bfx, \bfy)$ in the support of $p_t Q_t$, a condition that holds when the score estimator is reasonably accurate. \\
\textbf{Case 1 (Masked rate; $Q^{\mathrm{tok}}_t=Q_t^{\rm masked}$ \eqref{eq:Q_absorb_uni}):}
\begin{align}\label{eq:IPM_absorb}
    \begin{split}
    \gamma_\Psi(p_{\mathrm{data}}, \tilde{p}_0)  &\leq 2 C_\Psi \left( \underbrace{d e^{-\| \beta \|_1}}_{\text{initialization error}}   +   \sqrt{2 d} \sqrt{\| \beta (1-p(\bfm))\|_1} \underbrace{\sqrt{ \,\LSE(s, \tilde{s}) + \tfrac{2}{3}\,\LThree(s, \tilde{s}) }}_{\text{training error}} \right),
    \end{split}
\end{align}
where $d$ is the sequence length, $\| \beta \|_1:=\int_0^T \beta(t) dt$, $\bfm:=[\mask,...,\mask]$ is the sequence of all mask tokens, $\| \beta (1-p(\bfm))\|_1 :=  \int_0^T \beta(t) (1-p_t(\bfm)) dt  $, $p_t(\bfm)$ is the marginal probability of $\bfm$ at time $t$, $\LSE$ is the score entropy loss~\eqref{eq:LSE}, and $\LThree = \mathcal{O}(|s - \tilde{s}|^3)$ is a cubic correction term with scaling \textbf{independent of the vocabulary size $S$}; its exact formulation is provided in Appendix~\ref{subsec:proofIPM};\\
\textbf{Case~2 (Uniform rate; $Q^{\mathrm{tok}}_t=Q_t^{\rm uniform}$ \eqref{eq:Q_absorb_uni}):} 
\begin{equation}
    \gamma_\Psi(p_{\mathrm{data}}, \tilde{p}_0) \leq 2 C_\Psi \left( \underbrace{d e^{-\| \beta \|_1}}_{\text{initialization error}} +  \sqrt{2d} \sqrt{\|\beta \|_1}  \underbrace{\sqrt{\,\LSE(s, \tilde{s}) + \tfrac{2}{3}\,\LThree(s, \tilde{s})}}_{\text{training error}} \right), \; \text{and}
\end{equation}
\begin{align}
\begin{split}
\gamma_\Psi(p_{\mathrm{data}}, \tilde{p}_0)  &\leq 2 C_\Psi \left( \underbrace{d e^{-\| \beta \|_1}}_{\text{initialization error}}  +  \sqrt{2d} \sqrt{\|\beta \|_1}  \underbrace{\sqrt{\LWSM(s, \tilde{s})}}_{\text{training error}} \right),
\end{split}
\end{align}
where $\LWSM$ is the weighted score-matching loss defined in~\eqref{eq:LWSM};\\
\textbf{Case~3 (General rate):} 
\begin{equation}\label{eq:general_bound_LSE}
    \gamma_\Psi(p_{\mathrm{data}}, \tilde{p}_0) \leq 2 C_\Psi \left( \underbrace{\TV(p_T, p_{\rm base})}_{\text{initialization error}} +  \sqrt{2} \sqrt{ \int_0^T \mathbb{E}_{\bfx \sim p_t}\!\left[ \mu^\mathrm{out}_t(\bfx) \right] dt }  \underbrace{\sqrt{\,\LSE(s, \tilde{s}) + \tfrac{2}{3}\,\LThree(s, \tilde{s})}}_{\text{training error}} \right), 
\end{equation}
and
\begin{equation}\label{eq:general_bound_LWSM}
    \gamma_\Psi(p_{\mathrm{data}}, \tilde{p}_0) \leq 2 C_\Psi \left( \underbrace{\TV(p_T, p_{\rm base})}_{\text{initialization error}} +  \sqrt{2} \sqrt{ \int_0^T \mathbb{E}_{\bfx \sim p_t}\!\left[ \mu^\mathrm{in}_t(\bfx) \right] dt }  \underbrace{\sqrt{\LWSM(s, \tilde{s})}}_{\text{training error}} \right),
\end{equation}
where $\mu^\mathrm{out}_t(\bfx):=\sum_{\bfy \neq \bfx} Q_t(\bfx, \bfy)$ and $\mu^\mathrm{in}_t(\bfx):=\sum_{\bfy \neq \bfx} Q_t(\bfy, \bfx)$ denote the exit rate from and entrance rate to state $\bfx$ at time $t$, respectively.
The dependence on the IPM enters only through
the single constant $C_\Psi$. The list of common IPMs and their constant $C_\Psi$ are summarized in~\Cref{table:IPM}; the result extends straightforwardly to all other IPMs.
\begin{table}[!ht]
\centering
\caption{Examples of common IPMs, their corresponding function class, and the associated constants for~\Cref{thm:spec}, \Cref{cor:ES}, and \Cref{cor:StepComplexity}.  Here, $W_1$ and MMD represent the Wasserstein-1 distance and maximum mean discrepancy, respectively. \textbf{Our results extend straightforwardly and universally to every IPM.\label{table:IPM}}}
\renewcommand{\arraystretch}{1.05}
\begin{tabular}{lll}
\toprule
IPM & Function class $\Psi$ & $C_\Psi$ \\
\midrule
Total variation (TV) & $\|\psi\|_\infty \le 1$ & $1/2$ \\
Per-position TV & $\psi(\bfx) = f(x^i),\ i \in [d], \ \|f\|_\infty \le 1$ & $1/2$ \\
Block / $k$-gram TV & $\psi$ depends on $k$ positions,\ $\|\psi\|_\infty \le 1$ & $1/2$ \\
Bounded Lipschitz & $\|\psi\|_\infty \le 1,\ \Lip_d(\psi) \le 1$ & $1$ \\
$W_1^{d_H}$ (Hamming distance) & $\Lip_{d_H}(\psi) \le 1$ & $d/2$ \\
$W_1^{d_{\mathrm{embed}}}$ & $\Lip_{d_{\mathrm{embed}}}(\psi) \le 1$ & $\diam_{d_{\mathrm{embed}}}(\mathcal{X})/2$ \\
MMD (kernel $K \le 1$) & $\|\psi\|_{\mathcal{H}_K} \le 1$ & $1$ \\
MMD (general bounded $K$) & $\|\psi\|_{\mathcal{H}_K} \le 1$ & $\sup_\bfx \sqrt{K(\bfx, \bfx)}$ \\
\bottomrule
\end{tabular}
\end{table}
\end{restatable}

A key consequence of~\Cref{thm:spec} is that the approximation error, measured by any IPM~\eqref{eq:IPM}, is bounded by score-matching objectives. Specifically, the training objectives ($\LWSM$ and $\LSE$) and the faster-vanishing cubic correction ($\LThree$) all approach zero as the score estimator converges to the true score ($\tilde{s} \to s$). This provides a rigorous theoretical justification for the standard optimization objective: \textbf{Model training directly and explicitly minimizes an upper bound on the entire family of IPMs between the generated and the target distributions.} Consequently, a globally optimal model achieves the theoretical minimum of the distributional discrepancy. 

\begin{remark}[Key Improvements]\label{rmk:key_improve} Our results improve existing theory in the following key aspects (see also~\Cref{table:comparison}). 
\begin{enumerate}
    \item \label{rmk:S_free_key} \textbf{Independence of Vocabulary Size $S$:} The coefficients of our bounds are independent of the vocabulary size $S$ under masked and uniform rates. This is a significant improvement over existing literature, where bounds typically depend on $S$ and become loose or even vacuous in large-scale applications like language modeling (e.g., $S > 10^5$ in SOTA architectures~\citep{nie2026large}). Concretely, the TV bounds of \citet{liang2026sharp,liang2026convergence} carry a prefactor of $\sqrt{S}$, which exceeds 300 when $S > 10^5$ and thus inflates the training-error term by two orders of magnitude. In contrast, our $S$-independence is a direct consequence of our novel derivations, which circumvent the $S$-dependency inherent in prior approaches.
    
    Our key techniques include: adjoint equations (see~\eqref{eq:KBE} and Appendix~\ref{subsec:proofIPM}), together with a regularity analysis (see Appendix~\ref{lemma:phi_bound}), that yield bounds for any IPM; a score–marginal cancellation technique and an exit-routing technique that achieve $S$-independence under masked transitions (see \Cref{rmk:SMC_ER} and Appendix~\ref{sec:bound_scale}); and a coupling argument that achieves $S$-independence under uniform transitions (see~\Cref{rmk:coupling} and Appendix~\ref{sec:coupling_uniform}).
    \item \textbf{Unified Derivations and Universal IPM Coverage:} Our novel applications of adjoint equations and regularity theory allow us to derive bounds for all IPMs under a unified framework. This stands in stark contrast to existing analyses, which typically focus on only one or two specific metrics; see~\Cref{table:comparison}.
    \item \label{rmk:extend_step_IPM} \textbf{Direct Extension of Existing Guarantees to Any IPM:} Our convergence results bound any IPM in terms of the score entropy loss~\eqref{eq:LSE}. Existing step complexity analyses bound the score entropy loss~\eqref{eq:LSE} under numerical simulation, and in turn use the bound to derive step complexity required for the approximation error to fall below a given threshold. Since our results are in terms of the same score entropy, existing step complexity analyses can be directly applied to our bounds, extending their guarantees to any IPM.
    \item \textbf{Agnostic to Rates and Priors:} Our framework is agnostic to the choice of rate matrix or prior distribution; it applies straightforwardly to any valid transition process satisfying Assumption~\ref{ass:rate}. This sharply departs from existing works, which are inapplicable to masked transition~\citep{zhang2025convergence,liang2026absorb,ren2025discrete,liang2026discrete,dmitriev2026efficient} or focus on one transition process~\citep{liang2026sharp}.
    \item \textbf{Mild Assumptions:} Our analysis relies solely on Assumption~\ref{ass:rate}. This is in stark contrast to existing theory; for instance, the works of \cite{zhang2025convergence,liang2026absorb,ren2025discrete,liang2026discrete,dmitriev2026efficient} rely critically on the non-singularity of $p_{\mathrm{base}}$, without which their estimates diverge---a constraint we do not impose.
\end{enumerate}
\end{remark}

\begin{remark}[Score-Marginal Cancellation and Exit-Routing]\label{rmk:SMC_ER}
At its core, our \textbf{score-marginal cancellation technique} uses a simple identity $p_t(\bfx) s_t(\bfx)_\bfy=\cancel{p_t(\bfx)} \frac{p_t(\bfy)}{\cancel{p_t(\bfx)}}=p_t(\bfy)$. Via this identity and algebraic manipulation, we can express the 
bound in Case~3 of \Cref{thm:spec} using either the entrance rate $\mu^\mathrm{in}_t$~\eqref{eq:general_bound_LWSM} 
or the exit rate $\mu^\mathrm{out}_t$~\eqref{eq:general_bound_LSE}. For more details, see Appendices~\ref{ssec:bound_scale}~and~\ref{ssec:bound_detail}.

Cases~1 and~2 can both be derived as specializations of Case~3. This allows users to pick the option that gives a tighter bound for a rate matrix at hand. For instance, under the masked rate, the scaling of $\mathbb{E}[\mu^\mathrm{out}_t]$ is independent of
$S$ whereas $\mathbb{E}[\mu^\mathrm{in}_t]$ grows as $\mathcal{O}(S)$. \textbf{Thus, using the exit rate bound in Case~1 yields an $S$-free guarantee.} See~\Cref{rmk:uniform_vs_masked} for a more detailed discussion. 
\end{remark}

\begin{remark}[Theoretical Justification for Score Entropy]
    From~\Cref{rmk:SMC_ER}, the convergence bound under the masked rate is $S$-free when the training loss $\LSE$ is used, but grows as $\mathcal{O}(S)$ when guided by $\LWSM$. Consequently, the $\LSE$-based bound is substantially tighter for language tasks, where the vocabulary size $S$ is prohibitively large. This offers a principled explanation for the strong empirical performance of models trained using the score entropy loss, complementing existing theory on the score entropy loss~\cite{lou2024discrete,shi2024simplified,ou2025your}.
\end{remark}

\begin{remark}[Coupling Technique for Uniform Rate]\label{rmk:coupling}    

    To the best of our knowledge, existing results for the uniform rate~\citep{zhang2025convergence,ren2025discrete,liang2026discrete,dmitriev2026efficient} bound the initialization error under the Kullback-Leibler (KL) divergence by using a modified log-Sobolev inequality to derive the exponential mixing of the forward process; this introduces a ${\log(S)}$ factor. Converting to TV distance (a type of IPM~\eqref{eq:IPM}) via Pinsker's inequality then yields a $\sqrt{\log(S)}$ factor; see, e.g., the proof of Proposition~2 of \cite{zhang2025convergence}.

    In contrast, our initialization error is $S$-free. This is achieved by using a novel coupling argument adapted from the classical Markovian coupling theorem~\citep[Theorem~5.2]{levin2017markov}.  See~\Cref{sec:coupling_uniform} for a detailed derivation.
\end{remark}

\subsection{Weak Error Bound}
Our main~\Cref{thm:spec} is a specialization of the following lemma, which controls the error against a single fixed test function $\psi$. For compactness, from here on we state the bounds in a single master form and collect the coefficients in tables.

\begin{restatable}[Weak Error Bound for Score-based Discrete Diffusion]{lem}{lemIPM}\label{lem:IPM}
Under the same assumptions as~\Cref{thm:spec}, the weak error between the true data distribution $p_{\mathrm{data}}$ and the generated data distribution $\tilde{p}_0$ satisfies the following. \\
\begin{equation}\label{eq:master_specialized}
  \left| \mathbb{E}_{\bfx \sim p_{\mathrm{data}}}[\psi(\bfx)] - \mathbb{E}_{\bfx \sim \tilde{p}_0}[\psi(\bfx)] \right|  \;\le\;2\, \|\psi\|_\infty \Big(\,\underbrace{\mathcal I}_{\text{init.\ error}}
  +\; \sqrt{2\,\mathcal C}\,\underbrace{\sqrt{\mathcal L}}_{\text{train.\ error}}\Big).
\end{equation}
Here, $(\mathcal I,\mathcal C,\mathcal L)$ are reported in~\Cref{tab:coeff}.
\begin{table}[!ht]
\centering
\caption{Coefficients of the convergence bounds for Lemma~\ref{lem:IPM} and Corollary~\ref{cor:ES} \label{tab:coeff}}
\begin{tabular}{llll}
\toprule
Rate & Initialization Error \ $\mathcal I$ & Coefficient \ $\mathcal C$ & Training Error $\mathcal L$ \\
\midrule
Masked  & $d\,e^{-\|\beta\|_1}$ & $d\,\|\beta(1-p(\bfm))\|_1$ & $\LSE(s, \tilde{s})+\tfrac23\LThree(s, \tilde{s})$ \\
Uniform & $d\,e^{-\|\beta\|_1}$ & $d\,\|\beta\|_1$            & $\LSE(s, \tilde{s})+\tfrac23\LThree(s, \tilde{s})$ \\
Uniform & $d\,e^{-\|\beta\|_1}$ & $d\,\|\beta\|_1$            & $\LWSM(s, \tilde{s})$ \\
General & $\TV(p_T,p_{\rm base})$ & $\int_0^T\mathbb E_{p_t}[\mu^{\rm out}_t]\,dt$ & $\LSE(s, \tilde{s})+\tfrac23\LThree(s, \tilde{s})$ \\
General & $\TV(p_T,p_{\rm base})$ & $\int_0^T\mathbb E_{p_t}[\mu^{\rm in}_t]\,dt$  & $\LWSM(s, \tilde{s})$ \\
\bottomrule
\end{tabular}
\end{table}
\end{restatable}

Two features of~\eqref{eq:master_specialized} drive our results. First, it holds for each bounded test function $\psi$. The derivation considers a fixed $\psi$, propagates it through the adjoint equation~\eqref{eq:KBE}, and never references the class $\Psi$. Second, the test function enters the right-hand side only through the scalar $\|\psi\|_\infty$; every other term is independent of it. This is a direct consequence of our regularity analysis (see \Cref{lemma:phi_bound}), which explicitly bounds the evolution of $\psi$ under the adjoint equation. Thus, taking the supremum over an admissible function class $\Psi$ in~\eqref{eq:master_specialized} yields~\Cref{thm:spec}, where $C_\Psi$ is $\sup_{\psi \in \Psi} \|\psi\|_\infty$ up to a scalar multiple.

\subsection{Error Under Early Stopping}
In early stopping, the reverse denoising process is simulated until an early stopping time $\delta>0$ rather than $t=0$. This is necessary when $p_\mathrm{data}$ does not have full support on $\mathcal{X}$. In this case, $p_t(\bfx)$ can be close to zero at $t \approx 0$ for some $\bfx$. This causes the score function~\eqref{eq:score_def} to be unbounded, resulting in numerical instability during generation and training. 

The cost of early stopping is an additional approximation error, because the output distribution approximates $p_\delta$ instead of the data distribution $p_0 = p_\mathrm{data}$. In the following, we compute the error caused by early stopping.

\begin{restatable}[Error Bounds under Early Stopping]{cor}{corES}\label{cor:ES}
Under the same assumptions as~\Cref{thm:spec}, for any function class $\Psi= \{\psi: \mathcal{X} \to \mathbb{R} \}$, the IPM $\gamma_\Psi$ (defined in~\eqref{eq:IPM}) between the data distribution
$p_{\mathrm{data}}$ and the generated distribution $\tilde p_\delta$, obtained under early-stopping time $\delta>0$, satisfies 
\begin{equation}\label{eq:master_ES}
  \gamma_\Psi(p_{\mathrm{data}}, \tilde{p}_\delta)  \;\le\;2\, C_\Psi \Big(\,\underbrace{\mathcal{E}_\delta}_{\text{early-stopping error}} + \underbrace{\mathcal I}_{\text{init.\ error}}
  +\; \sqrt{2\,\mathcal C}\,\underbrace{\sqrt{\mathcal{L}_\delta}}_{\text{early-stopping train.\ error}}\Big).
\end{equation}
Here, $(\mathcal I,\mathcal C)$ are reported in~\Cref{tab:coeff}. The early stopping error $\mathcal{E}_\delta$ is given by $d\int_0^\delta \beta(t) dt$ for the masked and uniform rates and $\TV(p_0, p_\delta)$ for a general rate. The early-stopping training error $\mathcal{L}_\delta$ is also reported in~\Cref{tab:coeff}, except that the time integration of the training losses start at time $\delta$ rather than 0. These $\delta$-truncated losses are used in practice to avoid numerical instability. See~\Cref{table:IPM} for a common list of IPM and their constant $C_\Psi$. The results extend to all other IPMs.
\end{restatable}

\subsection{Step Complexity for Uniformization} 
We derive the expected number of steps under uniformization \citep{jensen1953markoff,grassmann1977transient,de2000transient}; see~\citep[Section~2]{chen2025convergence} for an introduction of the uniformization algorithm. Since early stopping is the practically relevant setting, we build on the early-stopping bounds from Corollary~\ref{cor:ES}.

\begin{restatable}[Expected Step Complexities for Uniformization]{cor}{corStepComplexity}\label{cor:StepComplexity}
    Assume that $\tilde{s}_{t}(\mathbf{x})_\mathbf{y} \asymp s_{t}(\mathbf{x})_\mathbf{y}$ whenever $Q_t(\mathbf{y}, \mathbf{x}) > 0$, and the uniformization Poisson rate is chosen as
    $\Lambda_k := C   \sup\limits_{\substack{\mathbf{x} \in \mathcal{X} \\ t \in (t_k, t_{k+1}]}} -\tilde{Q}^{\leftarrow}_t (\mathbf{x}, \mathbf{x})$
    for some constant $C\geq 1$. Then, for any IPM $\gamma_\Psi$ and tolerance $\epsilon > 0$, the approximation error satisfies
    $
    \gamma_\Psi(p_{\mathrm{data}}, \tilde{p}_\delta) < \epsilon
    $ if we choose $\delta \asymp \frac{\epsilon}{d C_\Psi \betamax }$, $T \asymp \frac{1}{\betamin} \log \left( \frac{d C_\Psi }{\epsilon} \right) $, and $\LSE^\delta(s, \tilde{s}) = \mathcal{O}(\frac{\epsilon^2}{d C_\Psi^2})$, in which case the expected number of steps are given by
    \begin{enumerate}
        \item \textbf{Case 1 (Masked rate; $Q^{\mathrm{tok}}_t=Q_t^{\rm masked}$ \eqref{eq:Q_absorb_uni}):}
        \[
        \mathbb{E}[N] = \mathcal{O}\left(d \, \frac{\betamax}{\beta_{\mathrm{min}}} \,  \log \left( \frac{\betamax}{\betamin}\cdot \frac{d C_\Psi}{\epsilon}  \right) \right),
        \quad \text{and,}
        \]
        \item \textbf{Case 2 (Uniform rate; $Q^{\mathrm{tok}}_t=Q_t^{\rm uniform}$ \eqref{eq:Q_absorb_uni}):} if the true score further satisfies $s_t(\mathbf{x})_\mathbf{y} \lesssim \max(1, t^{-1})$ \citep[following][Assumption~4.4]{ren2025discrete}, then
        \[
        \mathbb{E}[N] = \mathcal{O} \left( \betamax \, d \, \left( \log\betamax +  \frac{1}{\betamin} \log \frac{d C_\Psi}{\epsilon} \right) \right).
        \]
    \end{enumerate}
    Here, $\betamax = \max\limits_{0 \leq t \leq T} \beta(t)$ and $\betamin = \min\limits_{0 \leq t \leq T} \beta(t)$, and the values of $C_\Psi$ are  given in Table~\ref{table:IPM}. 
\end{restatable}

Here, the choice of the Poisson rate is a standard assumption required to ensure the validity of the uniformization method for the CTMC. See, e.g., \cite{grassmann1977transient,chen2025convergence}. 

\begin{remark}[$S$-free Training Loss Requirement]
    In~\Cref{cor:StepComplexity}, our $S$-free convergence bounds yield a distinct advantage on training loss requirement over existing results. In existing works~\citep{liang2026absorb,liang2026convergence}\footnote{In \citep[Assumption~1 \& Proposition~1]{liang2026absorb} and \citep[Assumption~1 \& Proposition~A1]{liang2026convergence}, it is required that their $\mathcal{L}_\mathrm{TV}=\mathcal{O}(\sqrt{S \LSE})$ to be bounded by a given threshold. This effectively requires $\LSE$ to be $\mathcal{O}(1/S)$.}, the $S$-dependent prefactor of the training error term means the training error needs to be $\mathcal{O}(1/S)$. This imposes an often impractical requirement for large token vocabularies in modern language tasks. Since our training error prefactor is $S$-free, the required bound $\LSE^\delta(s, \tilde{s}) = \mathcal{O}(\frac{\epsilon^2}{d C_\Psi^2})$ does not tighten as $S$ grows.  
\end{remark}

As also noted in Remark~\ref{rmk:key_improve}.\ref{rmk:extend_step_IPM}, existing step complexity analyses bound the score entropy loss under numerical simulation. The same loss controls the approximation error in~\Cref{cor:ES}. Thus, \textbf{existing step complexity guarantees can be extended to any IPM} by directly applying existing derivations to our bounds in~\Cref{cor:ES}. 

In~\Cref{cor:StepComplexity}, we adapt the step complexity derivations of~\cite{liang2026absorb,ren2025discrete} for uniformization to our setting. Their derivation is originally for a time-homogeneous forward process, and we modify it to handle our time-inhomogeneous forward process~\eqref{eq:Q_absorb_uni}. See Appendix~\ref{sec:appendix_step_complexity} for the detailed derivations. 

Similarly, derivations for other algorithms, e.g., $\tau$-leaping~\citep{liang2026absorb,dmitriev2026efficient}, can also be applied to our framework to obtain step complexity. Since our goal in this paper is to propose a unified and foundational convergence framework for all IPMs, we defer the step complexity analysis for other algorithms to future work.

\section{Conclusion}
We introduced a novel adjoint-equation-based framework for the convergence analysis of discrete diffusion models. Our analysis yields the first vocabulary-size-independent bounds in any integral probability metric, valid under
singular priors and across uniform and masked transitions.
The framework applies to general rate matrices. Five novel techniques drive these improvements: use of adjoint equations in the space of observables, a regularity analysis, a coupling argument tailored to uniform transitions, and score-marginal cancellation and exit routing techniques. 
More broadly, our framework opens several directions for further study of discrete diffusion, including sharper step complexity bounds applicable to any IPM, the choice of loss functions, convergence analyses for higher-order and adaptive samplers, and approximation errors from practical implementation.

\subsection*{AI use statement}
(This section is \textbf{required} and does not count toward the page limit.)

In this work, we used generative AI tools for polishing writing.
We have not used generative AI tools for helping develop theoretical models or conceptual frameworks, formulating mathematical claims, providing critical ingredients for proving mathematical claims, assisting in the writing of proofs, proposing or refine hypotheses, designing or providing feedback on research  methodology or experiments, implementing methods, assisting with translation, and interpreting results. And generating synthetic data sets,  cleaning and reformatting dataset, supporting qualitative and thematic data analysis are not applicable to this work.

Additionally, we used generative AI tools for drafting parts of a research paper and identifying relevant literature. We have reviewed all AI-assisted work. We have checked for the correctness of the AI writing. We take responsibility for the final content of this work,
including text, claims or artifacts produced with the aid of generative AI.

\subsection*{Ethics statement}
(This section is \textbf{recommended} and does not count toward the page limit.)

The authors believe that this paper submission does not raise questions regarding the Code of Ethics.

\subsection*{Reproducibility statement}
(This section is \textbf{recommended} and does not count toward the page limit.)

This is a theory paper, and all derivations are reported.

\subsubsection*{Acknowledgments}
K. Kan and S. Osher were partially funded by the U.S. Department of Energy (DOE), Office of Science (SC), Advanced Scientific Computing Research program under award B\&R\# KJ0401010, FWP\# CC147.

X. Li was partially funded by NSF 2339678 and 2321040.

T. Sahai and M. Katsoulakis were partially funded by DARPA under Agreement No. HR00112590112. 

S. Osher was partially funded by  DARPA under grant HR0011259007,  NSF under grants 1554564 and 220827, AFOSR under MURI grant N00014-20-1-2787, and ARO under grant W911NF-24-1-0157.d

\bibliography{iclr2027_conference}
\bibliographystyle{iclr2027_conference}

\newpage

\appendix
{\centering\bfseries\Huge APPENDIX\par}   
\renewcommand{\thepart}{}
\renewcommand{\partname}{}
\part{}                  
\etocsetnexttocdepth{subsection}
\localtableofcontents

\bigskip

\section{Related Work}
\paragraph{Discrete Diffusion Models.} Discrete diffusion models were first proposed by~\citet{sohl2015deep}. A substantial line of work formulates the forward and reverse processes as discrete-time Markov chains over categorical states~\citep{austin2021structured,hoogeboom2021argmax,chen2024fast}, with the model trained by optimizing the standard variational lower bound. Within this family, multinomial diffusion~\citep{hoogeboom2021argmax} and D3PM~\citep{austin2021structured} mirror DDPM~\citep{ho2020denoising} in the continuous setting, while DNDM~\citep{chen2024fast} introduces a non-Markovian reverse process that enables accelerated sampling, corresponding to the role of DDIM~\citep{song2021denoising}.

In a parallel line of work, \citet{campbell2022continuous} formulate discrete diffusion as a continuous-time Markov chain (CTMC). This formulation allows for natural theoretical convergence analysis and enables more flexible numerical inference that is not constrained by the discretization fixed at training time. Importantly, the formulation admits an exact formula for the reverse process, in which the unknown quantity is the ratio of marginal probabilities. This ratio is the discrete analog of the score in continuous diffusion~\citep{song2021scorebased}, and is commonly referred to as the discrete or concrete score~\citep{meng2022concrete}. A major line of work has since focused on learning the discrete score. \cite{meng2022concrete} propose concrete score matching with an $L^2$ objective, \cite{sun2023scorebased} learn it via cross-entropy on conditional marginals, and \citet{lou2024discrete} introduce the score entropy loss, which improves training stability and sample quality.

Building on the CTMC framework, recent work explores alternative formulations for discrete generative modeling~\citep{zhu2026mdns,li2026neural}, and \citet{park2025jump,zhao2026informed,ren2026fast,liang2026scores} propose higher-order, adaptive, or correction-based numerical solvers to accelerate inference. In parallel, a body of CTMC-inspired but simplified models has emerged as a practical alternative. These models, particularly those based on masked diffusion, have shown promising performance on large-scale language and image generation~\citep{chang2022maskgit,chang2023muse,shi2024simplified,sahoo2024simple,zheng2025masked,kim2025train,kim2026stop,nie2026large,xu2026scheduling,zhang2026masked}, bridging the gap in accuracy and speed to their autoregressive counterparts.

\paragraph{Convergence Analysis.} 
\citet{campbell2022continuous} derive a total variation error bound for the tau-leaping sampling algorithm, but their results rely on strong assumptions and are not directly linked to a practical training objective.

\citet{chen2025convergence} perform a convergence analysis on the binary hypercube, deriving an approximation error bound in KL divergence, similar results are also available in~\cite{le2025discrete}. \citet{zhang2025convergence} extend this result to general categorical state spaces, also in KL. \citet{ren2025discrete} propose a stochastic integral framework that yields KL-divergence error bounds via path-measure arguments, while \citet{liang2026discrete} establish KL bounds through a differential-inequality argument that avoids path-measure techniques. These analyses all \emph{fundamentally rely on KL-based estimates, which diverge when the prior is singular}. As a result, \emph{their guarantees do not directly apply to masked diffusions}---the dominant paradigm for discrete diffusion language models.

Some recent works have attempted to address this gap. \citet{liang2026absorb,dmitriev2026efficient} mix the uniform and singular mask distributions as a surrogate prior, circumventing the divergent KL at the all-mask state. This construction, however, does not apply to true masked diffusions. \citet{conforti2025non} combine KL-based path-measure estimates with a separate TV control of the initialization mismatch at the fully masked state. More recently, \citet{liang2026sharp} address the KL issue by switching to TV distance for the approximation error. \emph{Nevertheless, all these results depend on the vocabulary size, and the bounds can become vacuous for modern language tasks.}

A separate line of theory studies the sampling efficiency of discrete diffusion models, often under idealized score assumptions~\cite{chen2025optimal,li2026breaking,liang2026convergence}; these questions are complementary to, but outside the scope of the present work.

\section{Problem Formulation: Full Details}\label{sec:setup_full}
\subsection{Background}

\paragraph{Notation and Preliminaries.} We denote scalars by lowercase letters ($x$) and vectors by boldface lowercase letters ($\bfx$). All vectors are assumed to be column vectors unless otherwise specified. We denote the $i$-th entry of a vector $\bfx$ by $x^i$. Specifically, a $d$-dimensional vector is given by $\bfx=[x^1 \dotsm x^d]$. We use $\xito$ to represent the vector $\bfx$ with its $i$-th entry replaced by $\hat{x}^i$. We denote $\bfx^{\backslash i}$ as the vector $\bfx$ with its $i$-th entry removed. 

For $N\in \mathbb{Z}^+$, $[N]$ denotes the set $\{1,2,\dotsm,N\}$. We denote $\bfe_i$ as the $i$-th standard basis vector, $\mathbf{1}_N$ as the $N$-dimensional vector of all ones, and $I_N$ as the $N \times N$ identity matrix. We use $\delta \{ \cdot, \cdot \}$ to represent the Kronecker delta function. The Hamming distance between $\mathbf{x}, \mathbf{y} \in \mathbb{R}^N$ is the number of indices where the entries differ: $\text{Ham}(\mathbf{x}, \mathbf{y}) = |\{i \in [N] : x^i \neq y^i\}|$, where $| \cdot|$ denotes the cardinality of a set. 

We consider data defined on a discrete state space $\mathcal{X} = [N]$, where $N$ is the total number of possible states. A probability distribution on this discrete space is characterized by a probability mass vector $p \in \Delta(\mathcal{X})$, representing the probability simplex over the $N$ possible states. 

\paragraph{Forward Process.} The forward noising process is defined as a continuous-time Markov chain (CTMC) on $\mathcal{X}$ over the time interval $[0, T]$, governed by a rate matrix (or infinitesimal generator) $Q_t \in \mathbb{R}^{N \times N}$. Specifically, the process is initialized at $t=0$ with the data distribution $p_{\mathrm{data}}$. The evolution of the marginal distribution $p_t$ is governed by the \emph{Kolmogorov forward equation (KFE)}, a linear ordinary differential equation (ODE) given by
\begin{equation}\label{eq:KFE}
    \KFEntokeneq.
\end{equation}
Here, $Q_t(x,y)$ defines the infinitesimal rate of transitioning from states $x$ to $y$. Because $Q_t$ is constrained to have nonnegative off-diagonal entries and zero row sums, the dynamics~\eqref{eq:KFE} preserve the total mass of $p_t$, ensuring it remains a valid probability distribution throughout the evolution. 

There are two common choices for $Q_t$. They are the masked (absorbing) and uniform rates and given by
\begin{equation}\label{eq:Q_absorb_uni_N}
(Q_t^{\rm masked})^\top = \beta(t) ( \bfe_{\mask} \mathbf{1}_N^\top - I_N) \quad \text{and} \quad (Q_t^{\rm uniform})^\top = \beta(t) \left(\tfrac{1}{N} \mathbf{1}_N \mathbf{1}_N^\top - I_N \right),
\end{equation}
respectively. For the masked rate, the $N$th entry corresponds to the mask token $\mask$. Consequently, $p_t$ approaches a tractable limiting distribution $p_{\mathrm{base}}$: either a Dirac distribution centered on the mask token or the uniform distribution, as $t \to \infty$. The ease of sampling from these distributions makes them ideal priors for the reverse generative process. 
 
The dynamics~\eqref{eq:KFE} can be numerically integrated by discretizing the time horizon into intervals with width $\Delta t$. Specifically, the forward transition kernel from time $t$ to $t +\Delta t$ is given by
\begin{equation}\label{eq:forward_kernel}
    p_{t+\Delta t| t}(y|x) = \delta\{x, y\} + Q_t(x, y) \Delta t + o(\Delta t),
\end{equation}
where $\delta\{\cdot, \cdot\}$ denotes the Kronecker delta, and $Q_t(x,y)$ denotes the $(x,y)$-th entry of $Q_t$. Numerical simulation of CTMCs can be broadly divided into exact path simulation such as the Gillespie’s algorithm and approximate simulation such as $\tau$-leaping method. We refer to~\cite{gillespie2001approximate,cao2004numerical,cao2006efficient,hobolth2009simulation,arns2010numerical} for a comprehensive review of these methods.

\paragraph{Reverse Process.} The forward process~\eqref{eq:KFE} admits a reverse-time process, which is also a CTMC~\citep{kelly1979reversibility}. The reverse process satisfies the reverse-time Kolmogorov equation
\begin{equation}\tag{\ref*{eq:reverse}}
    \reverseeq
\end{equation}
where $\backQt$ is the reverse-time rate matrix, and
\begin{equation}\label{eq:reverse_matrix}
    \backQt(x, y) = \frac{p_{t}(y)}{p_{t}(x)} Q_{t}(y, x) \quad \text{for } x \neq y, \quad \text{and} \quad \backQt(x, x) = -\sum_{y \neq x} \backQt(x, y).
\end{equation}
\paragraph{Training Formulation.} The ratios $s_t(x):=\left[ \frac{p_{t}(y)}{p_{t}(x)} \right]_{y \neq x} \in \mathbb{R}^{N-1}_{\geq 0}$ are known as the discrete/concrete score~\citep{meng2022concrete}, which is analogous to the Stein score $\nabla_x \log p_t$ in continuous diffusion models~\citep{song2021scorebased}. We denote the entry corresponding to the state $y$ as $s_t(x)_y :=\frac{p_{t}(y)}{p_{t}(x)} \in \mathbb{R}_{\geq 0}$. In score-based discrete diffusion, the goal is to learn a score estimator
\begin{equation}
    \tilde{s}_t (x) \approx s_t(x) = \left[ \frac{p_{t}(y)}{p_{t}(x)} \right]_{y \neq x}, \quad \text{for } t\in[0, T].
\end{equation}
To this end, we consider two standard training losses. The first is a weighted score-matching (WSM) loss~\citep{meng2022concrete}
\begin{equation}\label{eq:LWSM_single}
    \LWSM(s, \tilde{s})=\int_0^T \half \mathbb{E}_{x \sim p_t}
    \sum_{y\neq x}  w_t{(x,y)} \left( s_t(x)_y - \tilde{s}_t(x)_y \right)^2 dt,
\end{equation}
and the second is the score entropy (SE) loss~\citep{benton2024denoising,lou2024discrete}
\begin{equation}\label{eq:LSE_single}
    \LSE(s, \tilde{s})=\int_0^T \mathbb{E}_{x \sim p_t}
    \sum_{y\neq x}  w_t{(x,y)} \left( s_t(x)_y \log\left( \frac{s_t(x)_y}{\tilde{s}_t(x)_y}\right) + \tilde{s}_t(x)_y - s_t(x)_y \right) dt,
\end{equation}
where $w_t{(x,y)} \geq 0$ are weighting coefficients, and $\LSE$ is obtained from the Bregman divergence generated by the function $a \mapsto a \log a - a$. Following~\cite{lou2024discrete}, we consider the transition weight where $w_t{(x,y)}=Q_t(y, x)$. This choice parallels the score-matching loss in continuous space, where it is typically weighted by the diffusion coefficient to align with the objective with the evidence lower bound (ELBO)~\citep{song2021scorebased}.

In practice, this training objective is intractable as the ground-truth score depends on the unknown marginal distribution $p_t$. To resolve this, one typically uses denoising or implicit score matching, which are tractable objectives that are equivalent to the original loss up to a constant independent of the score estimator~\cite{lou2024discrete}.

\paragraph{Integral Probability Metric (IPM).} 
IPMs are a general class of metrics for measuring the discrepancy between two probability measures $p, \tilde{p} \in \mathcal{P}(\mathcal{X})$. They are defined as
\begin{equation}\tag{\ref*{eq:IPM}}
    \IPMeq
\end{equation}
where $\Psi$ is a suitable class of admissible functions. Different choices of $\Psi$ induce different types of IPMs. Commonly used IPMs include the total variation (TV) distance, Wasserstein-$1$ ($W_1$) distance, and maximum mean discrepancy (MMD).
These correspond to choosing $\Psi$ as $ \{\psi :  \mathcal{X} \to \mathbb{R}, \| \psi \|_\infty \leq 1 \}$, $ \{\psi :  \mathcal{X} \to \mathbb{R}, \| \psi \|_{\mathrm{Lip}} \leq 1 \}$, and $ \{\psi :  \mathcal{X} \to \mathbb{R}, \| \psi \|_{\mathcal{H}} \leq 1 \}$, respectively, where $\|\cdot\|_\mathcal{H}$ is the norm for a reproducing kernel Hilbert space (RKHS) $\mathcal{H}$. We refer readers to~\cite{sriperumbudur2009integral,peyre2019computational,Birrell2022} for a more detailed overview of IPMs and their properties.

\subsection{Problem Formulation}
The discrete state space is given by $\mathcal{X}=[S]^d$, containing sequences that factorize as $\mathbf{x}=[x^1,\dotsm ,x^d]$. This setup is common among various applications, including text token sequences~\citep{lou2024discrete}, image pixel values~\citep{campbell2022continuous}, tabular data~\citep{kotelnikov2023tabddpm}, chemical molecule graphs~\citep{vignac2023digress}, musical  notes~\citep{campbell2022continuous}, DNA sequences~\citep{schiff2025simple}, protein sequences~\citep{campbell2024generative}, and others.

\paragraph{Forward Noising Process.} The forward process $X = \{ X_t\}_{0 \leq t \leq T}$ is defined as a continuous-time Markov chain (CTMC) on $\mathcal{X}$. The distribution of $X_t$ is characterized by a probability mass vector $p_t \in \Delta(\mathcal{X})$, representing the probability simplex over the $S^d$ possible states. This distribution is governed by a Kolmogorov forward equation with rate matrix $Q_t\in \mathbb{R}^{S^d \times S^d}$ given by
\begin{equation}\tag{\ref*{eq:KFE_ntoken}}
    \KFEntokeneq.
\end{equation}

Since the dimensions of $Q_t$ scale exponentially with the sequence length $d$, direct computation or storage of the generator is intractable. In practice, $Q_t$ is structured to operate on each dimension independently via a token-level rate matrix $Q_t^\mathrm{tok} \in \mathbb{R}^{S \times S}$. In standard frameworks~\citep{campbell2022continuous,lou2024discrete,zhang2025convergence}, the nonzero off-diagonal entries of $Q_t$ are structured as follows (although our results apply to general rate matrices; see \Cref{sec:result})
\begin{equation}\tag{\ref*{eq:Q_factor}}
    \Qfactoreq
\end{equation}
where $x^i \neq \hat{x}^i \in [S]$, and $Q_t^\mathrm{tok}$ is either the masked rate matrix or uniform rate matrix~\eqref{eq:Q_absorb_uni}. Under these two choices, the asymptotic limit $p_{\mathrm{base}}$ is the Dirac distribution centered on the mask token and the uniform distribution over $\mathcal{X}$, respectively.
The structure in~\eqref{eq:Q_factor} implies that $Q_t$ is a Kronecker sum of independent generators, modeling transitions between states at Hamming distance 1 (i.e., states differing by exactly one token).

Under this structure of $Q_t$, the forward transition kernel $p_{t|s}(\bfx_t|\bfx_s):=\Pr(X_t=\bfx_t| X_s=\bfx_s)$ factorizes as $p_{t|s}(\bfx_t|\bfx_s)=\prod\limits_{i=1}^d p^i_{t|s}(x_t^i|x_s^i)$ for $0 \leq s < t \leq T$. Here, $p^i_{t|s}(x_t^i|x_s^i):=\Pr(X_t^i=x_t^i|X_s^i=x_s^i)$ is the forward transition kernel for the $i$th token governed by rate $Q_t^\mathrm{tok}$.

The factorization of transition kernel is standard in the discrete diffusion literature, as it is essential for scaling to high-dimensional state spaces for empirical implementation~\cite{austin2021structured,campbell2022continuous,campbell2024generative,lou2024discrete}. This assumption is also widely adopted in theoretical analysis~\cite{zhao2025unified,chen2025convergence,zhang2025convergence,ren2025discrete}. 

\paragraph{Reverse Denoising Process.} The forward process~\eqref{eq:KFE_ntoken} admits a reverse-time (generative) process 
governed by the following Kolmogorov equation~\citep{kelly1979reversibility}
\begin{equation}\label{eq:reverse_ntoken}
    -\frac{dp_{t}}{dt} = (Q_{t}^{\leftarrow})^\top p_{t}, \quad \text{for } t \in [0, T],
\end{equation}
where the reverse rate matrix $Q_{t}^{\leftarrow}$ is given by
\begin{equation}\tag{\ref*{eq:reverse_matrix_ntoken}}
    \reversematrixntokeneq
\end{equation}
and $\backQt(\bfx, \bfx) = -\sum_{\bfy \neq \bfx} \backQt(\bfx, \bfy).$
In the second step of~\eqref{eq:reverse_matrix_ntoken}, we used the definition of $Q_t$~\eqref{eq:Q_factor}. The Kronecker delta ensures that at most one term in the summation survives---specifically, the term where $\bfx$ and $\bfy$ differ by exactly one token. Thus, the reverse rate matrix $\backQt$ inherits the same sparse structure as the forward rate $Q_t$, restricting non-zero transition rates to sequences at Hamming distance 1. 

Hence, we consider the discrete score $s_t$ defined for these sequences
\begin{equation*}
    \sito := \frac{p_t(\xito)}{p_t(\bfx)} \quad \text{for } \bfx \in \mathcal{X}, \, i\in [d], \, x^i\neq \hat{x}^i \in [S].
\end{equation*}
The score estimator $\tilde{s}: [0, T] \times \mathcal{X} \to \mathbb{R}^{d(S-1)}$ can then be obtained by minimizing $\LWSM(s, \tilde{s})$ defined in~\eqref{eq:LWSM_single} or $\LSE(s, \tilde{s})$ defined in~\eqref{eq:LSE_single}.

\paragraph{Approximate Reverse Process.}
The generative process is governed by the approximate reverse-time dynamics
\begin{equation}\tag{\ref*{eq:approx_reverse}}
 \approxreverseeq
\end{equation}
where the approximate reverse rate matrix $\tilde{Q}_t^{\leftarrow}$ is defined by 
\begin{equation}\tag{\ref*{eq:approx_reverse_matrix}} 
    \approxreversematrixeq
\end{equation}
It is important to note that the approximate reverse process differs from the exact reverse process~\eqref{eq:reverse_ntoken} by two key aspects. First, the true reverse rate $\backQt$ is replaced by the approximate $\tilde{Q}_t^{\leftarrow}$ constructed using the score estimator through~\eqref{eq:approx_reverse_matrix}. Second, since the ground truth terminal distribution $p_T$ is intractable, the prior $p_{\mathrm{base}}$ is used as a starting distribution for the reverse process. Here, $\tilde{p}$ denotes the marginal distribution of the approximate reverse process.

\paragraph{Error and Adjoint Equations.} Our analysis relies on computing the mismatch between the exact generative process~\eqref{eq:reverse} and its approximation~\eqref{eq:approx_reverse} over time. Defining the error $\lambda_t:=p_t-\tilde{p}_t$ and subtracting~\eqref{eq:approx_reverse} from~\eqref{eq:reverse}, we obtain the \emph{error equation} for $t \in [0, T]$,
\begin{equation}\tag{\ref*{eq:error}}  
    \erroreq
\end{equation}
To quantify this discrepancy under any IPM, we extend the duality argument of \citet[Section~4.1]{mimikos2024score} originally developed for continuous diffusion, by developing new technical tools specialized for the discrete setting. The core idea is to test~\eqref{eq:error} against a function $\phi:[0, T] \times [S]^d \to \mathbb{R}$ that satisfies the \emph{adjoint equation} associated with the approximate reverse dynamics~\eqref{eq:approx_reverse}, namely the \emph{Kolmogorov backward equation} 
\begin{equation}\tag{\ref*{eq:KBE}}   
    \KBEeq
\end{equation}
By selecting $\psi$ from the admissible function class $\Psi$ of the target metric~\eqref{eq:IPM}, we can leverage~\eqref{eq:KBE} to explicitly compute, in any IPM, how the approximation error propagates back to the initial data distribution.

\section{Proof Sketch: Convergence Analysis via Adjoint Equations}
In this section, we highlight the key technical innovation underlying our main \Cref{lem:IPM} and Corollaries~\ref{cor:TV}-\ref{thm:spec}. Our core strategy involves (1) constructing test functions for the IPMs~\eqref{eq:IPM} via an adjoint equation, (2) estimating the regularity of the test functions, and (3) bounding the approximation error for any IPM by specializing the admissible function class accordingly. The detailed derivations are given in Appendix~\ref{appendix:proof_main}.

\paragraph{Adjoint Equation.} We first compute the mismatch between the exact generative process~\eqref{eq:reverse} and its approximation~\eqref{eq:approx_reverse} over time. To this end, defining the error $\lambda_t:=p_t-\tilde{p}_t$ and subtracting~\eqref{eq:approx_reverse} from~\eqref{eq:reverse}, we obtain the \emph{error equation} for $t \in [0, T]$,
\begin{equation}\tag{\ref*{eq:error}}
    \erroreq
\end{equation}
To quantify this discrepancy under any IPM, we extend the duality argument of \citet[Section~4.1]{mimikos2024score} originally developed for continuous diffusion, by developing new technical tools specialized for the discrete setting. The core idea is to test~\eqref{eq:error} against a function $\phi:[0, T] \times [S]^d \to \mathbb{R}$ that satisfies the \emph{adjoint equation} associated with the approximate reverse dynamics~\eqref{eq:approx_reverse}, namely the \emph{Kolmogorov backward equation} 
\begin{equation}\tag{\ref*{eq:KBE}} 
    \KBEeq
\end{equation}

In particular, we integrate the error equation~\eqref{eq:error} against the test function $\phi_t$ in space, integrate in time and apply integration by parts to move the time derivative to $\phi_t$, yielding
\begin{equation}\label{eq:derivation_IBP}
\int_0^T \langle \frac{d}{dt} \lambda_t, \phi_t \rangle \, dt = \langle \lambda_T, \phi_T \rangle - \langle \lambda_0, \phi_0 \rangle - \int_0^T \langle \lambda_t, \frac{d}{dt} \phi_t \rangle \, dt.
\end{equation}
Substituting~\eqref{eq:error} and~\eqref{eq:KBE} into~\eqref{eq:derivation_IBP}, the $\lambda_t$-dependent contributions cancel via the adjointness identity $\langle (\tilde{Q}_t^{\leftarrow})^\top \lambda_t, \phi_t \rangle = \langle \lambda_t, \tilde{Q}_t^{\leftarrow} \phi_t \rangle$ -- which is precisely why $\tilde{Q}_t^{\leftarrow}$ is the choice of generator in~\eqref{eq:KBE} -- leaving only the $p_t$-dependent rate-matrix difference:
\begin{equation}
     \langle \lambda_0, \psi \rangle = \langle \lambda_0, \phi_0 \rangle = \langle \lambda_T, \phi_T \rangle + \int_0^T \left\langle (Q_{t}^{\leftarrow} - \tilde{Q}_{t}^{\leftarrow})^\top p_{t}  , \phi_t \right\rangle dt.
\end{equation}
Taking absolute values and plugging in $\lambda_t=p_t-\tilde{p}_t$, we obtain an upper bound for the approximation error
\begin{align}\label{eq:upper_bound_approximate}
\begin{split}
     &\left| \mathbb{E}_{\bfx \sim p_{\mathrm{data}}}[\psi(\bfx)] - \mathbb{E}_{\bfx \sim \tilde{p}_0}[\psi(\bfx)] \right| \\
     &\leq \left| \langle \lambda_T, \phi_T \rangle \right| +   \int_0^T \left|\left\langle (Q_{t}^{\leftarrow} - \tilde{Q}_{t}^{\leftarrow})^\top p_{t}  , \phi_t \right\rangle \right| dt  \\
     &\leq \| \phi_T \|_\infty \sum_{\bfx\in\mathcal{X}} \left|p_T - p_{\rm base} \right| (\bfx) + 2 \int_0^T \| \phi_t \|_\infty \sum_{\substack{\bfx,\bfy\in\mathcal{X} \\ \bfx\neq \bfy}}   p_t(\bfx) \left|
      (Q_{t}^{\leftarrow} - \tilde{Q}_{t}^{\leftarrow}) (\bfx, \bfy) \right| dt \\
      &= 2 \| \phi_T \|_\infty \TV(p_T, p_{\rm base}) + 2 \int_0^T \| \phi_t \|_\infty \sum_{\substack{\bfx,\bfy\in\mathcal{X} \\ \bfx\neq \bfy}}  p_t(\bfx) Q_t(\bfy, \bfx) \left| {s}_t(\bfx)_\bfy - \tilde{s}_t(\bfx)_\bfy  \right|  dt,
\end{split}
\end{align}
where $\| \phi_t \|_\infty := \max_{\bfx \in \mathcal{X}} |\phi_t(\bfx)|$, and in the last step, we used the reverse rates formulations~\eqref{eq:reverse_matrix} and \eqref{eq:approx_reverse_matrix}, as well as the definition of TV distance. Importantly, the integration term contains the score error $\left| {s}_t(\bfx)_\bfy - \tilde{s}_t(\bfx)_\bfy  \right|$, which we represent using the standard training losses~\eqref{eq:LWSM}--\eqref{eq:LSE} through our novel score-marginal cancellation technique; see Lemmas~\ref{lemma:error_bound_LSE} and \ref{lemma:error_bound_LWSM} and their proofs for more details. As a result, our bounds are expressed directly in terms of the standard training losses.

\paragraph{Regularity of Test Functions.} To make the bound~\eqref{eq:upper_bound_approximate} explicit, it remains to control the test function sup-norm $\| \phi_t \|_\infty$. Recall that $\phi_t$ is the time evolution of $\psi$ under the adjoint equation~\eqref{eq:KBE}, with $\phi_0=\psi$. Importantly, we show in \Cref{lemma:phi_bound} that the adjoint dynamics is non-expansive in the sup-norm: the solution $\phi_t$ to the adjoint equation satisfies $\| \phi_t \|_\infty \leq \| \psi \|_\infty$ for all $t \in [0,T]$. Thus, plugging in the contraction result into \eqref{eq:upper_bound_approximate}, we obtain
\begin{align}\label{eq:upper_bound_approximate_explicit}
\begin{split}
     &\left| \mathbb{E}_{\bfx \sim p_{\mathrm{data}}}[\psi(\bfx)] - \mathbb{E}_{\bfx \sim \tilde{p}_0}[\psi(\bfx)] \right| \\
     &\leq 2 \| \psi \|_\infty \TV(p_T, p_{\rm base}) + 2 \| \psi \|_\infty \int_0^T \sum_{\substack{\bfx,\bfy\in\mathcal{X} \\ \bfx\neq \bfy}}  p_t(\bfx) Q_t(\bfy, \bfx) \left| {s}_t(\bfx)_\bfy - \tilde{s}_t(\bfx)_\bfy  \right|  dt,
\end{split}
\end{align}
effectively removing the only remaining unknown. Intuitively, the non-expansiveness holds because $\phi_t$ is an expectation of $\psi$ over trajectories of the CTMC associated with the adjoint equation, and an expectation never exceeds the maximum. In our proof of~\Cref{lemma:phi_bound}, we make this rigorous by using a Feynman-Kac representation.

\paragraph{Specialization to Any IPM.} The bound~\eqref{eq:upper_bound_approximate_explicit} is valid for any admissible function $\psi$, and its right-hand side depends on $\psi$ only through the multiplicative factor $\| \psi \|_\infty$. Taking the supremum over the admissible class $\Psi$ of the target IPM~\eqref{eq:IPM}, the left-hand side becomes the IPM $\gamma_\Psi(p_{\mathrm{data}}, \tilde{p}_0)$. Since $\| \psi \|_\infty$ is decoupled from the other terms, those terms are unaffected by the supremum. And we have a prefactor $C_\Psi$ which is a constant multiple of $\sup_{\psi \in \Psi} \| \psi \|_\infty$. Since the state space $\mathcal{X} = [S]^d$ is finite, any admissible function $\psi$ is bounded and $C_\Psi$ is finite for every IPM. A single bound thus specializes to every IPM, with only the constant $C_\Psi$ changing.

\section{Proof of Main Theorems}\label{appendix:proof_main}
We restate and prove the main theorems.
\subsection{Proof of \texorpdfstring{\Cref{lem:IPM}}{Theorem \ref*{lem:IPM}}}\label{subsec:proofIPM}
\lemIPM*

Here, the cubic correction term is given by
\newif\iffirstfcases
\firstfcasestrue
\newcommand{\myfcases}{%
  \iffirstfcases
{
\small
\begin{subnumcases}{\LThree(s, \tilde{s}) = \label{eq:LThree}}
\int_0^T \beta(t) (1-p_t(\bfm)) \mathbb{E}_{\bfy \sim p_t(\bfy|\bfy \neq \bfm)} \left[ \sum_{\substack{\substack{\bfx \in \Nmasked(\bfy)}}} \frac{\,|s_t(\bfx)_\bfy - \tilde{s}_t(\bfx)_\bfy|^3}{ (s_t(\bfx)_\bfy)^3} \right] dt \label{eq:L3_masked}, \\ \text{under the masked rate,} \nonumber \\
\int_0^T \mathbb{E}_{\bfx \sim p_t} \left[ \frac{\beta(t)}{S} \sum_{\substack{\bfy \in\mathcal{X} \\ d_H(\bfy,\bfx)=1}}   \left( \frac{\,|s_t(\bfx)_\bfy - \tilde{s}_t(\bfx)_\bfy|^3}{(s_t(\bfx)_\bfy)^2} \right) \right]  dt, \label{eq:L3_uniform} \\ \text{under the uniform rate,} \nonumber \\
    \int_0^T \mathbb{E}_{\bfx \sim p_t} \left[ \sum_{\substack{\bfy\in\mathcal{X} \\ \bfy\neq \bfx}}  Q_t(\bfy, \bfx) \left( \frac{\,|s_t(\bfx)_\bfy - \tilde{s}_t(\bfx)_\bfy|^3}{(s_t(\bfx)_\bfy)^2} \right) \right]  dt 
    = \int_0^T \mathbb{E}_{\bfy \sim p_t} \left[ \sum_{\substack{\bfx\in\mathcal{X} \\ \bfx\neq \bfy}}  Q_t(\bfy, \bfx) \left( \frac{\,|s_t(\bfx)_\bfy - \tilde{s}_t(\bfx)_\bfy|^3}{ (s_t(\bfx)_\bfy)^3} \right)  \right] dt,\label{eq:L3_general} \\
    \text{under a general rate.} \nonumber
\end{subnumcases}}
\firstfcasesfalse
  \else
{
\small
\begin{subnumcases}{\LThree(s, \tilde{s}) = }
\int_0^T \beta(t) (1-p_t(\bfm)) \mathbb{E}_{\bfy \sim p_t(\bfy|\bfy \neq \bfm)} \left[ \sum_{\substack{\substack{\bfx \in \Nmasked(\bfy)}}} \frac{\,|s_t(\bfx)_\bfy - \tilde{s}_t(\bfx)_\bfy|^3}{ (s_t(\bfx)_\bfy)^3} \right] dt \tag{\ref{eq:L3_masked}}, \\ \text{under the masked rate,} \nonumber \\
\int_0^T \mathbb{E}_{\bfx \sim p_t} \left[ \frac{\beta(t)}{S} \sum_{\substack{\bfy \in\mathcal{X} \\ d_H(\bfy,\bfx)=1}}   \left( \frac{\,|s_t(\bfx)_\bfy - \tilde{s}_t(\bfx)_\bfy|^3}{(s_t(\bfx)_\bfy)^2} \right) \right]  dt, \tag{\ref{eq:L3_uniform}} \\ \text{under the uniform rate,} \nonumber \\
    \int_0^T \mathbb{E}_{\bfx \sim p_t} \left[ \sum_{\substack{\bfy\in\mathcal{X} \\ \bfy\neq \bfx}}  Q_t(\bfy, \bfx) \left( \frac{\,|s_t(\bfx)_\bfy - \tilde{s}_t(\bfx)_\bfy|^3}{(s_t(\bfx)_\bfy)^2} \right) \right]  dt 
    = \int_0^T \mathbb{E}_{\bfy \sim p_t} \left[ \sum_{\substack{\bfx\in\mathcal{X} \\ \bfx\neq \bfy}}  Q_t(\bfy, \bfx) \left( \frac{\,|s_t(\bfx)_\bfy - \tilde{s}_t(\bfx)_\bfy|^3}{ (s_t(\bfx)_\bfy)^3} \right)  \right] dt,\tag{\ref{eq:L3_general}} \\
    \text{under a general rate.} \nonumber
\end{subnumcases}}
   \fi
}
\myfcases
where $\Nmasked(\bfy)$ is the successor set of $\bfy$ (i.e. containing sequences $\bfy$ can transit to), defined by
\begin{equation}\label{eq:successor}
    \Nmasked(\bfy)=\left\{ \bfx \in \mathcal{X} \, \middle| \, \exists i \in [d], \bfx^{\backslash i}=\bfy^{\backslash i} , x^i=\mask, y^i \neq \mask \right\}.
    \end{equation}

\begin{remark}[Observables, not metrics, are the primary objects of our analysis]
The first output of our analysis is a single-observable guarantee (the Lemma
above): for every bounded $\psi:\mathcal X\to\mathbb R$,
\eqref{eq:master_specialized} controls
$|\mathbb E_{p_{\rm data}}\psi-\mathbb E_{\tilde p_0}\psi|$ by $\|\psi\|_\infty$
times the approximation-error terms bounded in \Cref{thm:spec}, with $\psi$
entering only through $\|\psi\|_\infty$. The IPM bound of \Cref{thm:spec} is the
corollary obtained by taking a supremum over a class $\Psi$. In language
modeling~\cite{lou2024discrete} one wants this supremum; in most other
applications the success criterion is itself a fixed bounded observable: motif and
$k$-mer frequencies for DNA~\cite{schiff2025simple}, the designable fraction and
mean pLDDT for proteins~\cite{campbell2024generative}, low-order marginals and
moments for tabular data~\cite{kotelnikov2023tabddpm}, and
maximum-mean-discrepancy metrics---IPMs with a bounded kernel---for
images~\cite{campbell2022continuous} and molecular graphs~\cite{vignac2023digress}.
Each has $\|\psi\|_\infty=O(1)$, so our bound applies directly, without ever
forming a full metric. 

A subtlety worth stating: the Fr\'echet-type scores common in these fields (FID,
FCD, and sequence analogues) are not IPMs---they are $W_2$ distances between
Gaussians fitted in a feature space---so they do not arise from the supremum in
\Cref{thm:spec} directly. They are, however, smooth functions of finitely many
\emph{bounded} observables: the first two moments $\mathbb E\phi,\ \mathbb
E[\phi\phi^\top]$ of a fixed feature map $\phi$. Our single-observable estimate
controls each such moment $S$-freely, so FID and its analogues are governed by our
guarantee up to the local Lipschitz modulus of the Fr\'echet functional and the
($S$-independent) feature dimension, degrading only near covariance degeneracy.
The observable Lemma  thus reaches \textbf{beyond IPMs}: an IPM is a single \emph{linear}
observable, a Fr\'echet score a \emph{nonlinear} functional of moment-observables,
and both are controlled.
\end{remark}

\begin{remark}[Details of Exit-Rate Routing Technique]\label{rmk:uniform_vs_masked}
Writing the Kolmogorov forward equation~\eqref{eq:KFE_ntoken} entry-wise 
at state $\bfx$ yields the \emph{flow balance}
\[
  \frac{dp_t(\bfx)}{dt} 
  \;=\; \underbrace{\sum_{\bfy \neq \bfx} Q_t(\bfy, \bfx)\,p_t(\bfy)}_{\text{actual inflow}} 
  \;-\; \underbrace{\mu^\mathrm{out}_t(\bfx)\,p_t(\bfx)}_{\text{actual outflow}},
\]
in which the jump rate $\mu^\mathrm{out}_t(\bfx) = -Q_t(\bfx, \bfx)$ governs the 
outflow and the entrance rate $\mu^\mathrm{in}_t(\bfx) = \sum_{\bfy \neq \bfx} Q_t(\bfy, \bfx)$ 
governs the inflow. Both are intrinsic to the generator $Q_t$: $\mu^\mathrm{out}_t$ 
is the negative diagonal, $\mu^\mathrm{in}_t$ is the column-off-diagonal sum.

\smallskip
\noindent Under the \emph{uniform rate}~\eqref{eq:Q_absorb_uni}, $Q_t^{\rm tok}$ 
is symmetric and
\[
  \mu_t^{\rm out, tok}(x) \;=\; \mu_t^{\rm in, tok}(x) \;=\; \beta(t)\bigl(1 - \tfrac{1}{S}\bigr)
  \quad \text{for every } x \in [S],
\]
so the two lemmas give the same scaling.

\smallskip
\noindent Under the \emph{masked rate}~\eqref{eq:Q_absorb_uni}, $Q_t^{\rm tok}$ 
is asymmetric and the two rates differ:
\[
  \mathbb{E}_{p_t^{\rm tok}}\!\bigl[\mu_t^{\rm out, tok}(x)\bigr] 
  \;=\; \beta(t)\bigl(1 - p_t^{\rm tok}(\mask)\bigr)
\]
is $S$-independent, whereas
\[
  \mathbb{E}_{p_t^{\rm tok}}\!\bigl[\mu_t^{\rm in, tok}(x)\bigr] 
  \;=\; (S-1)\,\beta(t)\,p_t^{\rm tok}(\mask)
\]
carries a factor of $S-1$ --- the per-token vocabulary size.

\noindent Accordingly, in the proof of~\Cref{lem:IPM} we use the SE route 
for the masked case to obtain an 
$S$-independent prefactor, and the simpler WSM route 
(\Cref{lemma:error_bound_LWSM}) for the uniform case, where its 
prefactor coincides with the SE one. Of course, either bound is mathematically valid, but only the SE 
route yields an $S$-independent prefactor under masked dynamics.
\end{remark}

\begin{remark}[$S$-free scaling of $\LThree$]
We note that $|\Nmasked(\bfy)| \leq d$ in~\eqref{eq:successor}, where $d$ is the sequence length. Overall, in~\eqref{eq:L3_masked}, the outer expectation
$\mathbb{E}_{\bfy \sim p_t(\bfy \mid \bfy \ne \bfm)}[\,\cdot\,]$ is a probability
average; its support has size $S^d - 1$, but the integrand averaged is
the at-most-$d$-term inner sum of cubic errors, with no count
of states entering per term. Thus, the scaling of $\LThree$ is independent of the vocabulary size $S$ under the masked rate.

Under the uniform rate, the outer expectation $\mathbb{E}_{\bfy \sim p_t}$ of~\eqref{eq:L3_uniform} is again a probability average. For each $\bfy$, the inner summation $\sum_{d_H(\bfx,\bfy)=1}$ runs over the $d(S-1)$ neighbors satisfying $\distH{\bfy}{\bfx} = 1$; together with the factor $\frac{1}{S}$, its magnitude is on the order of $\frac{d(S-1)}{S}<d$. Thus the scaling of $\LThree$ is again independent of $S$.
\end{remark}

\begin{proof}[Proof of~\Cref{lem:IPM}]
\begin{align*}
    \langle \lambda_0, \psi \rangle -\langle \lambda_T, \phi_T \rangle &= \langle \lambda_0, \phi_0 \rangle -\langle \lambda_T, \phi_T \rangle \\
    &= -\int_0^T \frac{d}{dt} \langle \lambda_t, \phi_t \rangle dt \\
    &= - \int_0^T  \langle \frac{d\lambda_t}{dt} , \phi_t \rangle +  \langle \lambda_t, \frac{d\phi_t}{dt}  \rangle dt \\
    &=  \int_0^T \left\langle (Q_{t}^{\leftarrow} - \tilde{Q}_{t}^{\leftarrow})^\top p_{t} + (\tilde{Q}_{t}^{\leftarrow})^\top \lambda_{t} , \phi_t \right\rangle -  \left\langle \lambda_t, \tilde{Q}_{t}^{\leftarrow} \phi_t  \right\rangle dt \quad \text{by~\eqref{eq:error} and~\eqref{eq:KBE}}, \\
    &= \int_0^T \left\langle (Q_{t}^{\leftarrow} - \tilde{Q}_{t}^{\leftarrow})^\top p_{t}  , \phi_t \right\rangle dt \\
    &= \int_0^T \sum_{\substack{\bfy\in\mathcal{X}}} \phi(t, \bfy) \left[ (Q_{t}^{\leftarrow} - \tilde{Q}_{t}^{\leftarrow})^\top p_t \right] (\bfy) dt \\
    &= \int_0^T \sum_{\substack{\bfy\in\mathcal{X}}} \sum_{\substack{\bfx\in\mathcal{X}}} \phi(t, \bfy) \left[ (Q_{t}^{\leftarrow} - \tilde{Q}_{t}^{\leftarrow}) \right] (\bfx, \bfy) p_t(\bfx) dt \\
    &= \int_0^T \sum_{\substack{\bfx\in\mathcal{X}}} \sum_{\substack{\bfy\in\mathcal{X}}} \phi(t, \bfy) \left[ (Q_{t}^{\leftarrow} - \tilde{Q}_{t}^{\leftarrow}) \right] (\bfx, \bfy) p_t(\bfx) dt, \quad \text{by Fubini's theorem,} \\
    &= \int_0^T \sum_{\substack{\bfx\in\mathcal{X}}} p_t(\bfx) \sum_{\substack{\bfy\in\mathcal{X}}} \phi(t, \bfy) \left[ (Q_{t}^{\leftarrow} - \tilde{Q}_{t}^{\leftarrow}) \right] (\bfx, \bfy)  dt \\
    &= \int_0^T \sum_{\substack{\bfx\in\mathcal{X}}} p_t(\bfx) \sum_{\substack{\bfy\in\mathcal{X}}} \left(\phi(t, \bfy)-\phi(t, \bfx) \right)\left[ (Q_{t}^{\leftarrow} - \tilde{Q}_{t}^{\leftarrow}) \right] (\bfx, \bfy)  dt, \\
    \intertext{where we used that $Q_{t}^{\leftarrow}$ and $\tilde{Q}_{t}^{\leftarrow}$'s rows sum to zero,}
    &= \int_0^T \sum_{\substack{\bfx\in\mathcal{X}}} p_t(\bfx) \sum_{\substack{\bfy\in\mathcal{X} \\ \bfy\neq \bfx}}  \left(\phi(t, \bfy)-\phi(t, \bfx) \right)\left[ (Q_{t}^{\leftarrow} - \tilde{Q}_{t}^{\leftarrow}) \right] (\bfx, \bfy)  dt, \\
    \intertext{as $\phi(t, \bfy=\bfx)-\phi(t, \bfx)=0$,}
    &= \int_0^T \sum_{\substack{\bfx\in\mathcal{X}}} \sum_{\substack{\bfy\in\mathcal{X} \\ \bfy\neq \bfx}} p_t(\bfx) \left(\phi(t, \bfy)-\phi(t, \bfx) \right)\left[ (Q_{t}^{\leftarrow} - \tilde{Q}_{t}^{\leftarrow}) \right] (\bfx, \bfy)  dt \\
    &= \int_0^T \sum_{\substack{\bfx\in\mathcal{X}}} \sum_{\substack{\bfy\in\mathcal{X} \\ \bfy\neq \bfx}}  p_t(\bfx) \left(\phi(t, \bfy)-\phi(t, \bfx) \right) Q_t(\bfy, \bfx) \left( {s}_t(\bfx)_\bfy - \tilde{s}_t(\bfx)_\bfy \right)  dt ,\quad \text{by~\eqref{eq:reverse_matrix} and \eqref{eq:approx_reverse_matrix}.}\\
\end{align*}

Thus, we have
\begin{align}
\langle \lambda_0, \psi \rangle -\langle \lambda_T, \phi_T \rangle  &=  \int_0^T \sum_{\substack{\bfx\in\mathcal{X}}} \sum_{\substack{\bfy\in\mathcal{X} \\ \bfy\neq \bfx}}  p_t(\bfx) \left(\phi(t, \bfy)-\phi(t, \bfx) \right) Q_t(\bfy, \bfx) \left( {s}_t(\bfx)_\bfy - \tilde{s}_t(\bfx)_\bfy \right)  dt \nonumber \\
    \langle \lambda_0, \psi \rangle &= \underbrace{\langle \lambda_T, \phi_T \rangle }_{\cI} + \underbrace{\int_0^T \sum_{\substack{\bfx\in\mathcal{X}}} \sum_{\substack{\bfy\in\mathcal{X} \\ \bfy\neq \bfx}}  p_t(\bfx) \left(\phi(t, \bfy)-\phi(t, \bfx) \right) Q_t(\bfy, \bfx) \left( {s}_t(\bfx)_\bfy - \tilde{s}_t(\bfx)_\bfy \right)  dt }_{\cII}. \label{eq:ineq_eval_1}
\end{align}

For the first term \textcircled{I}, it can be bounded as
\begin{align*}
    \langle \lambda_T, \phi_T \rangle &= \langle p_T - p_{\rm base}, \phi_T \rangle \\
    &= \sum_{\substack{\bfx\in\mathcal{X}}}  \left[p_T - p_{\rm base} \right] (\bfx) \cdot \phi_T(\bfx) \\
    &\leq \| \phi_T \|_\infty \sum_\bfx \left|p_T - p_{\rm base} \right| (\bfx) \\
    &\leq \| \psi \|_\infty \sum_{\substack{\bfx\in\mathcal{X}}}  \left|p_T - p_{\rm base} \right| (\bfx), \quad 
    \text{by Lemma~\ref{lemma:phi_bound}, } \\
    &= 2 \| \psi \|_\infty \TV(p_T, p_{\rm base}).
\end{align*}

For the second term \textcircled{II}, we bound it differently, depending on the choice of the rate matrix $Q_t^{\mathrm{tok}}$ in~\eqref{eq:reverse_matrix_ntoken}. \\
\textbf{Case 1 (Bound under masked rate using $\LSE$):} 
\begin{align*}
    &\int_0^T \sum_{\substack{\bfx\in\mathcal{X}}} \sum_{\substack{\bfy\in\mathcal{X} \\ \bfy\neq \bfx}} p_t(\bfx) \left(\phi(t, \bfy)-\phi(t, \bfx) \right) Q_t(\bfy, \bfx) \left( s_t(\bfx)_\bfy - \tilde{s}_t(\bfx)_\bfy \right)  dt \\
    \quad & \leq 2\| \psi \|_\infty \left( \int_0^T \sum_{\substack{\bfy\in\mathcal{X}}} \sum_{\substack{\bfx\in\mathcal{X} \\ \bfx\neq \bfy}}  Q_t(\bfy, \bfx) p_t(\bfy)   dt \right)^{\frac{1}{2}} \left( 2\,\LSE(s, \tilde{s}) + \tfrac{4}{3}\,\LThree(s, \tilde{s}) \right)^{\frac{1}{2}}, \quad \text{by~\Cref{lemma:error_bound_LSE},}\\
    \quad & = 2\| \psi \|_\infty \left( \int_0^T \sum_{\substack{\bfy\in\mathcal{X} }} \sum_{\substack{\bfx: \bfx \neq \bfy \\ Q_t(\bfy, \bfx)\neq 0}}  Q_t(\bfy, \bfx) p_t(\bfy)   dt \right)^{\frac{1}{2}} \left( 2\,\LSE(s, \tilde{s}) + \tfrac{4}{3}\,\LThree(s, \tilde{s}) \right)^{\frac{1}{2}} \\
    \quad & = 2\| \psi \|_\infty \left( \int_0^T \sum_{\substack{\bfy \in \mathcal{X} \\ \bfy \neq \bfm}} \sum_{\substack{\substack{\bfx \in \Nmasked(\bfy)}}}   Q^{\mathrm{tok}}_t(y^i, \mask) p_t(\bfy)    dt \right)^{\frac{1}{2}} \left( 2\,\LSE(s, \tilde{s}) + \tfrac{4}{3}\,\LThree(s, \tilde{s}) \right)^{\frac{1}{2}}, \\
    \intertext{where we used the definition of $Q_t$~\eqref{eq:Q_factor} under masked transition~\eqref{eq:Q_absorb_uni}, $\bfm=[\mask,...,\mask]$ is the sequence of all mask tokens, 
    $$
    \Nmasked(\bfy)=\left\{ \bfx \in \mathcal{X} \, \middle| \, \exists i \in [d], \bfx^{\backslash i}=\bfy^{\backslash i} , x^i=\mask, y^i \neq \mask \right\}
    $$
    is the successor set of $\bfy$, and $i$ denotes the unique index where $\bfx$ and $\bfy$ differ,}
    \quad & \leq 2\| \psi \|_\infty \left( \int_0^T \sum_{\substack{\bfy \in \mathcal{X} \\ \bfy \neq \bfm}} \underbrace{\sum_{\substack{\substack{\bfx \in \Nmasked(\bfy)}}}}_{\leq d \text{ terms}}   \beta(t) p_t(\bfy)   dt \right)^{\frac{1}{2}} \left( 2\,\LSE(s, \tilde{s}) + \tfrac{4}{3}\,\LThree(s, \tilde{s}) \right)^{\frac{1}{2}} \\
    \quad &\leq 2 \| \psi \|_\infty \sqrt{2 d} \left( \int_0^T \sum_{\substack{\bfy \in \mathcal{X} \\ \bfy \neq \bfm}}  \beta(t) p_t(\bfy) dt \right)^{\frac{1}{2}} \left( \,\LSE(s, \tilde{s}) + \tfrac{2}{3}\,\LThree(s, \tilde{s}) \right)^{\frac{1}{2}} \\
    \quad &\leq 2 \| \psi \|_\infty  \sqrt{2d} \left( \int_0^T \beta(t) (1-p_t(\bfm)) dt \right)^{\frac{1}{2}} \left( \,\LSE(s, \tilde{s}) + \tfrac{2}{3}\,\LThree(s, \tilde{s}) \right)^{\frac{1}{2}} \\
    \quad &= 2 \| \psi \|_\infty \sqrt{2 d} \sqrt{\| \beta (1-p(\bfm))\|_1} \sqrt{ \,\LSE(s, \tilde{s}) + \tfrac{2}{3}\,\LThree(s, \tilde{s}) },
\end{align*}
where in the last step, we used $\left( \int_0^T \beta(t) (1-p_t(\bfm)) dt \right)^{\frac{1}{2}} := \sqrt{\| \beta (1-p(\bfm))\|_1} $.
Thus, 
using the bounds, \eqref{eq:ineq_eval_1} can be evaluated as
\begin{align*}
     \mathbb{E}_{\bfx \sim p_{\mathrm{data}}}[\psi(\bfx)] - \mathbb{E}_{\bfx \sim \tilde{p}_0}[\psi(\bfx)]  &=  \mathbb{E}_{\bfx \sim p_{0}}[\psi(\bfx)] - \mathbb{E}_{\bfx \sim \tilde{p}_0}[\psi(\bfx)]  = \langle \lambda_0, \psi \rangle = \cI + \cII  \\
    \left| \mathbb{E}_{\bfx \sim p_{\mathrm{data}}}[\psi(\bfx)] - \mathbb{E}_{\bfx \sim \tilde{p}_0}[\psi(\bfx)] \right| &\leq 2 \|\psi\|_\infty \left( \TV(p_T, p_{\rm base})  +   \sqrt{2 d} \sqrt{\| \beta (1-p(\bfm))\|_1} \sqrt{ \,\LSE(s, \tilde{s}) + \tfrac{2}{3}\,\LThree(s, \tilde{s}) } \right) \\
    \left| \mathbb{E}_{\bfx \sim p_{\mathrm{data}}}[\psi(\bfx)] - \mathbb{E}_{\bfx \sim \tilde{p}_0}[\psi(\bfx)] \right| &\leq 2 \|\psi\|_\infty \left( d e^{-\| \beta \|_1}  +   \sqrt{2 d} \sqrt{\| \beta (1-p(\bfm))\|_1} \sqrt{ \,\LSE(s, \tilde{s}) + \tfrac{2}{3}\,\LThree(s, \tilde{s}) } \right),
\end{align*}
where we used~\Cref{lemma:prior_mismatch_masked} in the last step. \\
\textbf{Case 2(a) (Bound under uniform rate using $\LSE$):} 
\begin{align*}
    &\int_0^T \sum_{\substack{\bfx\in\mathcal{X}}} \sum_{\substack{\bfy\in\mathcal{X} \\ \bfy\neq \bfx}}  p_t(\bfx) \left(\phi(t, \bfy)-\phi(t, \bfx) \right) Q_t(\bfy, \bfx) \left( s_t(\bfx)_\bfy - \tilde{s}_t(\bfx)_\bfy \right)  dt \\
    \quad & \leq 2 \| \psi \|_\infty \left(\int_0^T \sum_{\substack{\bfx\in\mathcal{X}}} \sum_{\substack{\bfy\in\mathcal{X} \\\bfy\neq \bfx}} p_t(\bfy)  Q_t(\bfy, \bfx)  dt\right)^{\frac{1}{2}} \left( 2\,\LSE(s, \tilde{s}) + \tfrac{4}{3}\,\LThree(s, \tilde{s}) \right)^{\frac{1}{2}}, \quad \text{by Lemma~\ref{lemma:error_bound_LSE}}, \\
    \quad &\leq 2 \| \psi \|_\infty \sqrt{d} \sqrt{\|\beta \|_1}  \sqrt{2} \sqrt{\,\LSE(s, \tilde{s}) + \tfrac{2}{3}\,\LThree(s, \tilde{s})}, \quad \text{by Lemma~\ref{lemma:sum_uniform}.} \\
    \quad & = 2 \| \psi \|_\infty \sqrt{2d} \sqrt{\|\beta \|_1}  \sqrt{\,\LSE(s, \tilde{s}) + \tfrac{2}{3}\,\LThree(s, \tilde{s})}.
\end{align*}
Thus, 
using the bounds, \eqref{eq:ineq_eval_1} can be evaluated as
\begin{align}
    \mathbb{E}_{\bfx \sim p_{\mathrm{data}}}[\psi(\bfx)] - \mathbb{E}_{\bfx \sim \tilde{p}_0}[\psi(\bfx)]  &=  \mathbb{E}_{\bfx \sim p_{0}}[\psi(\bfx)] - \mathbb{E}_{\bfx \sim \tilde{p}_0}[\psi(\bfx)] = \langle \lambda_0, \psi \rangle = \cI + \cII \nonumber \\
    \left| \mathbb{E}_{\bfx \sim p_{\mathrm{data}}}[\psi(\bfx)] - \mathbb{E}_{\bfx \sim \tilde{p}_0}[\psi(\bfx)] \right| &\leq 2 \|\psi\|_\infty \left( \TV(p_T, p_{\rm base}) +  \sqrt{2d} \sqrt{\|\beta \|_1}  \sqrt{\,\LSE(s, \tilde{s}) + \tfrac{2}{3}\,\LThree(s, \tilde{s})} \right) \nonumber \\
    \left| \mathbb{E}_{\bfx \sim p_{\mathrm{data}}}[\psi(\bfx)] - \mathbb{E}_{\bfx \sim \tilde{p}_0}[\psi(\bfx)] \right| &\leq 2 \|\psi\|_\infty \left( d e^{-\| \beta \|_1} +  \sqrt{2d} \sqrt{\|\beta \|_1}  \sqrt{\,\LSE(s, \tilde{s}) + \tfrac{2}{3}\,\LThree(s, \tilde{s})} \right), \nonumber 
\end{align}
where we used~\Cref{thm:prior_noS_uniform} in the last step. \\

\textbf{Case 2(b) (Bound under uniform rate using $\LWSM$):} \begin{align*}
    &\int_0^T \sum_{\substack{\bfx\in\mathcal{X}}} \sum_{\substack{\bfy\in\mathcal{X} \\ \bfy\neq \bfx}}  p_t(\bfx) \left(\phi(t, \bfy)-\phi(t, \bfx) \right) Q_t(\bfy, \bfx) \left( s_t(\bfx)_\bfy - \tilde{s}_t(\bfx)_\bfy \right)  dt \\
    \quad & \leq 2 \| \psi \|_\infty \left(\int_0^T \sum_{\substack{\bfx\in\mathcal{X}}} \sum_{\substack{\bfy\in\mathcal{X} \\\bfy\neq \bfx}} p_t(\bfx)  Q_t(\bfy, \bfx)  dt\right)^{\frac{1}{2}} \cdot \sqrt{2} \sqrt{\LWSM(s, \tilde{s})}, \quad \text{by Lemma~\ref{lemma:error_bound_LWSM}}, \\
    \quad &\leq 2 \| \psi \|_\infty \sqrt{d} \sqrt{\|\beta \|_1}  \sqrt{2} \sqrt{\LWSM(s, \tilde{s})}, \quad \text{by Lemma~\ref{lemma:sum_uniform}.} \\
    \quad & = 2 \| \psi \|_\infty \sqrt{2d} \sqrt{\|\beta \|_1}  \sqrt{\LWSM(s, \tilde{s})}.
\end{align*}
Thus, 
using the bounds, \eqref{eq:ineq_eval_1} can be evaluated as
\begin{align}
    \mathbb{E}_{\bfx \sim p_{\mathrm{data}}}[\psi(\bfx)] - \mathbb{E}_{\bfx \sim \tilde{p}_0}[\psi(\bfx)]  &=  \mathbb{E}_{\bfx \sim p_{0}}[\psi(\bfx)] - \mathbb{E}_{\bfx \sim \tilde{p}_0}[\psi(\bfx)] = \langle \lambda_0, \psi \rangle = \cI + \cII \nonumber \\
    \left| \mathbb{E}_{\bfx \sim p_{\mathrm{data}}}[\psi(\bfx)] - \mathbb{E}_{\bfx \sim \tilde{p}_0}[\psi(\bfx)] \right| &\leq 2 \|\psi\|_\infty \left( \TV(p_T, p_{\rm base}) +  \sqrt{2d} \sqrt{\|\beta \|_1}  \sqrt{\LWSM(s, \tilde{s})} \right) \nonumber \\
    \left| \mathbb{E}_{\bfx \sim p_{\mathrm{data}}}[\psi(\bfx)] - \mathbb{E}_{\bfx \sim \tilde{p}_0}[\psi(\bfx)] \right| &\leq 2 \|\psi\|_\infty \left( d e^{-\| \beta \|_1} +  \sqrt{2d} \sqrt{\|\beta \|_1}  \sqrt{\LWSM(s, \tilde{s})} \right), \nonumber 
\end{align}
where we used~\Cref{thm:prior_noS_uniform} in the last step. \\

\textbf{Case 3(a) (Bound under general rate using $\LSE$):} 
\begin{align*}
    &\int_0^T \sum_{\substack{\bfx\in\mathcal{X}}} \sum_{\substack{\bfy\in\mathcal{X} \\ \bfy\neq \bfx}}  p_t(\bfx) \left(\phi(t, \bfy)-\phi(t, \bfx) \right) Q_t(\bfy, \bfx) \left( s_t(\bfx)_\bfy - \tilde{s}_t(\bfx)_\bfy \right)  dt \\
    \quad & \leq 2 \| \psi\|_\infty 
    \left( \int_0^T \mathbb{E}_{\bfx \sim p_t}\!\left[ \mu^\mathrm{out}_t(\bfx) \right] dt \right)^{\frac{1}{2}} \sqrt{2}
    \left( \,\LSE(s, \tilde{s}) + \tfrac{2}{3}\,\LThree(s, \tilde{s}) \right)^{\frac{1}{2}}, \quad \text{by Corollary~\ref{cor:error_bound_LSE}}.
\end{align*}
Thus, 
using the bounds, \eqref{eq:ineq_eval_1} can be evaluated as
\begin{align}
    \mathbb{E}_{\bfx \sim p_{\mathrm{data}}}[\psi(\bfx)] - \mathbb{E}_{\bfx \sim \tilde{p}_0}[\psi(\bfx)]  &=  \mathbb{E}_{\bfx \sim p_{0}}[\psi(\bfx)] - \mathbb{E}_{\bfx \sim \tilde{p}_0}[\psi(\bfx)] = \langle \lambda_0, \psi \rangle = \cI + \cII \nonumber \\
    \left| \mathbb{E}_{\bfx \sim p_{\mathrm{data}}}[\psi(\bfx)] - \mathbb{E}_{\bfx \sim \tilde{p}_0}[\psi(\bfx)] \right| &\leq 2 \|\psi\|_\infty \left( \TV(p_T, p_{\rm base}) +  \sqrt{2} \sqrt{ \int_0^T \mathbb{E}_{\bfx \sim p_t}\!\left[ \mu^\mathrm{out}_t(\bfx) \right] dt }  \sqrt{\,\LSE(s, \tilde{s}) + \tfrac{2}{3}\,\LThree(s, \tilde{s})} \right) \nonumber.
\end{align}

\textbf{Case 3(b) (Bound under general rate using $\LWSM$):}
\begin{align*}
  & \int_0^T \sum_{\bfx\in\mathcal{X}} \sum_{\substack{\bfy\in\mathcal{X} \\ \bfy\neq \bfx}} 
    p_t(\bfx) \left(\phi(t, \bfy)-\phi(t, \bfx) \right) Q_t(\bfy, \bfx) 
    \left( s_t(\bfx)_\bfy - \tilde{s}_t(\bfx)_\bfy \right) dt  \\
  &\quad\leq\; 2\sqrt{2}\, \| \psi\|_\infty 
    \left( \int_0^T \mathbb{E}_{\bfx \sim p_t}\!\left[ \mu^\mathrm{in}_t(\bfx) \right] dt \right)^{\frac{1}{2}}
    \sqrt{\LWSM(s, \tilde{s})}, \quad \text{by Corollary~\ref{cor:error_bound_LWSM}}.
\end{align*}
Thus, 
using the bounds, \eqref{eq:ineq_eval_1} can be evaluated as
\begin{align}
    \mathbb{E}_{\bfx \sim p_{\mathrm{data}}}[\psi(\bfx)] - \mathbb{E}_{\bfx \sim \tilde{p}_0}[\psi(\bfx)]  &=  \mathbb{E}_{\bfx \sim p_{0}}[\psi(\bfx)] - \mathbb{E}_{\bfx \sim \tilde{p}_0}[\psi(\bfx)] = \langle \lambda_0, \psi \rangle = \cI + \cII \nonumber \\
    \left| \mathbb{E}_{\bfx \sim p_{\mathrm{data}}}[\psi(\bfx)] - \mathbb{E}_{\bfx \sim \tilde{p}_0}[\psi(\bfx)] \right| &\leq 2 \|\psi\|_\infty \left( \TV(p_T, p_{\rm base}) +  \sqrt{2} \sqrt{ \int_0^T \mathbb{E}_{\bfx \sim p_t}\!\left[ \mu^\mathrm{in}_t(\bfx) \right] dt }  \sqrt{\LWSM(s, \tilde{s})} \right) \nonumber.
\end{align}
\end{proof}

\subsection{Proof of \texorpdfstring{\Cref{cor:TV}}{Theorem \ref*{cor:TV}}}\label{subsec:proofTV}
\begin{restatable}[TV Bounds for Score-based Discrete Diffusion]{cor}{corTV}\label{cor:TV}
Under the same assumptions as~\Cref{lem:IPM}, the total variation (TV) distance between the true data distribution $p_{\mathrm{data}}$ and the generated data distribution $\tilde{p}_0$ satisfies \\
\textbf{Case 1 (Masked rate; $Q^{\mathrm{tok}}_t=Q_t^{\rm masked}$ \eqref{eq:Q_absorb_uni}):}
\begin{equation}\label{eq:TV_masked_bound}
    \TV(p_{\mathrm{data}}, \tilde{p}_0) \leq  \underbrace{d e^{-\| \beta \|_1}}_{\text{initialization error}}  +  \sqrt{2d} \sqrt{\| \beta (1-p(\bfm))\|_1}  \underbrace{\sqrt{ \,\LSE(s, \tilde{s}) + \tfrac{2}{3}\,\LThree(s, \tilde{s})}}_{\text{training error}} ,
\end{equation}
\textbf{Case 2 (Uniform rate; $Q^{\mathrm{tok}}_t=Q_t^{\rm uniform}$ \eqref{eq:Q_absorb_uni}):} 
\begin{equation}
    \TV(p_{\mathrm{data}}, \tilde{p}_0) \leq  \underbrace{d e^{-\| \beta \|_1}}_{\text{initialization error}} +  \sqrt{2d} \sqrt{\|\beta \|_1}  \underbrace{\sqrt{\,\LSE(s, \tilde{s}) + \tfrac{2}{3}\,\LThree(s, \tilde{s})}}_{\text{training error}} , \; \text{and}
\end{equation}
    \begin{equation}\label{eq:TV_uniform_bound_LWSM}
    \TV(p_{\mathrm{data}}, \tilde{p}_0) \leq  \underbrace{d e^{-\| \beta \|_1}}_{\text{initialization error}}  +  \sqrt{2d} \sqrt{\|\beta \|_1}  \underbrace{\sqrt{\LWSM(s, \tilde{s})}}_{\text{training error}},
\end{equation}
\textbf{Case~3 (General rate):} 
\begin{equation}
    \TV(p_{\mathrm{data}}, \tilde{p}_0)  \leq  \underbrace{\TV(p_T, p_{\rm base})}_{\text{initialization error}} +  \sqrt{2} \sqrt{ \int_0^T \mathbb{E}_{\bfx \sim p_t}\!\left[ \mu^\mathrm{out}_t(\bfx) \right] dt }  \underbrace{\sqrt{\,\LSE(s, \tilde{s}) + \tfrac{2}{3}\,\LThree(s, \tilde{s})}}_{\text{training error}}, \; \text{and}
\end{equation}
\begin{equation}
    \TV(p_{\mathrm{data}}, \tilde{p}_0)  \leq  \underbrace{\TV(p_T, p_{\rm base})}_{\text{initialization error}} +  \sqrt{2} \sqrt{ \int_0^T \mathbb{E}_{\bfx \sim p_t}\!\left[ \mu^\mathrm{in}_t(\bfx) \right] dt }  \underbrace{\sqrt{\LWSM(s, \tilde{s})}}_{\text{training error}} .
\end{equation}
\end{restatable}
\begin{proof}[Proof of~\Cref{cor:TV}]
We demonstrate the derivations for the bounds~\eqref{eq:TV_masked_bound} and~\eqref{eq:TV_uniform_bound_LWSM}. The other bounds can be derived similarly.
From~\Cref{lem:IPM}, we have
\begin{align*}
    \left| \mathbb{E}_{\bfx \sim p_{\mathrm{data}}}[\psi(\bfx)] - \mathbb{E}_{\bfx \sim \tilde{p}_0}[\psi(\bfx)] \right| &\leq 2 \|\psi\|_\infty \left( d e^{-\| \beta \|_1}  +  \sqrt{2d} \sqrt{\| \beta (1-p(\bfm))\|_1}   \sqrt{ \,\LSE(s, \tilde{s}) + \tfrac{2}{3}\,\LThree(s, \tilde{s}) } \right),
\end{align*}
and
\begin{align*}
\left| \mathbb{E}_{\bfx \sim p_{\mathrm{data}}}[\psi(\bfx)] - \mathbb{E}_{\bfx \sim \tilde{p}_0}[\psi(\bfx)] \right| &\leq 2 \|\psi\|_\infty \left( d e^{-\| \beta \|_1}  +  \sqrt{2d} \sqrt{\|\beta \|_1}  \sqrt{\LWSM(s, \tilde{s})} \right),
\end{align*}
respectively, for masked and uniform rates. The TV distance can be obtained by taking supremum on the left hand side of the equations over the set of $\psi$ satisfying $\| \psi \|_\infty \leq 1$. Specifically, the supremum is attained at $\psi=\sign(p_{\mathrm{data}} - \tilde{p}_0)$ (satisfying $\| \psi \|_\infty = 1$). Thus, we have
\begin{align*}
    2\TV(p_{\mathrm{data}}, \tilde{p}_0) &=\sup_{\psi: \| \psi \|_\infty \leq 1}  \left| \mathbb{E}_{\bfx \sim p_{\mathrm{data}}}[\psi(\bfx)] - \mathbb{E}_{\bfx \sim \tilde{p}_0}[\psi(\bfx)] \right| \\
    &\leq 2 \underbrace{\|\psi\|_\infty}_{=1} \left( d e^{-\| \beta \|_1}  +  \sqrt{2d} \sqrt{\| \beta (1-p(\bfm))\|_1}   \sqrt{ \,\LSE(s, \tilde{s}) + \tfrac{2}{3}\,\LThree(s, \tilde{s}) } \right),
\end{align*}
and 
\begin{align*}
2\TV(p_{\mathrm{data}}, \tilde{p}_0) =\sup_{\psi: \| \psi \|_\infty \leq 1}  \left| \mathbb{E}_{\bfx \sim p_{\mathrm{data}}}[\psi(\bfx)] - \mathbb{E}_{\bfx \sim \tilde{p}_0}[\psi(\bfx)] \right| &\leq 2 \underbrace{\|\psi\|_\infty}_{=1} \left( d e^{-\| \beta \|_1}  +  \sqrt{2d} \sqrt{\|\beta \|_1}  \sqrt{\LWSM(s, \tilde{s})} \right),
\end{align*}
respectively. Dividing both sides by 2, we obtain,
\begin{align*}
    \TV(p_{\mathrm{data}}, \tilde{p}_0)  &\leq   d e^{-\| \beta \|_1}  +  \sqrt{2d} \sqrt{\| \beta (1-p(\bfm))\|_1}     \sqrt{ \,\LSE(s, \tilde{s}) + \tfrac{2}{3}\,\LThree(s, \tilde{s}) } ,
\end{align*}
and
\begin{align*}
\TV(p_{\mathrm{data}}, \tilde{p}_0) &\leq d e^{-\| \beta \|_1}  +  \sqrt{2d} \sqrt{\|\beta \|_1}  \sqrt{\LWSM(s, \tilde{s})} ,
\end{align*}
respectively, 
as desired.
\end{proof}

\subsection{Proof of \texorpdfstring{\Cref{thm:spec}}{Theorem \ref*{thm:spec}}}\label{subsec:proofspec}
\thmspec*
\begin{proof}
    The bounds are derived similarly to~\Cref{cor:TV}. Specifically, we take the supremum in~\Cref{lem:IPM} over a corresponding admissible function class $\Psi$. The left-hand side of the bounds will become integral probability metrics (for total variation, it equals $2\times \TV$). And the term $\| \psi \|_\infty$ becomes the coefficient $C_\Psi$; see~\Cref{table:IPM} for a list of common IPMs and their coefficients $C_\Psi$.
\end{proof}

\subsection{Proof of \texorpdfstring{\Cref{cor:ES}}{Theorem \ref*{cor:ES}}}\label{subsec:proofES}
\corES*
\begin{proof}
    We have, for any IPM $\gamma_{\Psi}$
    \begin{equation}\label{eq:ES_triangle}
        \gamma_\Psi(p_\mathrm{data}, \tilde{p}_\delta) = \gamma_\Psi(p_0, \tilde{p}_\delta) \leq \gamma_\Psi(p_0, {p}_\delta) + \gamma_\Psi(p_\delta, \tilde{p}_\delta),\quad \text{by triangle inequality.}
    \end{equation}
    The first term in the right-hand-side of~\eqref{eq:ES_triangle} is \textbf{the early-stopping error}. For both masked and uniform rates, we have
    \begin{align*}
        \gamma_\Psi(p_0, {p}_\delta) &= \sup_{\psi \in \Psi} \left| \mathbb{E}_{\bfx \sim p_0} \left[ \psi(\bfx)\right] - \mathbb{E}_{\bfx \sim p_\delta} \left[ \psi(\bfx)\right] \right| \\
        &= \sup_{\psi \in \Psi} \left| \sum_{\bfx \in \mathcal{X}} \left( p_0(\bfx) - p_\delta(\bfx) \right) \psi(\bfx) \right| \\
        &\leq \sup_{\psi \in \Psi} \left[ \|\psi\|_\infty \sum_{\bfx \in \mathcal{X}} \left| p_0(\bfx) - p_\delta(\bfx) \right| \right] \\
        &= 2 C_\Psi \TV(p_0, p_\delta), \quad \text{where $C_\Psi$ is given in~\Cref{table:IPM},} \\
        &\leq 2 C_\Psi \left[ d\int_0^\delta \beta(t)\,dt \right],
    \end{align*}
    where in the last step, we used~\Cref{lemma:early_stop_uniform} for the uniform rate and~\Cref{lemma:early_stop_masked} for the masked rate.

    The second term on the right-hand side of~\eqref{eq:ES_triangle} compares the true and generated marginals at time $\delta$. Since the generative reverse process is run only on $[\delta, T]$, the proofs of~\Cref{lem:IPM} and~\Cref{thm:spec} apply verbatim with the time interval $[0,T]$ replaced by $[\delta,T]$, yielding the same bound in terms of the $\delta$-truncated losses $\LSE^\delta$, $\LWSM^\delta$, $\LThree^\delta$. Similarly, the time integrals of $\beta$ in their coefficients are also taken over $[\delta,T]$.
    
    Thus, we have, for the masked rate
    \begin{align*}
        \gamma_\Psi(p_\mathrm{data}, \tilde{p}_\delta) &\leq \gamma_\Psi(p_0, {p}_\delta) + \gamma_\Psi(p_\delta, \tilde{p}_\delta) \\
        &\leq 2 C_\Psi \left[ d\int_0^\delta \beta(t)\,dt + d e^{-\| \beta \|_1}   +   \sqrt{2 d} \sqrt{\int_\delta^T \beta(t) (1-p_t(\bfm)) dt} \sqrt{ \,\LSE^\delta (s, \tilde{s}) + \tfrac{2}{3}\,\LThree^\delta (s, \tilde{s}) } \right] \\
        &\leq 2 C_\Psi \left[ d\int_0^\delta \beta(t)\,dt + d e^{-\| \beta \|_1}   +   \sqrt{2 d} \sqrt{\int_0^T \beta(t) (1-p_t(\bfm)) dt} \sqrt{ \,\LSE^\delta (s, \tilde{s}) + \tfrac{2}{3}\,\LThree^\delta (s, \tilde{s}) } \right] \\
        &= 2 C_\Psi \left[ d\int_0^\delta \beta(t)\,dt + d e^{-\| \beta \|_1}   +   \sqrt{2 d} \sqrt{\| \beta (1-p(\bfm))\|_1 } \sqrt{ \,\LSE^\delta (s, \tilde{s}) + \tfrac{2}{3}\,\LThree^\delta (s, \tilde{s}) } \right].
    \end{align*}
    Here, in the second last step, we upper bound the time integral over $[\delta, T]$ by $[0, T]$ (since the integrand $\beta(t) (1-p_t(\bfm))$ is nonnegative), allowing us to represent the bound using established notation concisely in the last step. 

    The bounds for the uniform rate follow similarly as
    \begin{align*}
        \gamma_\Psi(p_\mathrm{data}, \tilde{p}_\delta) &\leq \gamma_\Psi(p_0, {p}_\delta) + \gamma_\Psi(p_\delta, \tilde{p}_\delta) \\
        &\leq 2 C_\Psi \left[ d\int_0^\delta \beta(t)\,dt + d e^{-\| \beta \|_1}   +   \sqrt{2 d} \sqrt{\int_\delta^T \beta(t) dt} \sqrt{ \,\LSE^\delta(s, \tilde{s}) + \tfrac{2}{3}\,\LThree^\delta(s, \tilde{s}) } \right] \\
        &\leq 2 C_\Psi \left[ d\int_0^\delta \beta(t)\,dt + d e^{-\| \beta \|_1}   +   \sqrt{2 d} \sqrt{\int_0^T \beta(t) dt} \sqrt{ \,\LSE^\delta(s, \tilde{s}) + \tfrac{2}{3}\,\LThree^\delta(s, \tilde{s}) } \right] \\
        &= 2 C_\Psi \left[ d\int_0^\delta \beta(t)\,dt + d e^{-\| \beta \|_1}   +   \sqrt{2 d} \sqrt{\|\beta \|_1}  \sqrt{ \,\LSE^\delta(s, \tilde{s}) + \tfrac{2}{3}\,\LThree^\delta(s, \tilde{s}) }  \right],
    \end{align*}
    and
    \begin{align*}
        \gamma_\Psi(p_\mathrm{data}, \tilde{p}_\delta) &\leq \gamma_\Psi(p_0, {p}_\delta) + \gamma_\Psi(p_\delta, \tilde{p}_\delta) \\
        &\leq 2 C_\Psi \left[ d\int_0^\delta \beta(t)\,dt + d e^{-\| \beta \|_1}   +   \sqrt{2 d} \sqrt{\int_\delta^T \beta(t) dt} \sqrt{\LWSM^\delta(s, \tilde{s})} \right] \\
        &\leq 2 C_\Psi \left[ d\int_0^\delta \beta(t)\,dt + d e^{-\| \beta \|_1}   +   \sqrt{2 d} \sqrt{\int_0^T \beta(t) dt} \sqrt{\LWSM^\delta(s, \tilde{s})} \right] \\
        &= 2 C_\Psi \left[ d\int_0^\delta \beta(t)\,dt + d e^{-\| \beta \|_1}   +   \sqrt{2 d} \sqrt{\|\beta \|_1}  \sqrt{\LWSM^\delta(s, \tilde{s})}  \right].
    \end{align*}

    The bounds for a general rate also follow similarly as
    \begin{align*}
        \gamma_\Psi(p_\mathrm{data}, \tilde{p}_\delta) &\leq \gamma_\Psi(p_0, {p}_\delta) + \gamma_\Psi(p_\delta, \tilde{p}_\delta) \\
        &\leq 2 C_\Psi \left[ \TV(p_0, p_\delta) + \TV(p_T, p_\mathrm{base})   +   \sqrt{2} \sqrt{ \int_\delta^T \mathbb{E}_{\bfx \sim p_t}\!\left[ \mu^\mathrm{out}_t(\bfx) \right] dt }  \sqrt{\,\LSE^\delta(s, \tilde{s}) + \tfrac{2}{3}\,\LThree^\delta(s, \tilde{s})} \right] \\
        &\leq 2 C_\Psi \left[ \TV(p_0, p_\delta) + \TV(p_T, p_\mathrm{base})   +   \sqrt{2} \sqrt{ \int_0^T \mathbb{E}_{\bfx \sim p_t}\!\left[ \mu^\mathrm{out}_t(\bfx) \right] dt }  \sqrt{\,\LSE^\delta(s, \tilde{s}) + \tfrac{2}{3}\,\LThree^\delta(s, \tilde{s})} \right],
    \end{align*}
    and 
    \begin{align*}
        \gamma_\Psi(p_\mathrm{data}, \tilde{p}_\delta) &\leq \gamma_\Psi(p_0, {p}_\delta) + \gamma_\Psi(p_\delta, \tilde{p}_\delta) \\
        &\leq 2 C_\Psi \left[ \TV(p_0, p_\delta) + \TV(p_T, p_\mathrm{base})   +  \sqrt{2} \sqrt{ \int_\delta^T \mathbb{E}_{\bfx \sim p_t}\!\left[ \mu^\mathrm{in}_t(\bfx) \right] dt }  \sqrt{\LWSM^\delta(s, \tilde{s})} \right] \\
        &\leq 2 C_\Psi \left[ \TV(p_0, p_\delta) + \TV(p_T, p_\mathrm{base})   +   \sqrt{2} \sqrt{ \int_0^T \mathbb{E}_{\bfx \sim p_t}\!\left[ \mu^\mathrm{in}_t(\bfx) \right] dt }  \sqrt{\LWSM^\delta(s, \tilde{s})} \right].
    \end{align*}
\end{proof}

\section{Proof of Auxiliary Lemmas For Main Theorems}
We state and prove the auxiliary lemmas.

\subsection{Lemmas for the Adjoint Equation}

\begin{lemma}[Sup-norm Non-expansiveness Under CTMC Generators]\label{lemma:phi_bound}
    Suppose $\| \psi \|_\infty < \infty$, then the solution to the Kolmogorov backward equation~\eqref{eq:KBE} $\phi_t$ satisfies
    \begin{equation}
        \| \phi_t \|_\infty \leq \| \psi \|_\infty \quad \text{for all } t\in [0, T].
    \end{equation}
\end{lemma}

\begin{proof}
We provide two proofs here for completeness. The first direction is using a Feynman-Kac representation, which is short and clean. The second direction is to mimic the maximum principle used for parabolic partial differential equations (PDEs), which is our original intuition in deriving this result.

\textbf{First Direction (Feynman-Kac representation):}
The Feynman--Kac representation gives $\phi_t(\bfx) = \E^\bfx[\psi(\bfY_t)]$, where $(\bfY_s)$ is the CTMC with generator $\tilde{Q}_{t}^{\leftarrow}$ started at $\bfY_0 = \bfx$. Hence $|\phi_t(\bfx)| \le \E^\bfx[|\psi(\bfY_t)|] \le \|\psi\|_\infty$.

\textbf{Second Direction (Maximum principle):} Recall that the Kolmogorov backward equation is given by
    \begin{equation*}
    \frac{d\phi_t}{dt}=\tilde{Q}_{t}^{\leftarrow} \phi_t, \quad \phi_0(\bfx)=\psi(\bfx), \quad \text{for } t \in [0, T].
    \end{equation*}
    For any $\epsilon >0$, construct the function
    \begin{equation*}
        \phi_t^\epsilon(\bfx) := \phi_t(\bfx) - \epsilon t; \quad \text{in vector notation: } \phi_t^\epsilon := \phi_t - \epsilon t\mathbf{1}.
    \end{equation*}
    Evaluating its time derivative gives
    \begin{align}
        \frac{d \phi_t^\epsilon }{dt} &= \frac{d}{dt}(\phi_t - \epsilon t \mathbf{1}) \nonumber \\
        &= \tilde{Q}_{t}^{\leftarrow} \phi_t - \epsilon \mathbf{1} \nonumber \\
        &= \tilde{Q}_{t}^{\leftarrow} \left( \phi^\epsilon_t + \epsilon t \mathbf{1} \right) - \epsilon \mathbf{1} \nonumber \\
        &= \tilde{Q}_{t}^{\leftarrow}  \phi^\epsilon_t  - \epsilon \mathbf{1} ,  \label{eq:d_phi_epsilon} \end{align}
        as $\tilde{Q}_{t}^{\leftarrow}$'s rows sum to 0.
        
        Suppose that $\phi^\epsilon_t$ attains its maximum at $t_0 \in (0, T]$ and $\bfx_0$, then we have the following two properties:
        \begin{enumerate}
            \item \begin{equation}\label{eq:testfunc_deri_nonneg}\frac{d  }{dt} \phi_{t_0}^\epsilon(\bfx_0) \geq 0,\end{equation} \\
            as $\frac{d  }{dt} \phi_{t_0}^\epsilon(\bfx_0)  = 0$ if $t_0 < T$ and $\frac{d  }{dt} \phi_{t_0}^\epsilon(\bfx_0)  \geq 0$ if $t_0 = T$,
            \item \begin{equation}\label{eq:testfunc_Q_neg}[\tilde{Q}_{t_0}^{\leftarrow} \phi_{t_0}^\epsilon](\bfx_0) \leq 0,\end{equation} \\
            because
            \begin{align*}
                [\tilde{Q}_{t_0}^{\leftarrow} \phi_{t_0}^\epsilon](\bfx_0) &= \sum_\bfy \tilde{Q}_{t_0}^{\leftarrow}(\bfx_0, \bfy) \phi_{t_0}^\epsilon(\bfy)\\
                &=\sum_{\bfy \neq \bfx_0} \left[\tilde{Q}_{t_0}^{\leftarrow}(\bfx_0, \bfy) \phi_{t_0}^\epsilon(\bfy) \right] + \tilde{Q}_{t_0}^{\leftarrow}(\bfx_0, \bfx_0) \phi_{t_0}^\epsilon(\bfx_0)\\
                &=\sum_{\bfy \neq \bfx_0} \underbrace{\tilde{Q}_{t_0}^{\leftarrow}(\bfx_0, \bfy)}_{\geq 0} \underbrace{\left(\phi_{t_0}^\epsilon(\bfy)-\phi_{t_0}^\epsilon(\bfx_0) \right)}_{\leq 0, \text{ as $\phi_{t_0}^\epsilon(\bfx_0)$ is max}} \leq 0.
            \end{align*}
        \end{enumerate}
    The two properties together give us
    \begin{align*}
        0 &\leq \frac{d  }{dt} \phi_{t_0}^\epsilon(\bfx_0) , \quad \text{by~\eqref{eq:testfunc_deri_nonneg}}\\
        &= \underbrace{\left[ \tilde{Q}_{t}^{\leftarrow}  \phi^\epsilon_t\right](\bfx_0)}_{\leq 0, \text{ by~\eqref{eq:testfunc_Q_neg}}}  - \epsilon ,\quad \text{by~\eqref{eq:d_phi_epsilon},}\\
        &\leq -\epsilon \\
        &<0,
    \end{align*}
    leading to a contradiction. This implies $\phi^\epsilon_t$ cannot attain its maximum at $t_0 \in (0, T]$. Thus, $\phi^\epsilon_t$ must attain its maximum at $t=0$, and we have, for all $t \in [0, T], \bfx$,
    \begin{equation*}
        \phi_t(\bfx)-\epsilon t= \phi_t^\epsilon(\bfx) \leq \max_\bfy \phi_0^\epsilon(\bfy) = \max_\bfy \phi_0(\bfy)= \max_\bfy \psi(\bfy).
    \end{equation*}
    Since $\epsilon>0$ is arbitrary, taking $\epsilon \to 0^+$ gives
    \begin{equation*}
        \phi_t(\bfx) \leq \max_\bfy \psi(\bfy), \quad \text{for all $t \in [0, T], \bfx$}.
    \end{equation*}
    This implies 
    $$
    \| \phi_t \|_\infty \leq \| \psi \|_\infty, \quad \text{for all $t \in [0, T]$},
    $$
    as desired.
\end{proof}

\subsection{Lemmas for the Training Loss Term}
\begin{lemma}\label{lemma:sum_uniform}
    Suppose the forward process~\eqref{eq:KFE_ntoken} is governed by the factorizing transition~\eqref{eq:Q_factor} with uniform rate~\eqref{eq:Q_absorb_uni}, then
    $$
    \left( \int_0^T \sum\limits_{\substack{\bfx\in\mathcal{X}}} \sum\limits_{\substack{\bfy\in\mathcal{X} \\ \bfy\neq \bfx}} 
    p_t(\bfy)  Q_t(\bfy, \bfx) dt\right)^{\frac{1}{2}} = \left( \int_0^T \sum\limits_{\substack{\bfx\in\mathcal{X}}} \sum\limits_{\substack{\bfy\in\mathcal{X} \\ \bfy\neq \bfx}} 
    p_t(\bfx)  Q_t(\bfy, \bfx) dt\right)^{\frac{1}{2}} \leq \sqrt{d} \sqrt{\|\beta\|_1},
    $$
    where $\|\beta\|_1 := \int_0^T |\beta(t)| dt=\int_0^T \beta(t) dt$.
\end{lemma}

\begin{proof}
First, we derive the equality.
\begin{align*}
    \int_0^T \sum\limits_{\substack{\bfx\in\mathcal{X}}} \sum\limits_{\substack{\bfy\in\mathcal{X} \\ \bfy\neq \bfx}} 
    p_t(\bfy)  Q_t(\bfy, \bfx) dt &= \int_0^T \sum\limits_{\substack{\bfx\in\mathcal{X}}} \sum\limits_{\substack{\bfy\in\mathcal{X} \\ \bfy\neq \bfx}} 
    p_t(\bfy)  \frac{\beta(t)}{S} dt, \quad \text{by~\eqref{eq:Q_absorb_uni} and~\eqref{eq:Q_factor},} \\
    &= \int_0^T \sum\limits_{\substack{\bfy \in\mathcal{X}}} p_t(\bfy)\sum\limits_{\substack{\bfx \in\mathcal{X} \\ \bfx \neq \bfy}} \frac{\beta(t)}{S} dt, \quad \text{by Fubini's Theorem,} \\
    &= \int_0^T \sum\limits_{\substack{\bfx \in\mathcal{X}}} p_t(\bfx)\sum\limits_{\substack{\bfy \in\mathcal{X} \\ \bfy \neq \bfx}} \frac{\beta(t)}{S} dt, \quad \text{where we swapped the dummy variable $\bfx$ and $\bfy$,} \\
    &= \int_0^T \sum\limits_{\substack{\bfx \in\mathcal{X}}} p_t(\bfx)\sum\limits_{\substack{\bfy \in\mathcal{X} \\ \bfy \neq \bfx}} Q_t(\bfy, \bfx) dt, \quad \text{by~\eqref{eq:Q_absorb_uni} and~\eqref{eq:Q_factor} again,} \\
    &= \int_0^T \sum\limits_{\substack{\bfx \in\mathcal{X}}} \sum\limits_{\substack{\bfy \in\mathcal{X} \\ \bfy \neq \bfx}} p_t(\bfx) Q_t(\bfy, \bfx) dt.
\end{align*}
Next, we derive the upper bound.
\begin{align*}
    &\int_0^T \sum\limits_{\substack{\bfx\in\mathcal{X}}} \sum\limits_{\substack{\bfy\in\mathcal{X} \\ \bfy\neq \bfx}} 
    p_t(\bfx)  Q_t(\bfy, \bfx) dt \\
    &=\int_0^T \sum\limits_{\substack{\bfx\in\mathcal{X}}} \sum\limits_{\substack{\bfy\in\mathcal{X} \\ \distH{\bfy}{\bfx} = 1}} p_t(\bfx)  \frac{\beta(t)}{S} dt, \quad \text{by~\eqref{eq:Q_absorb_uni} and~\eqref{eq:Q_factor},}\\
    &=\int_0^T \sum\limits_{\substack{\bfx\in\mathcal{X}}} p_t(\bfx)  \frac{\beta(t)d(S-1)}{S}dt, \quad \text{as each $\bfx$ has $d(S-1)$ neighbors $\bfy$ satisfying }\distH{\bfy}{\bfx} = 1,\\
    &=\int_0^T \frac{\beta(t)d(S-1)}{S} dt \\
    &\leq d \int_0^T \beta(t) dt\\
    &= d \|\beta\|_1,
\end{align*}
taking square root on both sides yields the desired result.
\end{proof}

\begin{lemma}\label{lemma:relative_matching}
    Under the factorizing transition~\eqref{eq:Q_factor} with the masked transition rate~\eqref{eq:Q_absorb_uni}, for $0<t\leq T$, and $g_t: \mathcal{X} \times \mathcal{X} \to \mathbb{R}_{\geq 0}$,
    \begin{equation}
        \sum_{\substack{\bfy \in \mathcal{X} \\ \bfy \neq \bfm}} \sum_{\substack{\substack{\bfx \in \Nmasked(\bfy)}}} p_t(\bfy) g_t(\bfx, \bfy) = (1-p_t(\bfm)) \mathbb{E}_{\bfy \sim p_t(\bfy|\bfy \neq \bfm)} \left[ \sum_{\substack{\substack{\bfx \in \Nmasked(\bfy)}}} g_t(\bfx, \bfy) \right],
    \end{equation}
    where $\bfm=[\mask,...,\mask]$ is the sequence of all mask tokens, and $$\Nmasked(\bfy)=\left\{ \bfx \in \mathcal{X} \, \middle| \, \exists i \in [d], \bfx^{\backslash i}=\bfy^{\backslash i} , x^i=\mask, y^i \neq \mask \right\}$$ is the successor set of $\bfy$.
\end{lemma}
\begin{proof}
    We have
    \begin{align*}
        p_t(\bfy | \bfy\neq \bfm) &= \frac{p_t(\bfy)}{1-p_t(\bfm)}, \quad &\text{for } \bfy \neq \bfm,\\
        (1-p_t(\bfm)) p_t(\bfy | \bfy\neq \bfm) &= {p_t(\bfy)}, \quad &\text{for } \bfy\neq \bfm,\\
        ({1-p_t(\bfm)}) p_t(\bfy | \bfy\neq \bfm) g_t(\bfx, \bfy) &= p_t(\bfy) g_t(\bfx, \bfy), \quad &\text{for } \bfy \neq \bfm, \, \bfx \in \Nmasked(\bfy),\\
        ({1-p_t(\bfm)}) \sum_{\substack{\bfy \in \mathcal{X} \\ \bfy \neq \bfm}} \sum_{\substack{\substack{\bfx \in \Nmasked(\bfy)}}} p_t(\bfy | \bfy\neq \bfm) g_t(\bfx, \bfy) &= \sum_{\substack{\bfy \in \mathcal{X} \\ \bfy \neq \bfm}} \sum_{\substack{\substack{\bfx \in \Nmasked(\bfy)}}} p_t(\bfy) g_t(\bfx, \bfy) \\
        (1-p_t(\bfm)) \mathbb{E}_{\bfy \sim p_t(\bfy|\bfy \neq \bfm)} \left[ \sum_{\substack{\substack{\bfx \in \Nmasked(\bfy)}}} g_t(\bfx, \bfy)\right] &=\sum_{\substack{\bfy \in \mathcal{X} \\ \bfy \neq \bfm}} \sum_{\substack{\substack{\bfx \in \Nmasked(\bfy)}}} p_t(\bfy) g_t(\bfx, \bfy),
    \end{align*}
    swapping left-hand-side and right-hand-side, we get the desired result.
\end{proof}

\begin{lemma}
\label{lem:elementary}
For $|u| \le 1/2$,
\begin{equation}
\frac{u^2}{2} + u + \log(1 - u) \;\le\; \frac{2|u|^3}{3}.
\label{eq:elementary}
\end{equation}
\end{lemma}

\begin{proof}
Define $f(u) := \frac{2|u|^3}{3} - \frac{u^2}{2} - u - \log(1-u)$ on
$[-1/2, 1/2]$. We show $f(u) \ge 0$ by case analysis on the sign of $u$.

\emph{Case $u \in [0, 1/2]$.} Here $|u|^3 = u^3$, so
$f(u) = \tfrac{2u^3}{3} - \tfrac{u^2}{2} - u - \log(1-u)$ and
\begin{align*}
f(0) &= 0, \\
f'(u) &= 2u^2 - u - 1 + \tfrac{1}{1-u}, & f'(0) &= 0, \\
f''(u) &= 4u - 1 + \tfrac{1}{(1-u)^2}, & f''(0) &= 0, \\
f'''(u) &= 4 + \tfrac{2}{(1-u)^3} \;>\; 0 \quad \text{on } [0, 1/2].
\end{align*}
Since $f''' > 0$ on $[0, 1/2]$, $f''$ is strictly increasing; combined
with $f''(0) = 0$, this gives $f''(u) > 0$ for $u \in (0, 1/2]$. Hence
$f'$ is strictly increasing on $[0, 1/2]$, and $f'(0) = 0$ gives
$f'(u) > 0$ for $u \in (0, 1/2]$. Hence $f$ is strictly increasing on
$[0, 1/2]$, and $f(0) = 0$ gives $f(u) > 0$ for $u \in (0, 1/2]$.

\emph{Case $u \in [-1/2, 0]$.} Here $|u|^3 = -u^3$, so
$f(u) = -\tfrac{2u^3}{3} - \tfrac{u^2}{2} - u - \log(1-u)$ and
\begin{align*}
f(0) &= 0, \\
f'(u) &= -2u^2 - u - 1 + \tfrac{1}{1-u}, & f'(0) &= 0, \\
f''(u) &= -4u - 1 + \tfrac{1}{(1-u)^2}, & f''(0) &= 0, \\
f'''(u) &= -4 + \tfrac{2}{(1-u)^3}.
\end{align*}
On $[-1/2, 0]$, $1 - u \in [1, 3/2]$, so
$\tfrac{2}{(1-u)^3} \in [\tfrac{16}{27}, 2]$, giving
$f'''(u) \le 2 - 4 = -2 < 0$. Hence $f''$ is strictly decreasing on
$[-1/2, 0]$; combined with $f''(0) = 0$, this gives $f''(u) > 0$ for
$u \in [-1/2, 0)$. Hence $f'$ is strictly increasing on $[-1/2, 0]$,
and $f'(0) = 0$ gives $f'(u) < 0$ for $u \in [-1/2, 0)$. Hence $f$ is
strictly decreasing on $[-1/2, 0]$, and $f(0) = 0$ gives $f(u) > 0$
for $u \in [-1/2, 0)$.

Combining both cases, $f(u) \ge 0$ on $[-1/2, 1/2]$, with equality only
at $u = 0$. Rearranging gives~\eqref{eq:elementary}.
\end{proof}

\begin{lemma}[Bregman with positive cubic remainder]
\label{lem:bregman-cubic}
Define the Bregman divergence with generating function
$\varphi(a) = a\log a - a$ as
\begin{equation}
D_\varphi(s, \tilde s) = \varphi(s) - \varphi(\tilde s) - \varphi'(\tilde s)(s - \tilde s) = s\log\left(\frac{s}{\tilde s}\right) + \tilde s - s.
\label{eq:bregman}
\end{equation}
For $s > 0$ and $\tilde s \in [s/2, 3s/2]$,
\begin{equation}
\frac{s}{2} {\left(1 - \frac{\tilde s}{s}\right)^2}\;\le\; D_\varphi(s, \tilde s) + \frac{2\,|s - \tilde s|^3}{3 s^2}.
\label{eq:bregman-cubic}
\end{equation}
\end{lemma}

\begin{proof}
Let $u := (s - \tilde s)/s$. The hypothesis $\tilde s \in [s/2, 3s/2]$
gives $|u| \le 1/2$. Since $\tilde s/s = 1 - u$,
\begin{equation*}
D_\varphi(s, \tilde s) = -s\log(1-u) - s u.
\end{equation*}
Dividing~\eqref{eq:bregman-cubic} by $s > 0$ and rearranging gives the
equivalent statement
\begin{equation*}
\frac{u^2}{2} + u + \log(1 - u) \;\le\; \frac{2|u|^3}{3},
\end{equation*}
which is Lemma~\ref{lem:elementary}.
\end{proof}

\begin{lemma}[Bounding Error by Score Entropy]\label{lemma:error_bound_LSE}
    Under a general transition rate $Q_t$ satisfying~\Cref{ass:rate}, we have 
    \begin{align*}
        &\int_0^T \sum_{\substack{\bfx\in\mathcal{X}}} \sum_{\substack{\bfy\in\mathcal{X} \\ \bfy\neq \bfx}} p_t(\bfx) \left(\phi(t, \bfy)-\phi(t, \bfx) \right) Q_t(\bfy, \bfx) \left( s_t(\bfx)_\bfy - \tilde{s}_t(\bfx)_\bfy \right)  dt \\
        &\quad \leq 2 \| \psi\|_\infty  \left( \int_0^T \sum_{\substack{\bfy \in\mathcal{X}}} \sum_{\substack{\bfx \in\mathcal{X} \\ \bfx \neq \bfy}}  Q_t(\bfy, \bfx) p_t(\bfy)    dt \right)^{\frac{1}{2}} \left( 2\,\LSE(s, \tilde{s}) + \tfrac{4}{3}\,\LThree(s, \tilde{s}) \right)^{\frac{1}{2}}.
    \end{align*}
\end{lemma}

\begin{proof}
We have
    \begin{align*}
    &\int_0^T \sum_{\substack{\bfx\in\mathcal{X}}} \sum_{\substack{\bfy\in\mathcal{X} \\ \bfy\neq \bfx}} p_t(\bfx) \left(\phi(t, \bfy)-\phi(t, \bfx) \right) Q_t(\bfy, \bfx) \left( s_t(\bfx)_\bfy - \tilde{s}_t(\bfx)_\bfy \right)  dt \\
    \quad &= \int_0^T \sum_{\substack{\bfx\in\mathcal{X}}} \sum_{\substack{\bfy\in\mathcal{X} \\ \bfy\neq \bfx}}  p_t(\bfx) \left(\phi(t, \bfy)-\phi(t, \bfx) \right) Q_t(\bfy, \bfx) s_t(\bfx)_\bfy \left( 1 - \frac{\tilde{s}_t(\bfx)_\bfy}{s_t(\bfx)_\bfy} \right)  dt \\
    \quad & \leq \left( \int_0^T \sum_{\substack{\bfx\in\mathcal{X}}} \sum_{\substack{\bfy\in\mathcal{X} \\ \bfy\neq \bfx}} p_t(\bfx) Q_t(\bfy, \bfx) s_t(\bfx)_\bfy  \left(\phi(t, \bfy)-\phi(t, \bfx) \right)^2  dt \right)^{\frac{1}{2}} \\
    &\quad \left( \underbrace{\int_0^T \sum_{\substack{\bfx\in\mathcal{X}}} \sum_{\substack{\bfy\in\mathcal{X} \\ \bfy\neq \bfx}} p_t(\bfx) Q_t(\bfy, \bfx) s_t(\bfx)_\bfy \left( 1 - \frac{\tilde{s}_t(\bfx)_\bfy}{s_t(\bfx)_\bfy} \right) ^2  dt}_{=\LWRSM, \text{ defined in~\eqref{eq:LWRSM}}} \right)^{\frac{1}{2}}, \\
    \intertext{by Cauchy-Schwarz inequality,}
    \quad & \leq \left( \int_0^T \sum_{\substack{\bfx\in\mathcal{X}}} \sum_{\substack{\bfy\in\mathcal{X} \\ \bfy\neq \bfx}} \cancel{p_t(\bfx)} Q_t(\bfy, \bfx) \frac{p_t(\bfy)}{\cancel{p_t(\bfx)}}  \left(\phi(t, \bfy)-\phi(t, \bfx) \right)^2  dt \right)^{\frac{1}{2}} \left( 2\,\LSE(s, \tilde{s}) + \tfrac{4}{3}\,\LThree(s, \tilde{s}) \right)^{\frac{1}{2}}, \quad \text{by~\Cref{lemma:LWRSM_SE}.}\\
    \intertext{\textbf{And critically, we applied the score-cancellation technique that changes the probability mass term from $p_t(\bfx)$ to $p_t(\bfy)$. This removes the $S$-dependence of the resulting bound for masked rates. See~\Cref{ssec:bound_scale,ssec:bound_detail} for more details.}}
    \quad & = \left( \int_0^T \sum_{\substack{\bfy \in\mathcal{X}}} \sum_{\substack{\bfx \in\mathcal{X} \\ \bfx \neq \bfy}}  Q_t(\bfy, \bfx) p_t(\bfy)  \left(\phi(t, \bfy)-\phi(t, \bfx) \right)^2  dt \right)^{\frac{1}{2}} \left( 2\,\LSE(s, \tilde{s}) + \tfrac{4}{3}\,\LThree(s, \tilde{s}) \right)^{\frac{1}{2}}, \\
    \intertext{where in the first term, we used Fubini's Theorem,}
    \quad & \leq 2 \max_{0\leq t\leq T}\left(\| \phi_t\|_\infty \right)  \left( \int_0^T \sum_{\substack{\bfy \in\mathcal{X}}} \sum_{\substack{\bfx \in\mathcal{X} \\ \bfx \neq \bfy}}  Q_t(\bfy, \bfx) p_t(\bfy)    dt \right)^{\frac{1}{2}} \left( 2\,\LSE(s, \tilde{s}) + \tfrac{4}{3}\,\LThree(s, \tilde{s}) \right)^{\frac{1}{2}} \\
    \quad & = 2 \| \phi\|_\infty  \left( \int_0^T \sum_{\substack{\bfy \in\mathcal{X}}} \sum_{\substack{\bfx \in\mathcal{X} \\ \bfx \neq \bfy}}  Q_t(\bfy, \bfx) p_t(\bfy)    dt \right)^{\frac{1}{2}} \left( 2\,\LSE(s, \tilde{s}) + \tfrac{4}{3}\,\LThree(s, \tilde{s}) \right)^{\frac{1}{2}} \\
    \quad &\leq 2 \| \psi\|_\infty  \left( \int_0^T \sum_{\substack{\bfy \in\mathcal{X}}} \sum_{\substack{\bfx \in\mathcal{X} \\ \bfx \neq \bfy}}  Q_t(\bfy, \bfx) p_t(\bfy)    dt \right)^{\frac{1}{2}} \left( 2\,\LSE(s, \tilde{s}) + \tfrac{4}{3}\,\LThree(s, \tilde{s}) \right)^{\frac{1}{2}}, \quad \text{by~\Cref{lemma:phi_bound}.}
\end{align*}
\end{proof}

\begin{corollary}[Bounding Error by Score Entropy]\label{cor:error_bound_LSE}
Under a general transition rate $Q_t$ satisfying~\Cref{ass:rate}, let
\[
  \mu^\mathrm{out}_t(\bfx) \;:=\; -Q_t(\bfx, \bfx) 
   \;=\; \sum_{\bfy \neq \bfx} Q_t(\bfx, \bfy)
\]
denote the \textbf{exit rate} from state $\bfx$ at time $t$. Then
\begin{align*}
  & \int_0^T \sum_{\bfx\in\mathcal{X}} \sum_{\substack{\bfy\in\mathcal{X} \\ \bfy\neq \bfx}} 
    p_t(\bfx) \left(\phi(t, \bfy)-\phi(t, \bfx) \right) Q_t(\bfy, \bfx) 
    \left( s_t(\bfx)_\bfy - \tilde{s}_t(\bfx)_\bfy \right) dt \\
  &\quad\leq\; 2 \| \psi\|_\infty 
    \left( \int_0^T \mathbb{E}_{\bfx \sim p_t}\!\left[ \mu^\mathrm{out}_t(\bfx) \right] dt \right)^{\frac{1}{2}}
    \left( 2\,\LSE(s, \tilde{s}) + \tfrac{4}{3}\,\LThree(s, \tilde{s}) \right)^{\frac{1}{2}}.
\end{align*}
\end{corollary}

\begin{proof}
By Lemma~\ref{lemma:error_bound_LSE}, we have
\begin{align}
\begin{split}\label{eq:bound_error_LSE_quote}
        &\int_0^T \sum_{\substack{\bfx\in\mathcal{X}}} \sum_{\substack{\bfy\in\mathcal{X} \\ \bfy\neq \bfx}} p_t(\bfx) \left(\phi(t, \bfy)-\phi(t, \bfx) \right) Q_t(\bfy, \bfx) \left( s_t(\bfx)_\bfy - \tilde{s}_t(\bfx)_\bfy \right)  dt \\
        &\quad \leq 2 \| \psi\|_\infty  \left( \int_0^T \sum_{\substack{\bfy \in\mathcal{X}}} \sum_{\substack{\bfx \in\mathcal{X} \\ \bfx \neq \bfy}}  Q_t(\bfy, \bfx) p_t(\bfy)    dt \right)^{\frac{1}{2}} \left( 2\,\LSE(s, \tilde{s}) + \tfrac{4}{3}\,\LThree(s, \tilde{s}) \right)^{\frac{1}{2}}.
    \end{split}
    \end{align}

\smallskip
We first bound the integrand on the right-hand-side. Fubini's theorem gives
\[
  \sum_{\bfx} \sum_{\bfy \neq \bfx} Q_t(\bfy, \bfx)\,p_t(\bfy)
  \;=\; \sum_{\bfy} p_t(\bfy) \sum_{\bfx \neq \bfy} Q_t(\bfy, \bfx)
  \;=\; \sum_{\bfy} p_t(\bfy)\,\mu^\mathrm{out}_t(\bfy)
  \;=\; \mathbb{E}_{\bfx \sim p_t}\!\left[ \mu^\mathrm{out}_t(\bfx) \right].
\]
Thus, we obtain
\begin{equation}\label{eq:transition_out}
  \left( \int_0^T \sum_{\substack{\bfy \in\mathcal{X}}} \sum_{\substack{\bfx \in\mathcal{X} \\ \bfx \neq \bfy}}  Q_t(\bfy, \bfx) p_t(\bfy)    dt \right)^{\frac{1}{2}} \;=\;
      \left( \int_0^T \mathbb{E}_{\bfx \sim p_t}\!\left[ \mu^\mathrm{out}_t(\bfx) \right] dt \right)^{\frac{1}{2}}.
\end{equation}

Combining~\eqref{eq:bound_error_LSE_quote} and~\eqref{eq:transition_out} yields the claim.
\end{proof}
\begin{remark}[Interpretation of the prefactor]\label{rmk:exit_rate}
The quantity 
\[
  \mathcal{N}_T \;:=\; \int_0^T \mathbb{E}_{\bfx \sim p_t}\!\left[ \mu^\mathrm{out}_t(\bfx) \right] dt
\]
appearing in the bound is the \emph{expected total number of jumps} 
the forward CTMC makes on $[0, T]$ when initialized at 
$\bfX_0 \sim p_{\rm data}$. Indeed, the jump-counting process 
$J_t := \#\{\text{jumps on }[0, t]\}$ has compensator 
$A_t = \int_0^t \mu^\mathrm{out}_s(\bfX_s)\,ds$, so by the Doob--Meyer 
decomposition $\mathbb{E}[J_T] = \mathbb{E}[A_T] = \mathcal{N}_T$. 
The bound thus takes the form 
$$|\,\textit{approximation error}\,| \le 2 \|\psi\|_\infty \sqrt{\mathcal{N}_T}\,\sqrt{2\LSE + (4/3)\LThree}\, ,$$ 
i.e.\ the score-matching loss is amplified by the square root of 
the expected dynamical activity of the forward process. 
\end{remark}

\begin{lemma}[Bounding Error by Score Matching Loss]\label{lemma:error_bound_LWSM}
    Under a general transition rate $Q_t$ satisfying~\Cref{ass:rate}, we have
    \begin{align*}
        & \int_0^T \sum_{\substack{\bfx\in\mathcal{X}}} \sum_{\substack{\bfy\in\mathcal{X} \\ \bfy\neq \bfx}}  p_t(\bfx) \left(\phi(t, \bfy)-\phi(t, \bfx) \right) Q_t(\bfy, \bfx) \left( s_t(\bfx)_\bfy - \tilde{s}_t(\bfx)_\bfy \right)  dt \\ &\quad \leq 2 \| \psi \|_\infty \left(\int_0^T \sum_{\substack{\bfx\in\mathcal{X}}} \sum_{\substack{\bfy\in\mathcal{X} \\\bfy\neq \bfx}} p_t(\bfx)  Q_t(\bfy, \bfx)  dt\right)^{\frac{1}{2}} \cdot \sqrt{2} \sqrt{\LWSM(s, \tilde{s})}
    \end{align*}
\end{lemma}

\begin{proof}
    \begin{align*}
    &\int_0^T \sum_{\substack{\bfx\in\mathcal{X}}} \sum_{\substack{\bfy\in\mathcal{X} \\ \bfy\neq \bfx}}  p_t(\bfx) \left(\phi(t, \bfy)-\phi(t, \bfx) \right) Q_t(\bfy, \bfx) \left( s_t(\bfx)_\bfy - \tilde{s}_t(\bfx)_\bfy \right)  dt \\
    \quad &\leq \left(\int_0^T \sum_{\substack{\bfx\in\mathcal{X}}} \sum_{\substack{\bfy\in\mathcal{X} \\ \bfy\neq \bfx}} p_t(\bfx)  Q_t(\bfy, \bfx) \left(\phi(t, \bfy)-\phi(t, \bfx) \right)^2 dt\right)^{\frac{1}{2}} \\
    &\quad \left( 2 \cdot \underbrace{\int_0^T \half \sum_{\substack{\bfx\in\mathcal{X}}} \sum_{\substack{\bfy\in\mathcal{X} \\ \bfy\neq \bfx}} p_t(\bfx)  Q_t(\bfy, \bfx) \left( s_t(\bfx)_\bfy - \tilde{s}_t(\bfx)_\bfy \right)^2 dt}_{=\LWSM(s, \tilde{s}) \; \text{defined in~\eqref{eq:LWSM}}} \right)^{\frac{1}{2}}, \\
    \intertext{by Cauchy-Schwarz inequality,}
    \quad & \leq 2 \max_{0\leq t\leq T}\left(\| \phi_t\|_\infty \right) \left(\int_0^T \sum_{\substack{\bfx\in\mathcal{X}}} \sum_{\substack{\bfy\in\mathcal{X} \\ \bfy\neq \bfx}} p_t(\bfx)  Q_t(\bfy, \bfx)  dt\right)^{\frac{1}{2}} \cdot \sqrt{2} \sqrt{\LWSM(s, \tilde{s})} \\
    \quad & \leq 2 \| \psi \|_\infty \left(\int_0^T \sum_{\substack{\bfx\in\mathcal{X}}} \sum_{\substack{\bfy\in\mathcal{X} \\\bfy\neq \bfx}} p_t(\bfx)  Q_t(\bfy, \bfx)  dt\right)^{\frac{1}{2}} \cdot \sqrt{2} \sqrt{\LWSM(s, \tilde{s})}, \quad \text{by Lemma~\ref{lemma:phi_bound}}.
\end{align*}
\end{proof}

\begin{corollary}[Bounding Error by Score Matching Loss]\label{cor:error_bound_LWSM}
Under a general transition rate $Q_t$ satisfying~\Cref{ass:rate}, let
\[
  \mu^\mathrm{in}_t(\bfx) \;:=\; \sum_{\bfy \neq \bfx} Q_t(\bfy, \bfx)
\]
denote the \textbf{entrance rate} at state $\bfx$ at time $t$ (the total instantaneous 
rate at which other states feed into $\bfx$). Then
\begin{align*}
  & \int_0^T \sum_{\bfx\in\mathcal{X}} \sum_{\substack{\bfy\in\mathcal{X} \\ \bfy\neq \bfx}} 
    p_t(\bfx) \left(\phi(t, \bfy)-\phi(t, \bfx) \right) Q_t(\bfy, \bfx) 
    \left( s_t(\bfx)_\bfy - \tilde{s}_t(\bfx)_\bfy \right) dt  \\
  &\quad\leq\; 2\sqrt{2}\, \| \psi\|_\infty 
    \left( \int_0^T \mathbb{E}_{\bfx \sim p_t}\!\left[ \mu^\mathrm{in}_t(\bfx) \right] dt \right)^{\frac{1}{2}}
    \sqrt{\LWSM(s, \tilde{s})}.
\end{align*}
\end{corollary}

\begin{proof}
By Lemma~\ref{lemma:error_bound_LWSM}, we have
\begin{align}\label{eq:error_bound_LWSM_quote}
\begin{split}
        & \int_0^T \sum_{\substack{\bfx\in\mathcal{X}}} \sum_{\substack{\bfy\in\mathcal{X} \\ \bfy\neq \bfx}}  p_t(\bfx) \left(\phi(t, \bfy)-\phi(t, \bfx) \right) Q_t(\bfy, \bfx) \left( s_t(\bfx)_\bfy - \tilde{s}_t(\bfx)_\bfy \right)  dt \\ &\quad \leq 2 \| \psi \|_\infty \left(\int_0^T \sum_{\substack{\bfx\in\mathcal{X}}} \sum_{\substack{\bfy\in\mathcal{X} \\\bfy\neq \bfx}} p_t(\bfx)  Q_t(\bfy, \bfx)  dt\right)^{\frac{1}{2}} \cdot \sqrt{2} \sqrt{\LWSM(s, \tilde{s})}
    \end{split}
    \end{align}

\smallskip
On the other hand, we have
\[
  \sum_{\bfx} \sum_{\bfy \neq \bfx} p_t(\bfx)\,Q_t(\bfy, \bfx)
  \;=\; \sum_{\bfx} p_t(\bfx) \sum_{\bfy \neq \bfx} Q_t(\bfy, \bfx)
  \;=\; \sum_{\bfx} p_t(\bfx)\,\mu^\mathrm{in}_t(\bfx)
  \;=\; \mathbb{E}_{\bfx \sim p_t}\!\left[ \mu^\mathrm{in}_t(\bfx) \right].
\]
Thus, we obtain
\begin{equation}\label{eq:transition_in}
  \int_0^T \sum_{\substack{\bfx\in\mathcal{X}}} \sum_{\substack{\bfy\in\mathcal{X} \\\bfy\neq \bfx}} p_t(\bfx)  Q_t(\bfy, \bfx)  dt  \;=\; 
      \int_0^T \mathbb{E}_{\bfx \sim p_t}\!\left[ \mu^\mathrm{in}_t(\bfx) \right] dt .
\end{equation}
\smallskip

Combining~\eqref{eq:error_bound_LWSM_quote} and~\eqref{eq:transition_in} yields the claim.
\end{proof}

\begin{remark}[Exit rate vs.\ entrance rate under uniform and masked]
The exit rate $\mu^\mathrm{out}_t$ and entrance rate $\mu^\mathrm{in}_t$ behave very differently 
under the two transition rates in this work.

\noindent Under the uniform rate~\eqref{eq:Q_absorb_uni}, $Q_t^{\rm tok}$ is symmetric and
\[
  \mu_t^{\rm out, tok}(x) \;=\; \mu_t^{\rm in, tok}(x) \;=\; \beta(t)\bigl(1 - \tfrac{1}{S}\bigr)
  \quad \text{for every } x \in [S],
\]
so the two lemmas give the same scaling.

\smallskip
\noindent Under the masked rate~\eqref{eq:Q_absorb_uni}, $Q_t^{\rm tok}$ is 
asymmetric and the two rates differ:
\[
  \mathbb{E}_{p_t^{\rm tok}}\!\bigl[\mu_t^{\rm out, tok}(x)\bigr] 
  \;=\; \beta(t)\bigl(1 - p_t^{\rm tok}(\mask)\bigr)
\]
is $S$-independent, whereas
\[
  \mathbb{E}_{p_t^{\rm tok}}\!\bigl[\mu_t^{\rm in, tok}(x)\bigr] 
  \;=\; (S-1)\,\beta(t)\,p_t^{\rm tok}(\mask)
\]
carries a factor of $S-1$ — the vocabulary size.

\smallskip
\noindent This is why our masked-case analysis (\Cref{lem:IPM}, Case~1) 
uses~\Cref{lemma:error_bound_LSE} to expose the exit rate, while the 
uniform case (Case~2) can use either lemma.
\end{remark}

\begin{lemma}[Bounding $\LWRSM$ by $\LSE$]\label{lemma:LWRSM_SE}
    Let
    \begin{equation}\label{eq:LWRSM}
        \LWRSM(s, \tilde{s}) = \int_0^T \sum_{\substack{\bfx\in\mathcal{X}}} \sum_{\substack{\bfy\in\mathcal{X} \\ \bfy\neq \bfx}} p_t(\bfx) Q_t(\bfy, \bfx) s_t(\bfx)_\bfy \left( 1 - \frac{\tilde{s}_t(\bfx)_\bfy}{s_t(\bfx)_\bfy} \right) ^2  dt
    \end{equation}
    be the weighted relative score-matching loss (WRSM), and recall the score entropy (SE) loss is given by
    \begin{align*}
        \LSE(s, \tilde{s})&=\int_0^T \sum_{\substack{\bfx\in\mathcal{X}}} \sum_{\substack{\bfy\in\mathcal{X} \\ \bfy\neq \bfx}} p_t(\bfx)  Q_t(\bfy, \bfx) \left[ s_t(\bfx)_\bfy \log\left(\frac{s_t(\bfx)_\bfy}{\tilde{s}_t(\bfx)_\bfy}\right) + \tilde{s}_t(\bfx)_\bfy - s_t(\bfx)_\bfy \right] dt\\
        &=\int_0^T \sum_{\substack{\bfx\in\mathcal{X}}} \sum_{\substack{\bfy\in\mathcal{X} \\ \bfy\neq \bfx}} p_t(\bfx)  Q_t(\bfy, \bfx) D_\varphi(s_t(\bfx)_\bfy), \tilde{s}_t(\bfx)_\bfy) dt.
    \end{align*}
    Suppose that the score estimator
$\tilde s_t(\bfx)_\bfy \in [s_t(\bfx)_\bfy/2,\,3 s_t(\bfx)_\bfy/2]$
for all $(t, \bfx, \bfy)$ in the support of $p_t(\bfx) Q_t(\bfy, \bfx)$, then under the masked transition rate (\eqref{eq:Q_factor} and \eqref{eq:Q_absorb_uni}),
\begin{equation}
\LWRSM(s, \tilde{s}) \;\le\; 2\,\LSE(s, \tilde{s}) + \tfrac{4}{3}\,\LThree(s, \tilde{s}),
\label{eq:LWRSM-LSE}
\end{equation}
where
\myfcases
\end{lemma}

\begin{proof}
    \begin{align*}
        \LWRSM(s, \tilde{s}) &= \int_0^T \sum_{\substack{\bfx\in\mathcal{X}}} \sum_{\substack{\bfy\in\mathcal{X} \\ \bfy\neq \bfx}} p_t(\bfx) Q_t(\bfy, \bfx) s_t(\bfx)_\bfy \left( 1 - \frac{\tilde{s}_t(\bfx)_\bfy}{s_t(\bfx)_\bfy} \right) ^2  dt\\
        & \leq \int_0^T \sum_{\substack{\bfx\in\mathcal{X}}} \sum_{\substack{\bfy\in\mathcal{X} \\ \bfy\neq \bfx}} p_t(\bfx) Q_t(\bfy, \bfx) \cdot 2   \cdot \left( D_\varphi(s_t(\bfx)_\bfy, \tilde{s}_t(\bfx)_\bfy) + \frac{2\,|s_t(\bfx)_\bfy - \tilde{s}_t(\bfx)_\bfy|^3}{3 (s_t(\bfx)_\bfy)^2} \right) dt, \quad \text{by \Cref{lem:bregman-cubic},}\\
        & = 2 \LSE(s, \tilde{s}) +  \int_0^T \sum_{\substack{\bfx\in\mathcal{X}}} \sum_{\substack{\bfy\in\mathcal{X} \\ \bfy\neq \bfx}} p_t(\bfx) Q_t(\bfy, \bfx) \left( \frac{4\,|s_t(\bfx)_\bfy - \tilde{s}_t(\bfx)_\bfy|^3}{3 (s_t(\bfx)_\bfy)^2} \right) dt \\
         &= 2 \LSE(s, \tilde{s}) +  \frac{4}{3} {\LThree(s, \tilde{s})},
    \end{align*}
    where we used the definition of $\LThree$, and its formulations under different rates are given in~\Cref{lemma:L3_formulate}.
\end{proof}

\begin{lemma}[Formulation of $\LThree$ under masked and uniform rates]\label{lemma:L3_formulate}
    The cubic correction term
    \begin{equation*}
        \LThree(s, \tilde{s}) := \int_0^T \mathbb{E}_{\bfx \sim p_t} \left[ \sum_{\substack{\bfy\in\mathcal{X} \\ \bfy\neq \bfx}}  Q_t(\bfy, \bfx) \left( \frac{\,|s_t(\bfx)_\bfy - \tilde{s}_t(\bfx)_\bfy|^3}{(s_t(\bfx)_\bfy)^2} \right) \right]  dt
    \end{equation*}
    can be reformulated as: \\
    \textbf{Case 1 (Masked rate; $Q^{\mathrm{tok}}_t=Q_t^{\rm masked}$ \eqref{eq:Q_absorb_uni}):}
    \begin{equation}
        \LThree(s, \tilde{s}) = \int_0^T \beta(t) (1-p_t(\bfm)) \mathbb{E}_{\bfy \sim p_t(\bfy|\bfy \neq \bfm)} \left[ \sum_{\substack{\substack{\bfx \in \Nmasked(\bfy)}}} \frac{\,|s_t(\bfx)_\bfy - \tilde{s}_t(\bfx)_\bfy|^3}{ (s_t(\bfx)_\bfy)^3} \right] dt,
    \end{equation}
    \textbf{Case 2 (Uniform rate; $Q^{\mathrm{tok}}_t=Q_t^{\rm uniform}$ \eqref{eq:Q_absorb_uni}):} 
    \begin{equation}
        \LThree(s, \tilde{s}) = \int_0^T \mathbb{E}_{\bfx \sim p_t} \left[ \frac{\beta(t)}{S} \sum_{\substack{\bfy \in\mathcal{X} \\ d_H(\bfy,\bfx)=1}}   \left( \frac{\,|s_t(\bfx)_\bfy - \tilde{s}_t(\bfx)_\bfy|^3}{(s_t(\bfx)_\bfy)^2} \right) \right]  dt,
    \end{equation}
    \textbf{Case 3 (General rate):}
    \begin{equation}
        \LThree(s, \tilde{s}) = \int_0^T \mathbb{E}_{\bfy \sim p_t} \left[ \sum_{\substack{\bfx\in\mathcal{X} \\ \bfx\neq \bfy}}  Q_t(\bfy, \bfx)  \left( \frac{\,|s_t(\bfx)_\bfy - \tilde{s}_t(\bfx)_\bfy|^3}{ (s_t(\bfx)_\bfy)^3} \right) \right] dt.
    \end{equation}
\end{lemma}

\begin{proof}
        \textbf{Case 3 (General rate):} \\
    We first derive the equivalent formulation under a general rate.
    \begin{align*}
        & \int_0^T \mathbb{E}_{\bfx \sim p_t} \left[ \sum_{\substack{\bfy\in\mathcal{X} \\ \bfy\neq \bfx}}  Q_t(\bfy, \bfx) \left( \frac{\,|s_t(\bfx)_\bfy - \tilde{s}_t(\bfx)_\bfy|^3}{(s_t(\bfx)_\bfy)^2} \right) \right]  dt \\
        & = \int_0^T \sum_{\substack{\bfx\in\mathcal{X}}} \sum_{\substack{\bfy\in\mathcal{X} \\ \bfy\neq \bfx}} p_t(\bfx) Q_t(\bfy, \bfx) \left( \frac{\,|s_t(\bfx)_\bfy - \tilde{s}_t(\bfx)_\bfy|^3}{(s_t(\bfx)_\bfy)^2} \right) dt \\
         & = \int_0^T \sum_{\substack{\bfx\in\mathcal{X}}} \sum_{\substack{\bfy\in\mathcal{X} \\ \bfy\neq \bfx}} p_t(\bfx) Q_t(\bfy, \bfx) s_t(\bfx)_\bfy \left( \frac{\,|s_t(\bfx)_\bfy - \tilde{s}_t(\bfx)_\bfy|^3}{ (s_t(\bfx)_\bfy)^3} \right) dt \\
         & = \int_0^T \sum_{\substack{\bfx\in\mathcal{X}}} \sum_{\substack{\bfy\in\mathcal{X} \\ \bfy\neq \bfx}} \cancel{p_t(\bfx)} Q_t(\bfy, \bfx) \frac{p_t(\bfy)}{\cancel{p_t(\bfx)}} \left( \frac{\,|s_t(\bfx)_\bfy - \tilde{s}_t(\bfx)_\bfy|^3}{ (s_t(\bfx)_\bfy)^3} \right) dt \\
         & = \int_0^T \sum_{\substack{\bfy\in\mathcal{X}}} \sum_{\substack{\bfx\in\mathcal{X} \\ \bfx\neq \bfy}}  Q_t(\bfy, \bfx) {p_t(\bfy)} \left( \frac{\,|s_t(\bfx)_\bfy - \tilde{s}_t(\bfx)_\bfy|^3}{ (s_t(\bfx)_\bfy)^3} \right) dt, \quad \text{by Fubini's Theorem,}\\
         &= \int_0^T \mathbb{E}_{\bfy \sim p_t} \left[ \sum_{\substack{\bfx\in\mathcal{X} \\ \bfx\neq \bfy}}  Q_t(\bfy, \bfx)  \left( \frac{\,|s_t(\bfx)_\bfy - \tilde{s}_t(\bfx)_\bfy|^3}{ (s_t(\bfx)_\bfy)^3} \right) \right] dt.
    \end{align*}

    \textbf{Case 1 (Masked rate; $Q^{\mathrm{tok}}_t=Q_t^{\rm masked}$ \eqref{eq:Q_absorb_uni}):} \\
    We start with the equivalent formulation under a generate rate, which we have just derived,
    \begin{align*}
    & \int_0^T \mathbb{E}_{\bfy \sim p_t} \left[ \sum_{\substack{\bfx\in\mathcal{X} \\ \bfx\neq \bfy}}  Q_t(\bfy, \bfx)  \left( \frac{\,|s_t(\bfx)_\bfy - \tilde{s}_t(\bfx)_\bfy|^3}{ (s_t(\bfx)_\bfy)^3} \right) \right] dt \\
        &= \int_0^T \sum_{\substack{\bfy\in\mathcal{X}}} \sum_{\substack{\bfx\in\mathcal{X} \\ \bfx\neq \bfy}}  Q_t(\bfy, \bfx) {p_t(\bfy)} \left( \frac{\,|s_t(\bfx)_\bfy - \tilde{s}_t(\bfx)_\bfy|^3}{ (s_t(\bfx)_\bfy)^3} \right) dt \\
         & = \int_0^T \sum_{\substack{\bfy \in \mathcal{X} \\ \bfy \neq \bfm}} \sum_{\substack{\substack{\bfx \in \Nmasked(\bfy)}}}   Q^{\mathrm{tok}}_t(y^i, m) {p_t(\bfy)} \left( \frac{\,|s_t(\bfx)_\bfy - \tilde{s}_t(\bfx)_\bfy|^3}{ (s_t(\bfx)_\bfy)^3} \right) dt, \\
         \intertext{where we used the definition of $Q_t$~\eqref{eq:Q_factor} under masked transition~\eqref{eq:Q_absorb_uni}, $\bfm=[\mask,...,\mask]$ is the sequence of all mask tokens, 
    $$
    \Nmasked(\bfy)=\left\{ \bfx \in \mathcal{X} \, \middle| \, \exists i \in [d], \bfx^{\backslash i}=\bfy^{\backslash i} , x^i=\mask, y^i \neq \mask \right\}
    $$
    is the successor set of $\bfy$, and $i$ denotes the unique index where $\bfx$ and $\bfy$ differ,}
    & = \int_0^T \sum_{\substack{\bfy \in \mathcal{X} \\ \bfy \neq \bfm}} \sum_{\substack{\substack{\bfx \in \Nmasked(\bfy)}}}   \beta(t) {p_t(\bfy)} \left( \frac{\,|s_t(\bfx)_\bfy - \tilde{s}_t(\bfx)_\bfy|^3}{ (s_t(\bfx)_\bfy)^3} \right) dt \\
    &= \int_0^T \beta(t) (1-p_t(\bfm)) \mathbb{E}_{\bfy \sim p_t(\bfy|\bfy \neq \bfm)} \left[ \sum_{\substack{\substack{\bfx \in \Nmasked(\bfy)}}} \frac{\,|s_t(\bfx)_\bfy - \tilde{s}_t(\bfx)_\bfy|^3}{ (s_t(\bfx)_\bfy)^3} \right] dt.
    \end{align*}
    \textbf{Case 2 (Uniform rate; $Q^{\mathrm{tok}}_t=Q_t^{\rm uniform}$ \eqref{eq:Q_absorb_uni}):} \\
    In addition, we have
    \begin{align*}
        & \int_0^T \mathbb{E}_{\bfx \sim p_t} \left[ \sum_{\substack{\bfy\in\mathcal{X} \\ \bfy\neq \bfx}}  Q_t(\bfy, \bfx) \left( \frac{\,|s_t(\bfx)_\bfy - \tilde{s}_t(\bfx)_\bfy|^3}{(s_t(\bfx)_\bfy)^2} \right) \right]  dt \\
        & = \int_0^T \mathbb{E}_{\bfx \sim p_t} \left[ \frac{\beta(t)}{S} \sum_{\substack{\bfy \in\mathcal{X} \\ d_H(\bfy,\bfx)=1}}   \left( \frac{\,|s_t(\bfx)_\bfy - \tilde{s}_t(\bfx)_\bfy|^3}{(s_t(\bfx)_\bfy)^2} \right) \right]  dt, \quad \text{by the definition of the uniform rate~\eqref{eq:Q_factor} and~\eqref{eq:Q_absorb_uni},}  \\
    \end{align*}
\end{proof}

\begin{lemma}[Initialization Error Bound under Masked Diffusion]\label{lemma:prior_mismatch_masked}
    Under masked diffusion, the initialization error term is bounded by
    \begin{equation}
        \TV(p_T, \pbase) \leq d e^{-\int_0^T \beta(\tau) d\tau} = d e^{-\| \beta \|_1}.
    \end{equation}
\end{lemma}

\begin{proof}
The all-mask probability $p_t(\bfm) = \alpha_t^d$ under the Kronecker-sum factorization~\eqref{eq:Q_factor} with single-token
mask probability $\alpha_t$. Here, $\alpha_t$ satisfies the ODE
\begin{equation}
    \partial_t \alpha_t = \beta(t)(1 - \alpha_t) \quad \text{subject to} \quad \alpha_0=0,
\end{equation}
derived from the masked rate~\eqref{eq:Q_absorb_uni} and with solution
\begin{equation}
    \alpha_t = 1-e^{-\int_0^t \beta(\tau) d\tau}, \quad \text{for } t\in[0, T].
\end{equation}
    Thus, the initialization error satisfies
$$
\TV(p_T, p_{\mathrm{base}}) = \TV(p_T, \delta_{\bfm}) = 1 - p_T(\bfm) = 1 - \alpha_T^d = 1 - \left( 1-e^{-\int_0^T \beta(\tau) d\tau} \right)^d \leq d \cdot e^{-\int_0^T \beta(\tau) d\tau}= d e^{-\| \beta \|_1}.
$$
\end{proof}

\section{Bounding the Initialization Error Term under Uniform Rate}\label{sec:coupling_uniform}
In this section, we derive, to the best of our knowledge, the first $S$-independent initialization error bound under the uniform transition rate in the context of discrete diffusion generative models. Our key novelty is a coupling technique. For better readability, we begin by restating the setup.

\paragraph{Setup.}
We consider a forward CTMC on $[S]^d$. The dynamics are
\emph{non-interacting}: each coordinate (token) $i \in [d]$ evolves
independently according to the per-site generator
\begin{equation}\label{eq:Q_uni_appendix}
  (Q_t^{\rm uniform})^\top
  \;=\; \beta(t)\Big(\tfrac{1}{S}\mathbf{1}_S \mathbf{1}_S^\top - I_S\Big)
  \;=\; \beta(t)\,(\Punif - I),
\end{equation}
where $\Punif := \tfrac{1}{S}\mathbf{1}_S\mathbf{1}_S^\top$ is the rank-one
projection onto the uniform distribution $\pi = \mathrm{Unif}([S])$.
The joint generator is
\begin{equation}\label{eq:joint_generator}
\mathcal{L}_t = \sum_{i=1}^d Q_t^{({\rm uniform}),\,i},
\end{equation}
with
$Q_t^{({\rm uniform}),\,i}$ acting on coordinate $i$. The initial law
is $p_0 = p_{\rm data}$, the reference law is $\pi^{\otimes d}$. Set
\begin{equation}\label{eq:rhoT}
  \rho_t \;:=\; \exp\!\Big(-\!\int_0^t \beta(\tau)\,d\tau\Big); \quad \rho_T \;:=\; \exp\!\Big(-\!\int_0^T \beta(\tau)\,d\tau\Big).
\end{equation}

Bounding $\TV(p_T,\, \pbase)=\TV(p_T,\, \pi^{\otimes d})$ reduces to exhibiting
\emph{any} coupling with controlled disagreement probability, by the
following classical characterization of total variation.
\begin{lemma}[{Coupling characterization of $\TV$; \citep[Proposition~4.7]{levin2017markov}}]
\label{lem:coupling}
For any probability measures $\mu, \nu$ on a common Polish space,
\[
  \TV(\mu, \nu) \;=\; \min_{\Gamma \in \Pi(\mu, \nu)}\,
                       \Pr_{(X, Y) \sim \Gamma}\!\big(X \neq Y\big),
\]
where $\Pi(\mu, \nu)$ is the set of couplings of $\mu$ and $\nu$ (joint
laws with the prescribed marginals), and the minimum is attained (by
the maximal coupling). Consequently, for any \emph{specific} coupling
$\Gamma \in \Pi(\mu, \nu)$,
\begin{equation}\label{eq:coupling}
  \TV(\mu, \nu) \;\le\; \Pr_\Gamma\!\big(X \neq Y\big).
\end{equation}
\end{lemma}

\paragraph{Per-site uniformization.} Next, we consider a per-site uniformization representation of $Q^\mathrm{uniform}_t$. The entries of the per-site generator~\eqref{eq:Q_uni_appendix} are
\begin{equation}\label{eq:Q_entries}
  Q_t^{\rm uniform}(v, w) \;=\; \beta(t)\Big[\tfrac{1}{S} - \mathbf{1}[v=w]\Big]
  \;=\;
  \begin{cases}
    \;\;\;\,\beta(t)/S & w \neq v, \\[3pt]
    -(S-1)\,\beta(t)/S & w = v,
  \end{cases}
\end{equation}
so the total exit rate from any state $v \in [S]$ is the same constant
\begin{equation}\label{eq:exit_rate}
  |Q_t^{\rm uniform}(v,v)|
  \;=\; \sum_{w \neq v} Q_t^{\rm uniform}(v, w)
  \;=\; \frac{(S-1)\,\beta(t)}{S}
  \quad\text{for every } v \in [S].
\end{equation}
Uniformization requires only
$\Lambda(t) \ge \max_v |Q_t^{\rm uniform}(v,v)| = (S-1)\beta(t)/S$.
We choose
\begin{equation}\label{eq:choice_lambda}
  \Lambda(t) \;=\; \beta(t),
\end{equation}
which is valid since $\beta(t) \ge (S-1)\beta(t)/S$ for all $S \ge 1$,
with slack $\beta(t)/S$ corresponding to ``self-jumps'' $v \to v$ that
are invisible at the generator level. With this choice the
uniformization transition matrix is
\begin{equation}\label{eq:P_beta}
  P_\beta(t)
  \;:=\; I + \frac{Q_t^{\rm uniform}}{\beta(t)}
  \;=\; I + (\Punif - I)
  \;=\; \Punif,
  \qquad
  P_\beta(v, w) = \frac{1}{S} \text{ for all } v, w \in [S].
\end{equation}
The transition rule is the rank-one projection: \emph{at each Poisson
ring the chain is reset to a fresh uniform sample $\xi \sim \pi$,
independently of its current state}, with $1/S$ probability of an
invisible self-jump.

\begin{lemma}[Uniformized construction]
\label{lem:uniformization}
Let $N$ be a Poisson process with rate $\beta(t)$, and let
$(\xi_k)_{k \ge 1}$ be i.i.d.\ $\mathrm{Unif}([S])$, independent of $N$.
Define $(X_t)_{t \ge 0}$ by $X_0 = v_0$ and, at each event time $t_k$
of $N$, set $X_{t_k} = \xi_k$. Then $(X_t)$ is the per-site CTMC with
generator~\eqref{eq:Q_uni_appendix} started at $v_0$.
\end{lemma}

\begin{proof}
Conditional on $N$ and $(\xi_k)$,
\[
  X_T \;=\; \begin{cases}
    v_0 & \text{if } N([0,T]) = 0, \\
    \xi_{N([0,T])} & \text{otherwise.}
  \end{cases}
\]
The first case has probability $\rho_T = e^{-\int_0^T \beta(s) ds}$; the
second has probability $1 - \rho_T$ and yields an independent uniform
sample. Hence the $T$-step kernel is
\begin{equation}\label{eq:per_site_kernel}
  \Pr(X_T = w \mid X_0 = v)
   \;=\; \rho_T\,\mathbf{1}[v = w] + (1-\rho_T)/S.
\end{equation}
This matches the kernel obtained directly from~\eqref{eq:Q_uni_appendix}:
using $(\Punif - I)^2 = -(\Punif - I)$,
\[
  \exp\!\Big(\!\int_0^T Q_s^{\rm uniform}\,ds\Big)
  \;=\; \exp\big({-}\log\rho_T \cdot (\Punif - I)\big)
  \;=\; \rho_T\,I + (1-\rho_T)\,\Punif.
\]
The Markov property of $(X_t)$ follows from the strong Markov property
of~$N$ and the independence of $(\xi_k)$. \qedhere
\end{proof}

\begin{remark}[Joint Construction]\label{remark:joint_const}
By non-interaction, the joint CTMC on $[S]^d$ is realized by $d$
\emph{independent} per-site uniformized chains: independent Poisson
processes $N_1, \dots, N_d$ each at rate $\beta(t)$; for each $i$, an
i.i.d.\ sequence $(\xi^{(i)}_k)_{k \ge 1}$ in $[S]$, independent of
everything else; at each event time of $N_i$, position $i$ is reset
to the next sample $\xi^{(i)}_k$. The marginal law on position $i$ is
then~\eqref{eq:per_site_kernel}; the joint law has generator
$\mathcal{L}_t$ defined in~\eqref{eq:joint_generator}.
\end{remark}

\paragraph{Initialization Error Bound.} We are now ready to state and prove, to the best of our knowledge, the first $S$-independent initialization error bound under uniform transition in the context of discrete diffusion generative models.

\begin{theorem}\label{thm:prior_noS_uniform}
For all $T \ge 0$,
\[
  \TV(p_T,\, \pi^{\otimes d}) \;\le\; d\,\exp\!\Big(-\!\int_0^T \beta(t)\,dt\Big) \;=\;  d e^{-\| \beta \|_1}.
\]
In particular, $\TV(p_T, \pi^{\otimes d}) \le \varepsilon$ whenever
$\|\beta\|_1 \ge \log(d/\varepsilon)$, with no dependence on $S$.
\end{theorem}

\begin{proof}
We construct an explicit synchronous coupling of two copies of the
forward chain, using the joint uniformized representation of~\Cref{remark:joint_const}, and apply Lemma~\ref{lem:coupling}.

\medskip
\noindent\emph{Coupling.} Draw $X_0 \sim p_{\rm data}$ and
$Y_0 \sim \pi^{\otimes d}$ independently. Let $N_1, \dots, N_d$ be
i.i.d.\ Poisson processes with rate $\beta(t)$, and let
$(\xi^{(i)}_k)_{i \in [d],\, k \ge 1}$ be i.i.d.\ $\mathrm{Unif}([S])$,
independent of $N$. Evolve both $X$ and $Y$ by the joint uniformized
construction of~\Cref{remark:joint_const} driven by the \emph{same} clocks $(N_i)$ and
\emph{same} samples $(\xi^{(i)}_k)$.

\medskip
\noindent\emph{Marginal verification.} By Lemma~\ref{lem:uniformization}
applied per coordinate, $X_t \sim p_t$ (the forward chain started at
$p_{\rm data}$) and $Y_t \sim \pi^{\otimes d}$ for all $t \ge 0$ (the
forward chain started at its stationary law, which it preserves).

\medskip
\noindent\emph{Merging per coordinate.} Let
$\tau_i := \inf\{t \ge 0 : N_i([0,t]) \ge 1\}$ be the first ring time
at position $i$. At $t = \tau_i$, the shared sample $\xi^{(i)}_1$
resets both $X^{(i)}$ and $Y^{(i)}$ to the same value (this is
property~\eqref{eq:P_beta} of the uniformized representation: the
reset distribution is state-independent, so the same sample is valid
for both chains). The shared subsequent samples then keep them equal:
$X^{(i)}_t = Y^{(i)}_t$ for all $t \ge \tau_i$. Hence
\[
  \{X_T \neq Y_T\} \;\subseteq\; \bigcup_{i=1}^d \{\tau_i > T\}.
\]

\medskip
\noindent\emph{Union bound.} Each $\tau_i$ is the first event of a
Poisson process with rate $\beta(t)$, so $\Pr(\tau_i > T) = \rho_T$.
By the union bound,
\[
  \Pr(X_T \neq Y_T) \;\le\; \sum_{i=1}^d \Pr(\tau_i > T) \;=\; d\,\rho_T.
\]

\medskip
\noindent\emph{Conclusion.} The construction is a specific coupling
$\Gamma \in \Pi(p_T,\, \pi^{\otimes d})$. By~\eqref{eq:coupling} of Lemma~\ref{lem:coupling},
\[
  \TV(p_T,\, \pi^{\otimes d})
  \;\le\; \Pr_\Gamma(X_T \neq Y_T)
  \;\le\; d\,\rho_T \;=\; d\,\exp\!\Big(-\!\int_0^T \beta(t)\,dt\Big) \;=\;  d e^{-\| \beta \|_1},
\]
by~\eqref{eq:rhoT}.
The closure $d\,\rho_T \le \varepsilon$ holds iff
$\|\beta\|_1 \ge \log(d/\varepsilon)$. \qedhere
\end{proof}

\begin{remark*}[$S$-independence]
The vocabulary size $S$ enters the bound only through the choice of
reference $\pi$, never through any rate. The synchronous coupling
construction is agnostic to the reset distribution; the analysis
depends only on $\pi$ being preserved by the dynamics, so that
$Y_t \sim \pi^{\otimes d}$ for all $t$.
\end{remark*}

\begin{remark*}
{\textcolor{red} This argument is the continuous-time, general state-space analog of the
classical coupling analysis of the lazy random walk on the hypercube
$\{0,1\}^n$~\citep[Chapter~5]{levin2017markov}; the underlying workhorse is
the Markovian coupling theorem~\citep[Theorem~5.2]{levin2017markov}. An
equivalent analysis through the strong stationary time framework
\citep[Chapter~6]{levin2017markov} yields the slightly tighter bound
$\TV(p_T, \pi^{\otimes d}) \le 1 - (1-\rho_T)^d$.}
\end{remark*}

\section{Bounding the Early-Stopping Error Term}
When the reverse process is stopped at a small time $\delta > 0$ rather than at $0$, the law $p_\delta$ differs from the true data law $p_0$, contributing an early-stopping error $\TV(p_0, p_\delta)$. The same coupling argument used to bound the initialization error in~\Cref{sec:coupling_uniform} applies here, run on the time interval $[0,\delta]$ instead of $[0,T]$. We state the lemmas and proofs below.

\begin{lemma}[Early-Stopping Error under Uniform Transition]\label{lemma:early_stop_uniform}
For all $\delta \geq 0$,
\[
  \TV(p_0,\, p_\delta) \;\le\; 1 - \rho_\delta^{\,d} \;=\; 1 - e^{-d\int_0^\delta \beta(t)\,dt} \;\le\; d\int_0^\delta \beta(t)\,dt.
\]
\end{lemma}
\begin{proof}
We couple $X_0 \sim p_0 = p_{\rm data}$ with $X_\delta \sim p_\delta$ by running a single copy of the forward chain via the joint uniformized construction of~\Cref{remark:joint_const}, with i.i.d.\ per-coordinate Poisson clocks $N_1,\dots,N_d$ of rate $\beta(t)$.
\medskip

\noindent\emph{No-ring event.} Let $\tau_i := \inf\{t \ge 0 : N_i([0,t]) \ge 1\}$ be the first ring time at coordinate $i$. On $\{\tau_i > \delta\}$ coordinate $i$ is never reset, so $X^{(i)}_\delta = X^{(i)}_0$. Hence
\[
  \{X_0 = X_\delta\} \;\supseteq\; \bigcap_{i=1}^d \{\tau_i > \delta\}.
\]
\medskip

\noindent\emph{Probability bound.} The Poisson clocks are independent and $\Pr(\tau_i > \delta) = e^{-\int_0^\delta \beta(t)\,dt}=:\rho_\delta $, so
\[
  \Pr(X_0 = X_\delta) \;\ge\; \prod_{i=1}^d \Pr(\tau_i > \delta) \;=\; \rho_\delta^{\,d}.
\]
\medskip

\noindent\emph{Conclusion.} The construction is a coupling $\Gamma \in \Pi(p_0,\, p_\delta)$. By~\eqref{eq:coupling} of Lemma~\ref{lem:coupling},
\[
  \TV(p_0,\, p_\delta) \;\le\; \Pr_\Gamma(X_0 \neq X_\delta) \;\le\; 1 - \rho_\delta^{\,d} \;=\; 1 - e^{-d\int_0^\delta \beta(t)\,dt},
\]
and $1 - e^{-x} \le x$ yields $\TV(p_0,\, p_\delta) \le d\int_0^\delta \beta(t)\,dt$. \qedhere
\end{proof}

\begin{lemma}[Early-Stopping Bound under Masked Diffusion]\label{lemma:early_stop_masked}
    Under masked diffusion with early stopping at $\delta>0$, the early-stopping mismatch term is bounded by
    \begin{equation}
        \TV(p_0, p_\delta) \;\leq\; 1 - e^{-d\int_0^\delta \beta(t)\,dt} .
    \end{equation}
\end{lemma}
\begin{proof}
Under the Kronecker-sum factorization~\eqref{eq:Q_factor}, the forward process masks each of the $d$ tokens \emph{independently}, each with single-token mask probability
\begin{equation}
    \alpha_\delta = 1 - e^{-\int_0^\delta \beta(t)\,dt},
\end{equation}
so the probability that a given token remains unmasked is $1-\alpha_\delta = e^{-\int_0^\delta \beta(t)\,dt}$.

By independence across tokens, the probability that none of the $d$ tokens are selected for masking by time $\delta$ is $(1-\alpha_\delta)^d$. Conditional on this event, $\bfx_\delta = \bfx_0$ exactly, so there exists a coupling of $(\bfx_0, \bfx_\delta)\sim (p_0, p_\delta)$ with
$$
\mathbb{P}(x_0 = x_\delta) \;=\; \underbrace{(1-\alpha_\delta)^d}_{\text{prob. of all $d$ tokens remain unmasked}}.
$$
The total variation is bounded by the coupling disagreement probability (\Cref{lem:coupling}), hence
$$
\TV(p_0, p_\delta) \;\leq\; \mathbb{P}(x_0 \neq x_\delta) \;\leq\; 1 - (1-\alpha_\delta)^d \;=\; 1 - e^{-d\int_0^\delta \beta(t)\,dt},
$$
and $1 - e^{-x} \le x$ yields $\TV(p_0,\, p_\delta) \le d\int_0^\delta \beta(t)\,dt$. 
\end{proof}

\section{\texorpdfstring{$S$}{S}-Independence of the Bounds}\label{sec:bound_scale}
We provide the detailed discussion on the scaling of the bounds in our main theory, \Cref{lem:IPM}. A crucial feature of our bound is that, the coefficient of the training loss term is independent of the vocabulary size $S$. This independence is practically significant: in applications such as language modeling, $S$ corresponds to the vocabulary size, which is prohibitively large and renders any $S$-dependent bound vacuous.
\subsection{Scaling of the Bounds}
\label{ssec:bound_scale}

\paragraph{Case~1 (Masked).} In Case~1, the bound is completely independent of $S$.

The prefactor
$\sqrt{d} \sqrt{\| \beta (1-p(\bfm))\|_1} =\sqrt{d} \left( \int_0^T \beta(t) (1-p_t(\bfm)) dt \right)^{\frac{1}{2}} $ is obtained using the exit-routing technique; see Remarks~\ref{rmk:SMC_ER} and~\ref{rmk:uniform_vs_masked}.
The prefactor depends on the
schedule $\beta(t)$ and the all-mask probability $p_t(\bfm) = \alpha_t^d$ under the Kronecker-sum factorization~\eqref{eq:Q_factor} with single-token
mask probability $\alpha_t$. Here, $\alpha_t$ satisfies the ODE
\begin{equation}
    \partial_t \alpha_t = \beta(t)(1 - \alpha_t) \quad \text{subject to} \quad \alpha_0=1,
\end{equation}
derived from the masked rate~\eqref{eq:Q_absorb_uni}) and with solution
\begin{equation}
    \alpha_t = 1-e^{-\int_0^t \beta(s) ds}, \quad \text{for } t\in[0, T].
\end{equation}
Thus, neither of $\beta(t)$ and $p_t(\bfm)$
involves $S$. So the prefactor is indepedent of $S$.

The loss $\LSE$ defined in~\eqref{eq:LSE} under masked transition is structurally $S$-independent: at each $(t, \bfy)$
the integrand is summed by $\sum_{\bfy \neq \bfx}=\sum_{\bfx \in \Nmasked(\bfy)}$, where $\Nmasked(\bfy)$ is the successor set of $\bfy$. It is defined in~\eqref{eq:successor} and contains at most $d$ terms (one successor per position, obtained by masking that position, and there are $d$ positions).

\paragraph{Case 1 (Masked): Score-marginal Cancellation.} The $S$-independence depends critically on our \textbf{score-marginal cancellation}
$p_t(\bfx)s_t(\bfx)_\bfy = p_t(\bfy)$ applied in the proof of~\Cref{lem:IPM}.
The unsimplified integrand of~\eqref{eq:ineq_eval_1} sums over the
$m(\bfx)(S - 1)$ predecessors $\bfy$ of each $\bfx$, with the factor $S-1$
counting vocabulary values. The cancellation re-indexes the sum to range
over $\bfx \in \Nmasked(\bfy)$ ($\le d$ terms, no
vocabulary factor), absorbing the $S-1$; Lemma~\ref{lemma:relative_matching}
then converts the resulting sum into the conditional-expectation form
appearing in $\LSE$.

The boundary 
is also $S$-independent, by~\Cref{lemma:prior_mismatch_masked}.

\paragraph{Case~2 (Uniform).} In Case~2, the bound is completely independent of $S$.

The prefactor $\sqrt{d}\sqrt{\|\beta\|_1}$
and the loss $\LWSM$ defined in~\eqref{eq:LWSM} are both $S$-independent.
The boundary term $\TV(p_T, p_{\mathrm{base}})$ is also independent of $S$, by~\Cref{thm:prior_noS_uniform}.

Full step-by-step
derivations of the claims above -- the Kronecker-sum factorization
$p_t(\bfm) = \alpha_t^d$, the successor-set count
$|\Nmasked(\bfy)| \le d$, the re-indexing via the
score-marginal cancellation, the Case~2 boundary computation, and the
Case~1 spectral-gap bound -- are given in
Appendix~\ref{ssec:bound_detail}.

\subsection{Derivational Details to Achieve \texorpdfstring{$S$}{S}-Independence}
\label{ssec:bound_detail}

In this subsection, we detail the techniques used in our derivation to achieve independence in the vocabulary size $S$.

\paragraph{Case~1: Prefactor of training loss.} The prefactor outside
$\sqrt{\LSE(s, \tilde{s}) + \mathcal{O}(|s-\tilde{s}|^3)}$ in Case~1 is
\begin{equation}
\sqrt{2d}\,\Big(\int_0^T \beta(t)(1 - p_t(\bfm))\,dt\Big)^{\!1/2},
\label{eq:case2-prefactor}
\end{equation}
which is obtained using the exit-routing technique; see Remarks~\ref{rmk:SMC_ER} and~\ref{rmk:uniform_vs_masked}.
The factor $\sqrt{d}$ is the sequence length. Inside the integral, $\beta(t)$
is the schedule and $p_t(\bfm)$ is the joint probability of the all-mask
state. Under the Kronecker-sum dynamics \eqref{eq:Q_factor}, the conditional
forward kernel factorizes as $p_{t|0}(\bfx \mid \bfx_0) = \prod_i
p^{\mathrm{tok}}_{t|0}(x^i \mid x_0^i)$, where $p^{\mathrm{tok}}_{t|0}(\mask \mid x_0^i)
= \alpha_t$ for any non-mask $x_0^i$ and the single-token mask probability
$\alpha_t$ satisfies $\partial_t \alpha_t = \beta(t)(1 - \alpha_t)$.
Marginalizing over $\bfx_0 \sim p_{\mathrm{data}}$ (supported on non-mask sequences),
\begin{equation}
p_t(\bfm) = \sum_{\bfx_0} p_{\mathrm{data}}(\bfx_0)\!\prod_i p^{\mathrm{tok}}_{t|0}(\mask \mid x_0^i) = \alpha_t^d.
\end{equation}
The integral in \eqref{eq:case2-prefactor} depends only on $\beta$ and
$d$ (no $S$), and the whole prefactor is $S$-independent.

\paragraph{Case~1: Score-marginal Cancellation
$p_t(\bfx)s_t(\bfx)_\bfy = p_t(\bfy)$.} The $S$-independent of the bound~\eqref{eq:IPM_absorb} results from a re-indexing in the
proof. The integrant in the derivation in~\Cref{subsec:proofIPM} reads
\begin{equation}
\sum_{\bfx \in \mathcal{X}}\!\sum_{\bfy \ne \bfx}\!p_t(\bfx) Q_t(\bfy, \bfx)\,(s_t(\bfx)_\bfy - \tilde s_t(\bfx)_\bfy)\,[\,\cdot\,].
\end{equation}
For the masked kernel, $Q_t(\bfy, \bfx) > 0$ when $\bfy \to \bfx$ in
the forward chain, i.e., $\bfx$ is obtained from $\bfy$ by masking one
non-mask position; equivalently, $\bfx \in \Nmasked(\bfy)$. With
$\bfx$ fixed, $\bfy$ is obtained from $\bfx$ by replacing one of $\bfx$'s
$m(\bfx)$ $\mask$ tokens with one of $S - 1$ non-mask values:
\begin{equation}
\#\{\bfy : Q_t(\bfy, \bfx) > 0\} = m(\bfx)(S - 1),
\end{equation}
so the factor $S - 1$ enters the inner sum.

The substitution $s_t(\bfx)_\bfy = p_t(\bfy)/p_t(\bfx)$ gives
$p_t(\bfx)(s_t(\bfx)_\bfy - \tilde s_t(\bfx)_\bfy) = p_t(\bfy)(1 -
\tilde s_t(\bfx)_\bfy/s_t(\bfx)_\bfy)$. The sum re-indexes by holding
$\bfy$ fixed and summing over $\bfx \in \Nmasked(\bfy)$:
\begin{equation}
\sum_{\bfx, \bfy \ne \bfx}\!p_t(\bfx) Q_t(\bfy, \bfx)\,(s_t - \tilde s_t)[\,\cdot\,]
= \sum_{\bfy \ne \bfm}\!p_t(\bfy)\!\!\sum_{\bfx \in \Nmasked(\bfy)}\!\!\!Q_t(\bfy, \bfx)\,\big(1 - \tfrac{\tilde s_t(\bfx)_\bfy}{s_t(\bfx)_\bfy}\big)[\,\cdot\,].
\end{equation}
The outer sum is restricted to $\bfy \ne \bfm$ since $\Nmasked(\bfm)
= \emptyset$, and is weighted by $p_t(\bfy)$ with total mass
$\sum_{\bfy \ne \bfm} p_t(\bfy) = 1 - p_t(\bfm) \le 1$. The inner sum
has $|\Nmasked(\bfy)| \le d$ terms. The re-indexing converts the
$m(\bfx)(S - 1)$ choices in the original sum into $\le d$ choices in the
re-indexed sum; the factor of $S - 1$ is absorbed by the score-marginal
cancellation.

\paragraph{Case~1: boundary term.} For $p_{\rm base} = \delta_{\bfm}$,
\begin{align*}
  \TV(p_T, \delta_\bfm)
  &= \tfrac{1}{2}\!\sum_{\bfx}\!|p_T(\bfx) - \delta_\bfm(\bfx)| \\
  &= \tfrac{1}{2}\big[\,|p_T(\bfm) - 1| + \!\sum_{\bfx \ne \bfm}\!p_T(\bfx)\,\big] \\
  &= \tfrac{1}{2}\big[(1 - p_T(\bfm)) + (1 - p_T(\bfm))\big]
  \;=\; 1 - p_T(\bfm),
\end{align*}
using $p_T(\bfm) \le 1$ and $\sum_{\bfx \ne \bfm}p_T(\bfx) = 1 - p_T(\bfm)$.

By construction of the forward process, conditional on the initial
sequence $\bfx_0$, the $d$ positions evolve as \textbf{independent} CTMCs, each
with the same time-dependent rate $\beta_t$ and $\mask$ absorbing.
Hence for any non-mask token $v$ the per-position mask probability
\[
  \alpha_T \;:=\; \Pr(\bfx_T^{(i)} = \mask \mid \bfx_0^{(i)} = v)
\]
depends on neither $i$ (positions share dynamics) nor $v$ (the
absorption rate is the same for every non-mask token), so $\alpha_T$
is indeed a (pure) scalar. The data distribution $p_{\rm data}$ is
supported on non-mask sequences, so for $p_{\rm data}$-a.e.\ $\bfx_0$
every coordinate $\bfx_0^{(i)}$ is non-mask, and independence of the
per-position CTMCs gives
\[
  \Pr(\bfx_T = \bfm \mid \bfx_0)
  \;=\; \prod_{i=1}^d \Pr(\bfx_T^{(i)} = \mask \mid \bfx_0^{(i)})
  \;=\; \alpha_T^d.
\]
Marginalizing over $\bfx_0$,
\[
  p_T(\bfm) \;=\; \E_{\bfx_0 \sim p_{\rm data}}\!\big[\Pr(\bfx_T = \bfm \mid \bfx_0)\big]
  \;=\; \alpha_T^d,
\]
since the conditional probability is the constant $\alpha_T^d$
$p_{\rm data}$-a.s.

 The factorization
$1 - \alpha_T^d = (1-\alpha_T)\sum_{k=0}^{d-1}\alpha_T^k$ with
$\alpha_T \in [0,1]$ then gives
\[
  \TV(p_T, \delta_\bfm) \;=\; 1 - \alpha_T^d \;\le\; d(1-\alpha_T),
\]
independent of vocabulary size $S$: absorption sends every non-mask
token to a single destination per position.

\section{Proof of Uniformization Iteration Complexity}
\label{sec:appendix_step_complexity}

\subsection{Proof of Auxiliary Lemmas}
We state and proof the lemmas for the iteration complexity of uniformization.

\begin{lemma}[Conditional Representation of Score]\label{lemma:score_conditional}
    For any $\bfx$, $\bfy$, and $t>0$ satisfying $p_t(\bfx)\neq 0$ and $\bfx_{i \to y^i}=\bfy$, we have
    \begin{equation}
        s_t(\bfx)_\bfy = \mathbb{E}_{\bfx_0 \sim p_{0|t}(\cdot | \bfx)} \left[ \frac{p^i_{t|0}(y^i|x_0^i)}{p^i_{t|0}(x^i|x^i_0)} \right].
    \end{equation}
\end{lemma}

\begin{proof}
    This lemma directly follows~\cite[Lemma~2]{chen2025convergence}, we show the proof here for completeness.

    \begin{align*}
        s_t(\bfx)_\bfy = \frac{p_t(\bfy)}{p_t(\bfx)} &=  \frac{p_t(\bfx_{i \to y^i})}{p_t(\bfx)} \\
        &= \frac{1}{p_t(\bfx)} \sum_{\bfx_0 \in \mathcal{X}} p_0(\bfx_0) p_{t|0}(\bfx_{i \to y^i}|\bfx_0) \\
        &= \frac{1}{p_t(\bfx)} \sum_{\bfx_0 \in \mathcal{X}} p_0(\bfx_0) \prod_{k \neq i} p^k_{t|0}(x^k|x^k_0) p^i_{t|0}(y^i|x^i_0) \\
        &= \frac{1}{p_t(\bfx)} \sum_{\bfx_0 \in \mathcal{X}} p_0(\bfx_0) \prod_{k=1}^d p^k_{t|0}(x^k|x^k_0) \frac{p^i_{t|0}(y^i|x^i_0)}{p^i_{t|0}(x^i|x^i_0)} \\
        &= \sum_{\bfx_0 \in \mathcal{X}} \frac{p_0(\bfx_0) p_{t|0}(\bfx|\bfx_0)}{p_t(\bfx)} \frac{p^i_{t|0}(y^i|x^i_0)}{p^i_{t|0}(x^i|x^i_0)} \\
        &= \sum_{\bfx_0 \in \mathcal{X}}  p_{0|t}(\bfx_0|\bfx) \frac{p^i_{t|0}(y^i|x^i_0)}{p^i_{t|0}(x^i|x^i_0)}, \\
        \intertext{by Bayes' rule,}
        &= \mathbb{E}_{\bfx_0 \sim p_{0|t}(\cdot | \bfx)} \left[ \frac{p^i_{t|0}(y^i|x^i_0)}{p^i_{t|0}(x^i|x^i_0)} \right].
    \end{align*}
\end{proof}

\begin{lemma}[Score Representation under Masked Rate]\label{lemma:score_repre_masked}
    Under the masked rate~\eqref{eq:Q_absorb_uni} and~\eqref{eq:Q_factor}, when $Q_t(\bfy, \bfx)>0$, the score is given by
    \[
    s_t(\bfx)_\bfy = \frac{e^{-\int_0^t \beta(\tau) d\tau}}{1-e^{-\int_0^t \beta(\tau) d\tau}} p_{0|t}^i (y^i|\bfx).
    \]
\end{lemma}

\begin{proof}
    This proof is modified from~\cite[Lemma~1]{liang2026absorb} and~\cite[Theorem~1]{ou2025your}.

    Under the masked rate~\eqref{eq:Q_absorb_uni} and~\eqref{eq:Q_factor}, by Lemma~\ref{lemma:score_conditional}, we have
    \begin{equation}\label{eq:score_conditional_proof}
        s_t(\bfx)_\bfy = \mathbb{E}_{\bfx_0 \sim p_{0|t}(\cdot | \bfx)} \left[ \frac{p^i_{t|0}(y^i|x_0^i)}{p^i_{t|0}(x^i|x^i_0)} \right].
    \end{equation}

    On the other hand, the single-token
mask probability $\alpha_t$ satisfies the ODE
\[
    \partial_t \alpha_t = \beta(t)(1 - \alpha_t) \quad \text{subject to} \quad \alpha_0=1,
\]
derived from the masked rate~\eqref{eq:Q_absorb_uni} and with solution
\begin{equation}\label{eq:mask_prof}
    \alpha_t = 1-e^{-\int_0^t \beta(\tau) d\tau}, \quad \text{for } t\in[0, T].
\end{equation}

Consider the case where $Q_t(\bfy, \bfx)>0$, which means $\bfx=\bfy_{i \to \mask}$ and $y^i \neq \mask$ for an $i \in [d]$. And the likelihood $p^i_{t|0}(y^i|x_0^i) \neq 0$ only when $x_0^i=y^i$. In this case, by~\eqref{eq:mask_prof}, the likelihood ratio in~\eqref{eq:score_conditional_proof} is given by
\begin{equation}\label{eq:likelihood}
    \frac{p^i_{t|0}(y^i|y^i)}{p^i_{t|0}(x^i|y^i)} =\frac{p^i_{t|0}(y^i|y^i)}{p^i_{t|0}(\mask|y^i)}  = \frac{1-\alpha_t}{\alpha_t} = \frac{e^{-\int_0^t \beta(\tau) d\tau}}{1-e^{-\int_0^t \beta(\tau) d\tau}}.
\end{equation}

Thus, when $Q_t(\bfy, \bfx)>0$, \eqref{eq:score_conditional_proof} can be evaluated as
\begin{align*}
s_t(\bfx)_\bfy &= \mathbb{E}_{\bfx_0 \sim p_{0|t}(\cdot | \bfx)} \left[ \frac{p^i_{t|0}(y^i|x_0^i)}{p^i_{t|0}(x^i|x^i_0)} \right] \\
&= \mathbb{E}_{x^i_0 \sim p^i_{0|t}(\cdot | \bfx)} \left[ \frac{p^i_{t|0}(y^i|x_0^i)}{p^i_{t|0}(x^i|x^i_0)} \right], \quad \text{since the expression inside the expectation only depends on $x^i_0$,} \\
&= \sum_{x^i_0 \in [S]} \frac{p^i_{t|0}(y^i|x_0^i)}{p^i_{t|0}(x^i|x^i_0)} p^i_{0|t}(x^i_0 | \bfx) \\
&= \sum_{\substack{x^i_0 \in [S] \\ x_0^i=y^i}} \frac{p^i_{t|0}(y^i|x_0^i)}{p^i_{t|0}(x^i|x^i_0)} p^i_{0|t}(x^i_0 | \bfx), \quad \text{as $p^i_{t|0}(y^i|x_0^i)=0$ when $x_0^i\neq y^i$ under the masked rate,} \\
&=  \frac{p^i_{t|0}(y^i|y^i)}{p^i_{t|0}(x^i|y^i)} p^i_{0|t}(y^i | \bfx) \\
&= \frac{e^{-\int_0^t \beta(\tau) d\tau}}{1-e^{-\int_0^t \beta(\tau) d\tau}} p_{0|t}^i (y^i|\bfx), \quad \text{by~\eqref{eq:likelihood}.}
\end{align*}
\end{proof}

\begin{lemma}[Score Upper Bound under Masked Rate]\label{lemma:score_upper_bound_masked}
    Under the masked rate~\eqref{eq:Q_absorb_uni} and~\eqref{eq:Q_factor}, 
    \[
    \sum_{\substack{\bfy\in {[S]^d} \\ \bfy\neq \bfx}}  s_t(\bfx)_\bfy Q_t(\bfy, \bfx) \leq \frac{e^{-\int_0^t \beta(\tau) d\tau}}{1-e^{-\int_0^t \beta(\tau) d\tau}} \beta(t) \underbrace{\left[ \sum_{i \in [d]} \delta\{x^i, \mask\} \right]}_{\text{number of masks in $\bfx$}}.
    \]
\end{lemma}

\begin{proof}
This proof is modified from~\cite[Lemma~5]{liang2026absorb}.

Under the masked rate~\eqref{eq:Q_absorb_uni} and~\eqref{eq:Q_factor}, by Lemma~\ref{lemma:score_repre_masked}, when $Q_t(\bfy, \bfx)>0$, we have
\begin{equation}\label{eq:score_repre_masked_proof}
s_t(\bfx)_\bfy = \frac{e^{-\int_0^t \beta(\tau) d\tau}}{1-e^{-\int_0^t \beta(\tau) d\tau}} p_{0|t}^i (y^i|\bfx).
\end{equation}
Here, $\bfx=\bfy_{i \to \mask}$ and $y^i \neq \mask$. We evaluate
\begin{align*}
    \sum_{\substack{\bfy\in {[S]^d} \\ \bfy\neq \bfx}}  s_t(\bfx)_\bfy Q_t(\bfy, \bfx) &= \sum_{\substack{\bfy\in [S]^d \\ Q_t(\bfy, \bfx)>0}}  s_t(\bfx)_\bfy Q_t(\bfy, \bfx) \\
    &= \sum_{\substack{i \in [d] \\ x^i = \mask}} \sum_{\substack{y^i \in [S] \\ y^i \neq \mask}}  s_t(\bfx)_\bfy \beta(t) \\
    &= \frac{e^{-\int_0^t \beta(\tau) d\tau}}{1-e^{-\int_0^t \beta(\tau) d\tau}} \beta(t) \sum_{\substack{i \in [d] \\ x^i = \mask}} \underbrace{\sum_{\substack{y^i \in [S] \\ y^i \neq \mask}}  p_{0|t}^i (y^i|\bfx)}_{\leq 1}, \\
    \intertext{by~\eqref{eq:score_repre_masked_proof},}
    &\leq \frac{e^{-\int_0^t \beta(\tau) d\tau}}{1-e^{-\int_0^t \beta(\tau) d\tau}} \beta(t) \underbrace{\left[ \sum_{i \in [d]} \delta\{x^i, \mask\} \right]}_{\text{number of masks in $\bfx$}}.
\end{align*}
\end{proof}

\begin{lemma}\label{lemma:log_increase}
Given a constant $R > 1$, if $r$ satisfies $1 \leq r \leq R$, then
    \begin{equation}\label{eq:log_increase}
    r-1 \leq \frac{R-1}{\log R} \log r.
    \end{equation}
\end{lemma}

\begin{proof}
    If $r=1$, then $\log r=0$ and \eqref{eq:log_increase} trivially holds.
    Consider the case when $r>1$, , we can equivalently show that 
    \[
    \frac{r-1}{\log r} \leq \frac{R-1}{\log R}.
    \]
    To this end, it suffices to show that the function
    \[
    h(r) = \frac{r-1}{\log r}
    \]
    is increasing for $r>1$. It derivative is given by
    \begin{equation}
        h'(r) = \frac{\log r - \frac{r-1}{r}}{(\log r)^2} = \frac{\log r - 1 + \frac{1}{r}}{(\log r)^2}.
    \end{equation}
    Here, let the numerator be $g(r) = \log r - 1 + \frac{1}{r}$. Remark that $g(1)=0$ and $g'(r)=\frac{1}{r}-\frac{1}{r^2}>0$ for $r>1$. Thus, $g(r)>0$ for $r>1$. This implies $h'(r)>0$ and thus $h(r)$ is increasing for $r>1$. 
\end{proof}

\begin{lemma}\label{lemma:time_ratio_bound}
    For a time discretization $\{ t_k\}_{k=0}^{N_B}$, where $\delta=t_0 < t_1 < \cdots < t_{N_B}=T$,
     and satisfying $\max\limits_k \frac{t_{k+1}}{t_k} = R$ for some constant $R>1$, we have
    \begin{equation}
        \frac{t_{k+1}-t_k}{t_k} \leq \frac{R-1}{\log R} \log\frac{t_{k+1}}{t_k}.
    \end{equation}
\end{lemma}

\begin{proof}
    Note that 
    \[
    \frac{t_{k+1}-t_k}{t_k} = \frac{t_{k+1}}{t_k} - 1.
    \]
    Letting $r=\frac{t_{k+1}}{t_k}$ and applying Lemma~\ref{lemma:log_increase} yields the desired result.
\end{proof}

\subsection{Proof of \texorpdfstring{\Cref{cor:StepComplexity}}{Corollary \ref*{cor:StepComplexity}}}
\corStepComplexity*

\begin{proof}[Proof of \Cref{cor:StepComplexity} for the Masked Rate] We first prove the expected step complexity for the masked rate.
The proof is modified from~\cite[Theorem~3]{liang2026absorb}.

    By Lemma~\ref{lemma:score_upper_bound_masked}, we have
    \[
    \sum_{\substack{\bfy\in {[S]^d} \\ \bfy\neq \bfx}}  s_t(\bfx)_\bfy Q_t(\bfy, \bfx) \leq \underbrace{\frac{e^{-\int_0^t \beta(\tau) d\tau}}{1-e^{-\int_0^t \beta(\tau) d\tau}}}_{\leq 1/\int_0^t \beta(\tau) d\tau} \beta(t) \underbrace{\left[ \sum_{i \in [d]} \delta\{x^i, \mask\} \right]}_{\text{number of masks in $\bfx$, $\leq d$}} \leq \frac{d \beta(t)}{\int_0^t \beta(\tau) d\tau}.
    \]

    Thus, under the assumption that $\tilde{s}_{t}(\bfx)_\bfy \asymp s_{t}(\bfx)_\bfy$ when $Q_t(\bfy, \bfx) > 0$, we have
    \begin{equation}\label{eq:sQ_bound}
    -\tilde{Q}_{t}^{\leftarrow}(\bfx, \bfx) = \sum_{\substack{\bfy\in{[S]^d} \\ \bfy\neq \bfx}} \tilde{Q}_{t}^{\leftarrow}(\bfx, \bfy) = \sum_{\substack{\bfy\in{[S]^d} \\ \bfy\neq \bfx}} Q_t(\bfy, \bfx) \tilde{s}_t(\bfx)_\bfy \lesssim \sum_{\substack{\bfy\in{[S]^d} \\ \bfy\neq \bfx}} Q_t(\bfy, \bfx) {s}_t(\bfx)_\bfy \leq \frac{d \beta(t)}{\int_0^t \beta(\tau) d\tau}.
    \end{equation}
    To simulate the approximate reverse process, in each discretized time interval $(t_k, t_{k+1}]$, uniformization requires that the Poisson rate $\Lambda_k$ satisfies
    \[
        \Lambda_k \geq \sup_{\substack{\bfx \in \mathcal{X} \\ t \in (t_k, t_{k+1}]}} -\tilde{Q}^{\leftarrow}_t (\bfx, \bfx).
    \]
    Thus, suppose we choose the Poisson rate to be
    \begin{equation}\label{eq:Pois_rate_bound}
        \Lambda_k = C \cdot \sup_{\substack{\bfx \in \mathcal{X} \\ t \in (t_k, t_{k+1}]}} -\tilde{Q}^{\leftarrow}_t (\bfx, \bfx),
    \end{equation}
    for some constant $C \geq 1$.
    Remark that the sum of independent Poisson$(\Lambda_k)$ random variables is distributed as Poisson$(\sum_k \Lambda_k)$, the expected total number of steps are given by
    \begin{align}
        \mathbb{E}[N] &= \sum_{k=0}^{N_B-1} \Lambda_k (t_{k+1}-t_k), \nonumber \\
        \intertext{where $\delta =t_0 <t_1 < \cdots <t_{N_B}=T$ is the time discretization,}
        & \lesssim \sum_{k=0}^{N_B-1} \frac{d  \,\max\limits_{0\leq t \leq T}\left(\beta(t)\right)}{\int_0^{t_{k}} \beta(\tau) d\tau} (t_{k+1}-t_k), \quad \text{by~\eqref{eq:sQ_bound} and~\eqref{eq:Pois_rate_bound},} \nonumber\\
        &= d  \,\max\limits_{0\leq t \leq T}\left(\beta(t)\right) \sum_{k=0}^{N_B-1} \frac{1}{\int_0^{t_{k}} \beta(\tau) d\tau} (t_{k+1}-t_k) \nonumber\\ 
        & \leq d \, \,\max\limits_{0\leq t \leq T}\left(\beta(t)\right) \sum_{k=0}^{N_B-1} \frac{1}{t_{k} \min\limits_{0 \leq t \leq t_{k}} \left(\beta(t)\right) } (t_{k+1}-t_k) \nonumber\\
        & \leq d \, \frac{\max\limits_{0\leq t \leq T}\left(\beta(t)\right)}{\min\limits_{0\leq t \leq T}\left(\beta(t)\right)} \sum_{k=0}^{N_B-1} \frac{t_{k+1}-t_k}{t_{k}} \nonumber\\
        & = d \, \frac{\betamax}{\beta_{\mathrm{min}}} \sum_{k=0}^{N_B-1} \frac{t_{k+1}-t_k}{t_{k}} \nonumber\\
        & \leq d \, \frac{\betamax}{\beta_{\mathrm{min}}} \sum_{k=0}^{N_B-1}  \frac{R-1}{\log R} \log\frac{t_{k+1}}{t_k}, \quad \text{by Lemma~\ref{lemma:time_ratio_bound}, where $R=\max_k \frac{t_{k+1}}{t_k}$,}\nonumber\\
        & \lesssim d \, \frac{\betamax}{\beta_{\mathrm{min}}} \sum_{k=0}^{N_B-1} \log\frac{t_{k+1}}{t_k} \nonumber\\
        &= d \, \frac{\betamax}{\beta_{\mathrm{min}}}\sum_{k=0}^{N_B-1} \int_{t_k}^{t_{k+1}} \frac{1}{t} dt \nonumber\\
        &= d \,\frac{\betamax}{\beta_{\mathrm{min}}} \int_\delta^T \frac{1}{t} dt \nonumber\\
        &= d \,\frac{\betamax}{\beta_{\mathrm{min}}} (\log(T) + \log(\delta^{-1})). \label{eq:expected_step_masked}
    \end{align}
    On the other hand, by Corollary~\ref{cor:ES}, the approximation error under any IPM $\gamma_\Phi$ satisfies
    \[
     \gamma_\Psi(p_{\mathrm{data}}, \tilde{p}_\delta)  \leq 2 C_\Psi \left( d {\textstyle \int_0^\delta \beta(t)\,dt }+ d e^{-\| \beta \|_1}   +   \sqrt{2 d} \sqrt{\| \beta (1-p(\bfm))\|_1} \sqrt{ \,\LSE^\delta(s, \tilde{s}) + \tfrac{2}{3}\,\LThree^\delta(s, \tilde{s}) } \right).
    \]
    For a given threshold $\epsilon>0$ for the approximation error, it suffices to bound each of the three terms on the right-hand-side by $\epsilon/3$. This is achieved by
    \begin{align}
         \delta &< \frac{\epsilon}{6 C_\Psi d \betamax} \;\Rightarrow \; \delta^{-1} > \frac{6 C_\Psi d \betamax}{\epsilon} , \quad \text{thus we choose} \; \delta^{-1} \asymp \frac{d C_\Psi \betamax }{\epsilon}, \label{eq:masked_delta_bound}\\
        T  &> \frac{1}{\betamin} \log\left( \frac{6C_\Psi d}{\epsilon} \right), \quad \text{thus we choose} \; T \asymp \frac{1}{\betamin} \log\left( \frac{C_\Psi d}{\epsilon} \right), \label{eq:masked_epsilon_bound}\\
        \LSE^\delta(s, \tilde{s}) + \tfrac{2}{3}\,\LThree^\delta(s, \tilde{s}) &= \frac{\epsilon^2}{72 d C_\Psi^2 \| \beta (1-p(\bfm))\|_1} \;\Rightarrow \; \LSE^\delta(s, \tilde{s}) \lesssim \frac{\epsilon^2}{ d C_\Psi^2 \| \beta (1-p(\bfm))\|_1}.
    \end{align}

    Plugging \eqref{eq:masked_delta_bound} and \eqref{eq:masked_epsilon_bound} into \eqref{eq:expected_step_masked} yields
    \begin{align*}
        \mathbb{E}[N] &\lesssim d \,\frac{\betamax}{\beta_{\mathrm{min}}} (\log(T) + \log(\delta^{-1}))\\
        &= \mathcal{O}\left(d \,\frac{\betamax}{\beta_{\mathrm{min}}} \left( \log \left( \frac{1}{\betamin} \log\frac{d C_\Psi}{\epsilon} \right) + \log\frac{d C_\Psi \betamax}{\epsilon} \right) \right) \\
        &= \mathcal{O}\left(d \, \frac{\betamax}{\beta_{\mathrm{min}}} \, \left( \log\frac{1}{\betamin} + \log\frac{d C_\Psi \betamax}{\epsilon} \right) \right)\\
        &= \mathcal{O}\left(d \, \frac{\betamax}{\beta_{\mathrm{min}}} \,  \log \left( \frac{\betamax}{\betamin}\cdot \frac{d C_\Psi}{\epsilon}  \right) \right) .
    \end{align*}
\end{proof}

\begin{proof}[Proof of \Cref{cor:StepComplexity} for the Uniform Rate] Next, we prove the expected step complexity for the uniform rate. This proof is modified from the proof of~\cite[Theorem~4.9]{ren2025discrete}.

To simulate the approximate reverse process, in each discretized time interval $(t_k, t_{k+1}]$, uniformization requires that the Poisson rate $\Lambda_k$ satisfies
    \[
        \Lambda_k \geq \sup_{\substack{\bfx \in \mathcal{X} \\ t \in (t_k, t_{k+1}]}} -\tilde{Q}^{\leftarrow}_t (\bfx, \bfx).
    \]
    Thus, suppose we choose the Poisson rate to be
    \begin{equation}\label{eq:Pois_rate_bound_uniform}
        \Lambda_k = C \cdot \sup_{\substack{\bfx \in \mathcal{X} \\ t \in (t_k, t_{k+1}]}} -\tilde{Q}^{\leftarrow}_t (\bfx, \bfx),
    \end{equation}
    for some constant $C \geq 1$.

    Thus, under the assumption that $\tilde{s}_{t}(\bfx)_\bfy \asymp s_{t}(\bfx)_\bfy$ when $Q_t(\bfy, \bfx) > 0$ and $s_t(\bfx)_\bfy \lesssim \max(1, \frac{1}{t})$, we have
    \begin{equation}\label{eq:sQ_bound_uniform}
    -\tilde{Q}_{t}^{\leftarrow}(\bfx, \bfx) = \sum_{\substack{\bfy\in{[S]^d} \\ \bfy\neq \bfx}} \tilde{Q}_{t}^{\leftarrow}(\bfx, \bfy) = \sum_{\substack{\bfy\in{[S]^d} \\ \bfy\neq \bfx}} Q_t(\bfy, \bfx) \tilde{s}_t(\bfx)_\bfy \lesssim \underbrace{\sum_{\substack{\bfy\in{[S]^d} \\ \bfy\neq \bfx}} Q_t(\bfy, \bfx)}_{\substack{d(S-1) \text{ nonzero terms} \\ \text{each term }=\frac{\beta(t)}{S}}} \underbrace{{s}_t(\bfx)_\bfy}_{\substack{\lesssim \max \left(1, \frac{1}{t} \right), \\ \text{by assumption}}} \lesssim d \beta(t) \max \left( 1, \frac{1}{t} \right).
    \end{equation}
    
    Remark that the sum of independent Poisson$(\Lambda_k)$ random variables is distributed as Poisson$(\sum_k \Lambda_k)$, the expected total number of steps are given by
    \begin{align}
        \mathbb{E}[N] &= \sum_{k=0}^{N_B-1} \Lambda_k (t_{k+1}-t_k), \nonumber \\
        \intertext{where $\delta =t_0 <t_1 < \cdots <t_{N_B}=T$ is the time discretization,}
        & \lesssim \betamax \,d\, \sum_{k=0}^{N_B-1}  \max \left(1, \frac{1}{t_k} \right) (t_{k+1}-t_k), \quad \text{by \eqref{eq:Pois_rate_bound_uniform} and \eqref{eq:sQ_bound_uniform},} \nonumber\\
        &\leq \betamax \,d\, \left( \sum_{k=0}^{N_B'-1} \underbrace{\frac{t_{k+1}-t_k}{t_k}}_{\substack{\leq \frac{R-1}{\log R} \log \frac{t_{k+1}}{t_k}, \\ \text{by Lemma~\ref{lemma:time_ratio_bound}}}} + \underbrace{\sum_{k=N_B'}^{N_B-1} (t_{k+1}-t_k)}_{\leq T} \right), \nonumber\\
        \intertext{where $N_B'$ is the unique index satisfying $t_{N_b'-1} \leq 1 <t_{N_b'} $, and $R=\max_k \frac{t_{k+1}}{t_k}$,}
        &\lesssim \betamax \, d \, \left( \left[\sum_{k=0}^{N_B'-1} \log \frac{t_{k+1}}{t_k} \right] + T \right) \nonumber\\
        &\leq \betamax \, d \, \left( \left[\sum_{k=0}^{N_B-1} \log \frac{t_{k+1}}{t_k} \right] + T \right), \quad \text{as $N_B \geq N_b'$,} \nonumber\\
        &= \betamax \, d \,  \left( \int_\delta^T \frac{1}{t} \, dt + T \right) \nonumber\\
        &= \betamax \, d \,  \left( \log \left( T\right) + \log \left( \delta^{-1} \right) + T \right). \label{eq:expected_step_uniform}
    \end{align}
    On the other hand, by Corollary~\ref{cor:ES}, the approximation error under any IPM $\gamma_\Phi$ satisfies
    \[
    \gamma_\Psi(p_{\mathrm{data}}, \tilde{p}_\delta)  \leq 2 C_\Psi \left( d {\textstyle \int_0^\delta \beta(t)\,dt } + d e^{-\| \beta \|_1}  +  \sqrt{2d} \sqrt{\|\beta \|_1}  \sqrt{ \,\LSE^\delta(s, \tilde{s}) + \tfrac{2}{3}\,\LThree^\delta(s, \tilde{s}) } \right).
    \]
    For a given threshold $\epsilon>0$ for the approximation error, it suffices to bound each of the three terms on the right-hand-side by $\epsilon/3$. This is achieved by
    \begin{align}
        \delta &< \frac{\epsilon}{6 C_\Psi d \betamax} \;\Rightarrow \; \delta^{-1} > \frac{6 C_\Psi d \betamax}{\epsilon} , \quad \text{thus we choose} \; \delta^{-1} \asymp \frac{d C_\Psi \betamax }{\epsilon},\label{eq:uniform_delta_bound}\\
        T  &> \frac{1}{\betamin} \log\left( \frac{6C_\Psi d}{\epsilon} \right), \quad \text{thus we choose} \; T \asymp \frac{1}{\betamin} \log\left( \frac{C_\Psi d}{\epsilon} \right), \label{eq:uniform_epsilon_bound}\\
        \LSE^\delta(s, \tilde{s}) + \tfrac{2}{3}\,\LThree^\delta(s, \tilde{s}) &= \frac{\epsilon^2}{72 d C_\Psi^2 \| \beta \|_1} \;\Rightarrow \; \LSE^\delta(s, \tilde{s}) \lesssim \frac{\epsilon^2}{ d C_\Psi^2 \| \beta \|_1}.
    \end{align}
    Plugging \eqref{eq:uniform_delta_bound} and \eqref{eq:uniform_epsilon_bound} into \eqref{eq:expected_step_uniform} yields
    \begin{align*}
        \mathbb{E}[N] &\lesssim \betamax \, d \,  \left( \log \left( T\right) + \log \left( \delta^{-1} \right) + T \right) \\
        &= \mathcal{O} \left( \betamax \, d \,  \left(  \log \left( \delta^{-1} \right) + T \right) \right) \\
        &= \mathcal{O} \left( \betamax \, d \, \left( \log \frac{d C_\Psi \betamax}{\epsilon} + \frac{1}{\betamin} \log \frac{d C_\Psi}{\epsilon} \right) \right) \\
        &= \mathcal{O} \left( \betamax \, d \, \left( \log\betamax +  \frac{1}{\betamin} \log \frac{d C_\Psi}{\epsilon} \right) \right). 
    \end{align*}
\end{proof}

\end{document}

%% file: math_commands.tex
\usepackage{amsmath,amsfonts,bm}

\def\eqref#1{equation~\ref{#1}}

\def\1{\bm{1}}

\DeclareMathAlphabet{\mathsfit}{\encodingdefault}{\sfdefault}{m}{sl}
\SetMathAlphabet{\mathsfit}{bold}{\encodingdefault}{\sfdefault}{bx}{n}

\newcommand{\E}{\mathbb{E}}

\DeclareMathOperator{\sign}{sign}

%% file: abbrv.tex
\newcommand{\bfX}{{\bf X}}
\newcommand{\bfY}{{\bf Y}}

\newcommand{\bfe}{{\bf e}}

\newcommand{\bfx}{{\bf x}}
\newcommand{\bfy}{ {\bf y}}

\newcommand{\bfm}{{\bf m}}

\newcommand{\TV}{\mathrm{TV}}
\newcommand{\cI}{\text{\textcircled{I}}}
\newcommand{\cII}{\text{\textcircled{II}}}
\newcommand{\mask}{\text{[MASK]}}
\newcommand{\xito}{\mathbf{x}_{i \to \hat{x}^i}}

\newcommand{\sito}{s_t(\bfx)_{i \to \hat{x}^i}}
\newcommand{\half}{\frac{1}{2}}

\newcommand{\backQt}{Q_t^{\leftarrow}}
\newcommand{\LWSM}{\mathcal{L}_{\mathrm{WSM}}}
\newcommand{\LWRSM}{\mathcal{L}_{\mathrm{WRSM}}}
\newcommand{\LSE}{\mathcal{L}_{\mathrm{SE}}}
\renewcommand{\eqref}[1]{\textup{(\ref{#1})}}
\newcommand{\distH}[2]{d_H(#1, #2)}
\newcommand{\LThree}{\mathcal{L}_{3}}
\newcommand{\diam}{\mathrm{diam}}
\newcommand{\Lip}{\mathrm{Lip}}
\newcommand{\pbase}{p_\mathrm{base}}
\newcommand{\Punif}{\Pi}

\newcommand{\Nmasked}{\mathcal{N}^{\leftarrow}_{\text{masked}}}

\newcommand{\betamax}{\beta_{\mathrm{max}}}
\newcommand{\betamin}{\beta_{\mathrm{min}}}

%% file: iclr2027_conference.bib
@article{Birrell2022,
  author  = {Jeremiah Birrell and Paul Dupuis and Markos A. Katsoulakis and Yannis Pantazis and Luc Rey-Bellet},
  title   = {(f,Gamma)-Divergences: Interpolating between f-Divergences and Integral Probability Metrics},
  journal = {Journal of Machine Learning Research},
  year    = {2022},
  volume  = {23},
  number  = {39},
  pages   = {1--70},
  url     = {http://jmlr.org/papers/v23/21-0100.html}
}

@misc{liang2026sharp,
      title={Sharp Convergence Rates for Masked Diffusion Models}, 
      author={Yuchen Liang and Zhiheng Tan and Ness Shroff and Yingbin Liang},
      year={2026},
      eprint={2602.22505},
      archivePrefix={arXiv},
      primaryClass={cs.LG},
      url={https://arxiv.org/abs/2602.22505}, 
}

@inproceedings{
lou2024discrete,
title={Discrete Diffusion Modeling by Estimating the Ratios of the Data Distribution},
author={Aaron Lou and Chenlin Meng and Stefano Ermon},
booktitle={Forty-first International Conference on Machine Learning},
year={2024},
url={https://openreview.net/forum?id=CNicRIVIPA}
}

@article{zhang2026masked,
  title={Masked Diffusion Modeling for Anomaly Detection},
  author={Zhang, Lixing and Liang, Yuchen and Xie, Liyan},
  journal={arXiv preprint arXiv:2605.30046},
  year={2026}
}

@article{liang2026scores,
  title={From Scores to {G}ibbs Correctors: Accelerating Uniform-Rate Discrete Diffusion Models},
  author={Liang, Yuchen and Shroff, Ness and Liang, Yingbin},
  journal={arXiv preprint arXiv:2605.27352},
  year={2026}
}

@article{campbell2022continuous,
  title={A continuous time framework for discrete denoising models},
  author={Campbell, Andrew and Benton, Joe and De Bortoli, Valentin and Rainforth, Thomas and Deligiannidis, George and Doucet, Arnaud},
  journal={Advances in Neural Information Processing Systems},
  volume={35},
  pages={28266--28279},
  year={2022}
}

@inproceedings{
campbell2024generative,
title={Generative Flows on Discrete State-Spaces: Enabling Multimodal Flows with Applications to Protein Co-Design},
author={Andrew Campbell and Jason Yim and Regina Barzilay and Tom Rainforth and Tommi Jaakkola},
booktitle={Forty-first International Conference on Machine Learning},
year={2024},
url={https://openreview.net/forum?id=kQwSbv0BR4}
}

@inproceedings{rombach2022high,
  title={High-resolution image synthesis with latent diffusion models},
  author={Rombach, Robin and Blattmann, Andreas and Lorenz, Dominik and Esser, Patrick and Ommer, Bj{\"o}rn},
  booktitle={Proceedings of the IEEE/CVF conference on computer vision and pattern recognition},
  pages={10684--10695},
  year={2022}
}

@inproceedings{kotelnikov2023tabddpm,
  title={Tabddpm: Modelling tabular data with diffusion models},
  author={Kotelnikov, Akim and Baranchuk, Dmitry and Rubachev, Ivan and Babenko, Artem},
  booktitle={International conference on machine learning},
  pages={17564--17579},
  year={2023},
  organization={PMLR}
}

@inproceedings{
vignac2023digress,
title={DiGress: Discrete Denoising diffusion for graph generation},
author={Clement Vignac and Igor Krawczuk and Antoine Siraudin and Bohan Wang and Volkan Cevher and Pascal Frossard},
booktitle={The Eleventh International Conference on Learning Representations },
year={2023},
url={https://openreview.net/forum?id=UaAD-Nu86WX}
}

@inproceedings{schiff2025simple,
  title={Simple guidance mechanisms for discrete diffusion models},
  author={Schiff, Yair and Sahoo, Subham and Phung, Hao and Wang, Guanghan and Boshar, Sam and Dalla-torre, Hugo and Almeida, Bernardo and Rush, Alexander and Pierrot, Thomas and Kuleshov, Volodymyr},
  booktitle={International Conference on Learning Representations},
  volume={2025},
  pages={43776--43821},
  year={2025}
}

@inproceedings{zhang2025convergence,
  title={Convergence of score-based discrete diffusion models: A discrete-time analysis},
  author={Zhang, Zikun and Chen, Zixiang and Gu, Quanquan},
  booktitle={International Conference on Learning Representations},
  volume={2025},
  pages={34747--34772},
  year={2025}
}

@article{chen2025convergence,
  title   = {Convergence Analysis of Discrete Diffusion Model: Exact Implementation through Uniformization},
  author  = {Chen, Hongrui and Ying, Lexing},
  journal = {Journal of Machine Learning},
  volume  = {4},
  number  = {2},
  pages   = {108--127},
  year    = {2025},
  month   = jun,
  doi     = {10.4208/jml.240812},
  url     = {https://www.global-sci.com/index.php/jml/article/view/13211}
}

@article{austin2021structured,
  title={Structured denoising diffusion models in discrete state-spaces},
  author={Austin, Jacob and Johnson, Daniel D and Ho, Jonathan and Tarlow, Daniel and Van Den Berg, Rianne},
  journal={Advances in neural information processing systems},
  volume={34},
  pages={17981--17993},
  year={2021}
}

@book{kelly1979reversibility,
  title={Reversibility and stochastic networks},
  author={Kelly, Frank P},
  year={1979},
  publisher={J. Wiley}
}

@article{zhao2025unified,
  title={Unified discrete diffusion for categorical data},
  author={Zhao, Lingxiao and Ding, Xueying and Yu, Lijun and Akoglu, Leman},
  journal={Journal of Machine Learning Research},
  volume={26},
  number={215},
  pages={1--49},
  year={2025}
}

@inproceedings{
sun2023scorebased,
title={Score-based Continuous-time Discrete Diffusion Models},
author={Haoran Sun and Lijun Yu and Bo Dai and Dale Schuurmans and Hanjun Dai},
booktitle={The Eleventh International Conference on Learning Representations },
year={2023},
url={https://openreview.net/forum?id=BYWWwSY2G5s}
}

@article{meng2022concrete,
  title={Concrete score matching: Generalized score matching for discrete data},
  author={Meng, Chenlin and Choi, Kristy and Song, Jiaming and Ermon, Stefano},
  journal={Advances in Neural Information Processing Systems},
  volume={35},
  pages={34532--34545},
  year={2022}
}

@inproceedings{
song2021scorebased,
title={Score-Based Generative Modeling through Stochastic Differential Equations},
author={Yang Song and Jascha Sohl-Dickstein and Diederik P Kingma and Abhishek Kumar and Stefano Ermon and Ben Poole},
booktitle={International Conference on Learning Representations},
year={2021},
url={https://openreview.net/forum?id=PxTIG12RRHS}
}

@inproceedings{ren2025discrete,
  title={How discrete and continuous diffusion meet: Comprehensive analysis of discrete diffusion models via a stochastic integral framework},
  author={Ren, Yinuo and Chen, Haoxuan and Rotskoff, Grant and Ying, Lexing},
  booktitle={International Conference on Learning Representations},
  volume={2025},
  pages={42904--42941},
  year={2025}
}

@article{ren2026fast,
  title={Fast solvers for discrete diffusion models: Theory and applications of high-order algorithms},
  author={Ren, Yinuo and Chen, Haoxuan and Zhu, Yuchen and Guo, Wei and Chen, Yongxin and Rotskoff, Grant and Tao, Molei and Ying, Lexing},
  journal={Advances in Neural Information Processing Systems},
  volume={38},
  pages={167228--167282},
  year={2026}
}

@article{gillespie2001approximate,
  title={Approximate accelerated stochastic simulation of chemically reacting systems},
  author={Gillespie, Daniel T},
  journal={The Journal of chemical physics},
  volume={115},
  number={4},
  pages={1716--1733},
  year={2001},
  publisher={AIP Publishing}
}

@article{cao2006efficient,
  title={Efficient step size selection for the tau-leaping simulation method},
  author={Cao, Yang and Gillespie, Daniel T and Petzold, Linda R},
  journal={The Journal of chemical physics},
  volume={124},
  number={4},
  year={2006},
  publisher={AIP Publishing}
}

@article{hobolth2009simulation,
  title={Simulation from endpoint-conditioned, continuous-time Markov chains on a finite state space, with applications to molecular evolution},
  author={Hobolth, Asger and Stone, Eric A},
  journal={The annals of applied statistics},
  volume={3},
  number={3},
  pages={1204},
  year={2009}
}

@article{li2026neural,
  title={Neural Continuous-Time Markov Chain: Discrete Diffusion via Decoupled Jump Timing and Direction},
  author={Li, Jingyuan and Jiang, Xiaoyi and Wen, Fukang and Liu, Wei and Luo, Renqian and Zhu, Yi and Shi, Zuoqiang and Hu, Pipi},
  journal={arXiv preprint arXiv:2604.15694},
  year={2026}
}

@article{zhu2026mdns,
  title={Mdns: Masked diffusion neural sampler via stochastic optimal control},
  author={Zhu, Yuchen and Guo, Wei and Choi, Jaemoo and Liu, Guan-Horng and Chen, Yongxin and Tao, Molei},
  journal={Advances in Neural Information Processing Systems},
  volume={38},
  pages={35260--35308},
  year={2026}
}

@inproceedings{chang2022maskgit,
  title={Maskgit: Masked generative image transformer},
  author={Chang, Huiwen and Zhang, Han and Jiang, Lu and Liu, Ce and Freeman, William T},
  booktitle={Proceedings of the IEEE/CVF conference on computer vision and pattern recognition},
  pages={11315--11325},
  year={2022}
}

@article{chang2023muse,
  title={Muse: Text-to-image generation via masked generative transformers},
  author={Chang, Huiwen and Zhang, Han and Barber, Jarred and Maschinot, AJ and Lezama, Jose and Jiang, Lu and Yang, Ming-Hsuan and Murphy, Kevin and Freeman, William T and Rubinstein, Michael and others},
  journal={arXiv preprint arXiv:2301.00704},
  year={2023}
}

@article{nie2026large,
  title={Large language diffusion models},
  author={Nie, Shen and Zhu, Fengqi and You, Zebin and Zhang, Xiaolu and Ou, Jingyang and Hu, Jun and Zhou, Jun and Lin, Yankai and Wen, Ji-Rong and Li, Chongxuan},
  journal={Advances in Neural Information Processing Systems},
  volume={38},
  pages={50608--50646},
  year={2026}
}

@article{shi2024simplified,
  title={Simplified and generalized masked diffusion for discrete data},
  author={Shi, Jiaxin and Han, Kehang and Wang, Zhe and Doucet, Arnaud and Titsias, Michalis},
  journal={Advances in neural information processing systems},
  volume={37},
  pages={103131--103167},
  year={2024}
}

@article{zhao2026informed,
  title={Informed correctors for discrete diffusion models},
  author={Zhao, Yixiu and Shi, Jiaxin and Chen, Feng and Druckmann, Shaul and Mackey, Lester and Linderman, Scott},
  journal={Advances in Neural Information Processing Systems},
  volume={38},
  pages={125510--125538},
  year={2026}
}

@inproceedings{le2025discrete,
  title={Discrete markov probabilistic models: An improved discrete score-based framework with sharp convergence bounds under minimal assumptions},
  author={Le-Tuyet-Nhi, PHAM and Shariatian, Dario and Ocello, Antonio and Conforti, Giovanni and Durmus, Alain Oliviero},
  booktitle={Forty-second International Conference on Machine Learning},
  year={2025}
}

@article{cao2004numerical,
  title={The numerical stability of leaping methods for stochastic simulation of chemically reacting systems},
  author={Cao, Yang and Petzold, Linda R and Rathinam, Muruhan and Gillespie, Daniel T},
  journal={Journal of Chemical Physics},
  volume={121},
  number={24},
  pages={12169--12178},
  year={2004}
}

@article{arns2010numerical,
  title={On the numerical analysis of inhomogeneous continuous-time Markov chains},
  author={Arns, Markus and Buchholz, Peter and Panchenko, Andriy},
  journal={INFORMS Journal on Computing},
  volume={22},
  number={3},
  pages={416--432},
  year={2010},
  publisher={INFORMS}
}

@book{peyre2019computational,
  title={Computational optimal transport: With applications to data science},
  author={Peyr{\'e}, Gabriel and Cuturi, Marco},
  year={2019},
  publisher={Now Foundations and Trends}
}

@article{sriperumbudur2009integral,
  title={On integral probability metrics,$\backslash$phi-divergences and binary classification},
  author={Sriperumbudur, Bharath K and Fukumizu, Kenji and Gretton, Arthur and Sch{\"o}lkopf, Bernhard and Lanckriet, Gert RG},
  journal={arXiv preprint arXiv:0901.2698},
  year={2009}
}

@article{mimikos2024score,
  title={Score-based generative models are provably robust: an uncertainty quantification perspective},
  author={Mimikos-Stamatopoulos, Nikiforos and Zhang, Benjamin J and Katsoulakis, Markos A},
  journal={Advances in Neural Information Processing Systems},
  volume={37},
  pages={63154--63183},
  year={2024}
}

@article{sahoo2024simple,
  title={Simple and effective masked diffusion language models},
  author={Sahoo, Subham S and Arriola, Marianne and Schiff, Yair and Gokaslan, Aaron and Marroquin, Edgar and Chiu, Justin T and Rush, Alexander and Kuleshov, Volodymyr},
  journal={Advances in Neural Information Processing Systems},
  volume={37},
  pages={130136--130184},
  year={2024}
}

@inproceedings{zheng2025masked,
  title={Masked diffusion models are secretly time-agnostic masked models and exploit inaccurate categorical sampling},
  author={Zheng, Kaiwen and Chen, Yongxin and Mao, Hanzi and Liu, Ming-Yu and Zhu, Jun and Zhang, Qinsheng},
  booktitle={International Conference on Learning Representations},
  volume={2025},
  pages={63186--63227},
  year={2025}
}

@inproceedings{park2025jump,
  title={Jump your steps: Optimizing sampling schedule of discrete diffusion models},
  author={Park, Yong-Hyun and Lai, Chieh-Hsin and Hayakawa, Satoshi and Takida, Yuhta and Mitsufuji, Yuki},
  booktitle={International Conference on Learning Representations},
  volume={2025},
  pages={96272--96300},
  year={2025}
}

@book{levin2017markov,
  title={Markov chains and mixing times},
  author={Levin, David A and Peres, Yuval},
  volume={107},
  year={2017},
  publisher={American Mathematical Soc.}
}

@article{liang2026absorb,
  title={Absorb and converge: Provable convergence guarantee for absorbing discrete diffusion models},
  author={Liang, Yuchen and Huang, Renxiang and Lai, Lifeng and Shroff, Ness and Liang, Yingbin},
  journal={Advances in Neural Information Processing Systems},
  volume={38},
  pages={20283--20318},
  year={2025}
}

@inproceedings{sohl2015deep,
  title={Deep unsupervised learning using nonequilibrium thermodynamics},
  author={Sohl-Dickstein, Jascha and Weiss, Eric and Maheswaranathan, Niru and Ganguli, Surya},
  booktitle={International conference on machine learning},
  pages={2256--2265},
  year={2015},
  organization={pmlr}
}

@article{hoogeboom2021argmax,
  title={Argmax flows and multinomial diffusion: Learning categorical distributions},
  author={Hoogeboom, Emiel and Nielsen, Didrik and Jaini, Priyank and Forr{\'e}, Patrick and Welling, Max},
  journal={Advances in neural information processing systems},
  volume={34},
  pages={12454--12465},
  year={2021}
}

@article{chen2024fast,
  title={Fast sampling via discrete non-markov diffusion models with predetermined transition time},
  author={Chen, Zixiang and Yuan, Huizhuo and Li, Yongqian and Kou, Yiwen and Zhang, Junkai and Gu, Quanquan},
  journal={Advances in Neural Information Processing Systems},
  volume={37},
  pages={106870--106905},
  year={2024}
}

@article{ho2020denoising,
  title={Denoising diffusion probabilistic models},
  author={Ho, Jonathan and Jain, Ajay and Abbeel, Pieter},
  journal={Advances in neural information processing systems},
  volume={33},
  pages={6840--6851},
  year={2020}
}

@inproceedings{
song2021denoising,
title={Denoising Diffusion Implicit Models},
author={Jiaming Song and Chenlin Meng and Stefano Ermon},
booktitle={International Conference on Learning Representations},
year={2021},
url={https://openreview.net/forum?id=St1giarCHLP}
}

@article{liang2026discrete,
  title={Discrete diffusion models: Novel analysis and new sampler guarantees},
  author={Liang, Yuchen and Liang, Yingbin and Lai, Lifeng and Shroff, Ness},
  journal={Advances in Neural Information Processing Systems},
  volume={38},
  pages={165511--165548},
  year={2026}
}

@article{dmitriev2026efficient,
  title={Efficient sampling with discrete diffusion models: Sharp and adaptive guarantees},
  author={Dmitriev, Daniil and Huang, Zhihan and Wei, Yuting},
  journal={arXiv preprint arXiv:2602.15008},
  year={2026}
}

@article{chen2025optimal,
  title={Optimal inference schedules for masked diffusion models},
  author={Chen, Sitan and Cong, Kevin and Li, Jerry},
  journal={arXiv preprint arXiv:2511.04647},
  year={2025}
}

@inproceedings{
kim2025train,
title={Train for the Worst, Plan for the Best: Understanding Token Ordering in Masked Diffusions},
author={Jaeyeon Kim and Kulin Shah and Vasilis Kontonis and Sham M. Kakade and Sitan Chen},
booktitle={Forty-second International Conference on Machine Learning},
year={2025},
url={https://openreview.net/forum?id=DjJmre5IkP}
}

@inproceedings{kim2026stop,
  title={Stop training for the worst: Progressive unmasking accelerates masked diffusion training},
  author={Kim, Jaeyeon and Geuter, Jonathan and Alvarez-Melis, David and Kakade, Sham and Chen, Sitan},
  year={2026},
  booktitle={ICLR 2026 Workshop on MM Intelligence},
    url={https://openreview.net/forum?id=UIcRg65DF0}
}

@article{li2026breaking,
  title={Breaking AR’s sampling bottleneck: Provable acceleration via diffusion language models},
  author={Li, Gen and Cai, Changxiao},
  journal={Advances in Neural Information Processing Systems},
  volume={38},
  pages={11700--11725},
  year={2026}
}

@article{conforti2025non,
  title={Non-Asymptotic Convergence of Discrete Diffusion Models: Masked and Random Walk dynamics},
  author={Conforti, Giovanni and Durmus, Alain and Pham, Le-Tuyet-Nhi and Raoul, Gael},
  journal={arXiv preprint arXiv:2512.00580},
  year={2025}
}

@article{xu2026scheduling,
  title={Scheduling Thoughts: Learning the Order of Thought in Diffusion Language Models},
  author={Xu, Jiawei and Liu, Minghui and Agrawal, Aakriti and Chen, Yifan and Huang, Furong},
  journal={arXiv preprint arXiv:2606.23567},
  year={2026}
}

@article{benton2024denoising,
  title={From denoising diffusions to denoising markov models},
  author={Benton, Joe and Shi, Yuyang and De Bortoli, Valentin and Deligiannidis, George and Doucet, Arnaud},
  journal={Journal of the Royal Statistical Society Series B: Statistical Methodology},
  volume={86},
  number={2},
  pages={286--301},
  year={2024},
  publisher={Oxford University Press US}
}

@article{grassmann1977transient,
  title={Transient solutions in Markovian queueing systems},
  author={Grassmann, Winfried K},
  journal={Computers \& Operations Research},
  volume={4},
  number={1},
  pages={47--53},
  year={1977},
  publisher={Elsevier}
}

@incollection{de2000transient,
  title={Transient solutions for Markov chains},
  author={de Souza e Silva, Edmundo and Gail, H Richard},
  booktitle={Computational probability},
  pages={43--79},
  year={2000},
  publisher={Springer}
}

@article{jensen1953markoff,
  title={Markoff chains as an aid in the study of Markoff processes},
  author={Jensen, Arne},
  journal={Scandinavian Actuarial Journal},
  volume={1953},
  number={sup1},
  pages={87--91},
  year={1953},
  publisher={Taylor \& Francis}
}

@article{liang2026convergence,
  title={Convergence Guarantees for Time-Inhomogeneous Uniform-Rate Discrete Diffusion Models},
  author={Liang, Yuchen and Lai, Lifeng and Shroff, Ness and Liang, Yingbin},
  journal={Entropy},
  volume={28},
  number={6},
  pages={675},
  year={2026},
  publisher={MDPI}
}

@inproceedings{ou2025your,
  title={Your absorbing discrete diffusion secretly models the conditional distributions of clean data},
  author={Ou, Jingyang and Nie, Shen and Xue, Kaiwen and Zhu, Fengqi and Sun, Jiacheng and Li, Zhenguo and Li, Chongxuan},
  booktitle={International Conference on Learning Representations},
  volume={2025},
  pages={64972--65009},
  year={2025}
}
